\documentclass[twoside,11pt]{article}
\usepackage[preprint]{jmlr2e}

\usepackage{amsmath}
\usepackage{amssymb}
\usepackage{mathtools}
\usepackage{dsfont}
\usepackage{thm-restate}
\usepackage{cases}
\usepackage{stmaryrd,bm}
\usepackage{natbib}
\usepackage[dvipsnames]{xcolor}
\usepackage{hyperref,url}
\usepackage{algorithm}
\usepackage[noend]{algorithmic}
\usepackage{tikz}
\usetikzlibrary{decorations.markings}
\usetikzlibrary{calc}

\newcommand{\transp}{\top}

\renewcommand{\ge}{\geqslant}
\renewcommand{\leq}{\leqslant}
\renewcommand{\geq}{\geqslant}

\newcommand{\bp}{\boldsymbol{p}}
\newcommand{\bdelta}{\boldsymbol{\delta}}

\newcommand{\bq}{\boldsymbol{q}}
\newcommand{\by}{\boldsymbol{y}}
\newcommand{\bz}{\boldsymbol{z}}

\newcommand{\bc}{\boldsymbol{c}}
\newcommand{\bk}{\boldsymbol{k}}
\newcommand{\bu}{\boldsymbol{u}}

\newcommand{\R}{\mathbb{R}}
\newcommand{\cA}{\mathcal{A}}
\newcommand{\cB}{\mathcal{B}}
\newcommand{\cC}{\mathcal{C}}

\newcommand{\cM}{\mathcal{M}}
\newcommand{\cP}{\mathcal{P}}

\newcommand{\X}{\mathcal{X}}
\newcommand{\cH}{\mathcal{H}}
\newcommand{\cF}{\mathcal{F}}
\newcommand{\eqdef}{\stackrel{\mbox{\tiny \rm def}}{=}}
\newcommand{\defeq}{\eqdef}
\newcommand{\bb}[1]{\bm{#1}}
\newcommand{\cD}{\mathcal{D}}
\newcommand{\Dtest}{\cD_{\test}}
\newcommand{\Dtrain}{\cD_{\train}}
\newcommand{\Dcalib}{\cD_{\calib}}
\newcommand{\C}{\mathcal{C}}

\newcommand{\M}{\mathcal{M}}
\newcommand{\V}{\mathcal{V}}

\renewcommand{\S}{\mathcal{S}}
\newcommand{\K}{\mathcal{K}}
\newcommand{\G}{\mathcal{G}}

\newcommand{\F}{\mathcal{F}}
\newcommand{\projC}{\operatorname{Proj}_{\cC}}

\newcommand{\projCQ}{\operatorname{Proj}_{\C_Q}}

\renewcommand{\epsilon}{\varepsilon}

\newcommand{\E}{\mathbb{E}}
\renewcommand{\P}{\mathbb{P}}

\newcommand{\indic}{\mathbf{1}}
\newcommand{\as}{\qquad \mbox{\rm a.s.}}

\DeclareMathOperator{\test}{\mbox{\tiny \rm test}}
\DeclareMathOperator{\train}{\mbox{\tiny \rm train}}
\DeclareMathOperator{\calib}{\mbox{\tiny \rm calib}}
\DeclareMathOperator{\unif}{\mbox{\tiny \rm unif}}
\DeclareMathOperator{\Leb}{{\mathfrak{L}}\!}
\DeclareMathOperator{\conv}{\operatorname{conv}}
\DeclareMathOperator{\interior}{\operatorname{int}}
\DeclareMathOperator{\cl}{\operatorname{cl}}
\DeclareMathOperator{\dtv}{d_{\operatorname{TV}}}
\DeclareMathOperator{\dK}{d_{\operatorname{K}}}

\DeclareMathOperator*{\argmin}{arg\,min}
\DeclareMathOperator*{\bmax}{\mathrm{max}\vphantom{arg\,min}}

\DeclareMathOperator{\quant}{\operatorname{mscv}}

\newcommand{\mfm}{\mathfrak{m}}
\newcommand{\mfp}{\mathfrak{p}}
\newcommand{\dash}{\dagger}

\newtheorem{fact}[theorem]{Fact}
\newtheorem{property}[theorem]{Property}
\newtheorem{informal}[theorem]{Informal Result}
\newtheorem{assumption}[theorem]{Assumption}

\usepackage{lastpage}
\jmlrheading{Vol}{2026}{1-\pageref{LastPage}}{7/26; revised M/YY}{M/YY}{Paper id}{Guillaume Principato and Gilles Stoltz}

\ShortHeadings{Adaptive Conformal Inference through the Lens of Blackwell Approachability}{Principato and Stoltz}
\firstpageno{1}

\begin{document}

\title{Adaptive Conformal Inference \\ through the Lens of Blackwell Approachability}

\author{\name Guillaume Principato \email guillaume.principato@universite-paris-saclay.fr \\
       \addr Université Paris-Saclay, CNRS, Inria, Laboratoire de mathématiques d'Orsay, 91405, Orsay, France \\
       EDF R{\&}D, Palaiseau, France
       \AND
       \name Gilles Stoltz \email gilles.stoltz@universite-paris-saclay.fr \\
       \addr Université Paris-Saclay, CNRS, Inria, Laboratoire de mathématiques d'Orsay, 91405, Orsay, France \\
       HEC Paris, Jouy-en-Josas, France}

\editor{My editor}

\maketitle

\begin{abstract}
This article considers an online
version of conformal inference, called adaptive conformal inference [ACI]
and introduced by \citet{gibbs2021adaptive}:
prediction sets are issued sequentially, after observing
features and before the outcomes are revealed.
These sets are evaluated both in
terms of validity (the fraction of rounds where the outcome
was lying in the prediction set) and efficiency
(the average lengths of the prediction sets).
The two criteria point to different directions
(validity favors larger sets).
We also target a wide range of scenarios,
with exchangeable data and arbitrary data
(lack of any stochastic guarantees) as two extremes.
A series of existing strategies for ACI typically guarantee that
empirical coverage converges to the desired level
for arbitrary sequences,
but they generally lack simultaneous efficiency guarantees.
To provide a unified study,
we first formulate ACI as a repeated two-player game
with finite action sets and vector-valued payoffs
encoding validity and efficiency.
Building on this reformulation,
we introduce a strategy based on
Blackwell approachability and on its opportunistic extension
by \citet{bernstein2014opportunistic}
that ensures validity
while adapting the efficiency
of the prediction intervals to the underlying degree of stochasticity
of the opponent player.
The resulting guarantee is ``best of many worlds'':
it recovers the relevant efficiency guarantees in exchangeable and adversarial settings,
and provides guarantees in intermediate settings
that arise in typical applications such as the forecasting of time series.

\medskip

\end{abstract}

\section{Introduction}

Conformal inference \citep{vovk2005algorithmic} is a general framework
for constructing prediction sets with finite-sample
coverage guarantees under the assumption of exchangeability.
This assumption, however, is often violated in sequential
or time-dependent settings, such as the forecasting of time series,
where the distribution of observations may evolve over time.

To address this limitation, \citet{gibbs2021adaptive} introduced
a setting (and an algorithm) called adaptive conformal inference [ACI], which extends
conformal inference beyond exchangeability.
The underlying idea is to sequentially update the (levels of the) conformal sets
in response to observed outcomes to ensure that
empirical coverage converges to the desired level
for arbitrary sequences.
The ACI algorithm and its extensions satisfy this property, called \emph{validity},
but they generally lack theoretical guarantees on the
\emph{efficiency} of the prediction sets,
notably for prediction intervals, for which efficiency is measured
in terms of their average lengths.

In this work, we address validity and efficiency
jointly by formulating adaptive conformal inference
as a repeated two-player finite game with
vector-valued payoffs that
simultaneously encode these two objectives.
Building on this perspective, we design a
strategy that ensures validity while
achieving a setting-dependent notion of efficiency.
Our approach relies on
an extension of approachability called
opportunistic approachability
\citep{bernstein2014opportunistic},
and is ``best of many worlds'' in the sense that this notion of efficiency
recovers the relevant guarantees in the extreme settings
formed by exchangeable data and adversarial sequences,
while offering guarantees in intermediate settings
that arise in practical applications such as the forecasting
of time series.

\subsection{Related Works}


\paragraph{ACI and online conformal prediction.}
ACI maintains coverage by updating the conformal sets
via online gradient descent \citep{gibbs2021adaptive}.
Subsequent works largely adopted this online optimization
perspective, with a particular emphasis on the step-size schedule.
\citet{zaffran2022adaptive} resorts to aggregation of experts,
\citet{angelopoulos2024online} study decaying step sizes,
and both \citet{gibbs2024conformal} and \citet{bhatnagar2023improved}
establish connections to adaptive regret minimization.

Within this literature, two parameterizations have emerged.
The original ACI updates the miscoverage levels $\alpha_t$
used to compute empirical quantiles of conformity scores,
whereas online conformal prediction
updates the conformity threshold directly,
bypassing the quantile computation~\citep{bhatnagar2023improved,
angelopoulos2024online}.
Despite this distinction,
both approaches rely on the same underlying online optimization principle,
suggesting that the asymptotic validity achieved depends only weakly on the
specific construction of conformal sets \citep{szabadvary2025beyond}.
Other extensions beyond exchangeability include proportional-integral-derivative [PID] control,
which combines quantile tracking, error integration,
and quantile forecasting to adapt the prediction
sets online \citep{angelopoulos2023conformal},
and reweighting schemes that assign larger
weights to recent observations in the empirical distribution of conformity
scores \citep{barber2023conformal}.
Closer to our approach, \citet{cauchois2024robust} handle distribution shifts
by estimating their magnitude and by adjusting the targeted miscoverage levels accordingly.

\paragraph{Validity and something else.}
Recent work has aimed to extend the guarantees offered by ACI,
typically by offering strategies that guarantee
both validity and something else.
For instance, \citet{angelopoulos2025gradient}
introduce the notion of gradient equilibrium
for which ACI is a special case,
and show that the broader class of
strategies at gradient equilibrium
enjoy coverage properties.
\citet{ramalingam2025relationship}
investigate the relationship between pinball-loss regret
and validity and provide strategies controlling both notions.
Pinball-loss regret is frequently used as a proxy for efficiency,
but lacks interpretability in terms
of the average length of prediction intervals.

Other works focus more directly on efficiency,
measured by the average lengths of the prediction intervals
output.
\citet{zaffran2022adaptive} analyze the efficiency of
the original ACI procedure
in the exchangeable and autoregressive settings,
and show that it is generally suboptimal in terms of efficiency
(see Section~\ref{sec:exch-main} for a detailed discussion on their results).
In the adversarial setting, \citet{srinivas2025online}
establish impossibility results,
proving that no nontrivial efficiency
guarantees are attainable without additional assumptions.
This result motivates the consideration of the
assumption of $Q$--restricted opponent
introduced by \citet{bernstein2014opportunistic},
for which we build a meaningful theory of efficiency
under validity constraint.

The key tool to do so is related to repeated finite games
and is called approachability theory. We review it next.

\paragraph{Approachability theory.}
\citet{blackwell} introduced the concept
of approachability as the possibility,
for a learner, to force the average payoffs to converge
to a given target set $\cC$, irrespective of the opponent.
Approachability is a foundational tool in the analysis
of repeated finite games and characterizes tractable
online learning problems. It was used to show the
existence of
no-regret strategies
(\citealp{hart2000simple}, see also
\citealp{CBL03} and \citealp{abernethy2011blackwell})
and of calibrated forecasters
(see \citealp{foster1999proof}, and see also \citealp{mannor2010geometric}).

In the context of adaptive conformal inference,
\citet{marx2025calibrated} mention that
asymptotic validity can be seen as an approachability problem,
with no further analysis of efficiency.

Approachability is generally studied
through the lens of the primal condition,
which provides an explicit approachability strategy.
Our focus, however, lies on the dual condition instead,
and is a response-based approachability strategy.
Response-based strategies have been studied
by \citet{bernstein2015response} (see also \citealp{mannor2014approachability}),
and we rely especially on the opportunistic approachability theory developed
by~\citet{bernstein2014opportunistic}.
Opportunistic approachability
is a calibration-based approachability algorithm
(see \citealp{perchet2009calibration} for a general study of such algorithms)
allowing the learner to take advantage of potential restrictions in the adversary's play.

\subsection{Contributions and Outline}

Our contributions are twofold.

\paragraph{(i) A game-theoretical formulation of adaptive conformal inference [ACI].}
We first formulate ACI as a repeated game with continuous action sets and then
provide a reformulation thereof as a repeated game with finite action sets
and vector-valued payoffs that encode simultaneously
the validity and efficiency criteria.

\paragraph{(ii) Efficiency guarantees under validity constraints.}
We then state a strategy, which is an instance of the strategy by
\citet{bernstein2014opportunistic}, that achieves
the validity guarantee for any data (be it stochastic or
picked by a malicious opponent) while providing
a ``best of many worlds'' efficiency guarantee.
Indeed, while ignoring the degree of stochasticity of
the online data,
this strategy recovers the relevant guarantees
in exchangeable and adversarial settings,
as well as in intermediate regimes that arise
in typical applications such as the forecasting of time series.
\medskip

The rest of the paper is organized as follows.
Section~\ref{sec:setting}, introduces the setting and
reformulates ACI as a repeated game with finite action
sets and vector-valued payoffs.
Section~\ref{sec:objectives},
presents the validity and efficiency objectives
and provides informal statements of the main theoretical results.
Section~\ref{sec:main-results-OA}
is devoted to presenting and analyzing in
general the opportunistic approachability
strategy that achieves these objectives.
Section~\ref{ex:Q-restr} focuses on more specific examples
of settings encompassed by the framework introduced,
illustrating its practical relevance.
Section~\ref{sec:boaci-epscalibr}
mentions a computationally more tractable variant
of the strategy based on $\epsilon$-calibrated forecasters,
and establishes the corresponding theoretical guarantees.
Finally, the appendices provide the proofs of all
results stated in the main body, as well as some background
on (opportunistic) approachability and calibration.

\paragraph{Notation.}
For two integers $1 \leq n \leq m$, let $\llbracket n,m \rrbracket = \{n, n+1, \ldots, m\}$.
For an integer $n \geq 1$, let $[n] = \llbracket 1,n \rrbracket = \{1, \dots, n\}$.
We denote by $\lceil x \rceil$ the upper integer part of the real number $x$.
    The Lebesgue measure on $\mathbb{R}$ is denoted by $\Leb$. The Dirac mass on a point $z$ is denoted by $\bdelta_z$.
    The indicator function of an event $A$ is denoted by $\indic_A$.
    The distance over $\R^d$ is the Euclidean norm $\Arrowvert \,\cdot\, \Arrowvert$, and the corresponding inner
    product is denoted by~$\langle \,\cdot\, , \,\cdot\, \rangle$.
    The convex hull generated by a set $S$ is denoted by $\conv(S)$. The interior of a set $S$ is denoted by $\interior(S)$ and its closure by $\cl(S)$.
    For $u \in \R^d$ and a closed subset $S \subseteq \R^d$, the distance of $u$ to $S$ is defined by $d(u,S) =
    \smash{\displaystyle{\min_{s \in S} \Arrowvert u-s \Arrowvert}}$. A sequence $(u_t)_{t \geq 1}$ converges to $S$,
    which we denote $u_t \to S$, if $d(u_t,S) \to 0$.

    The set of all probability measures on a set $S$ is denoted by $\Delta(S)$.
    The total variation distance between two probability distributions $\S$ and $\V$ is denoted by $\dtv(\S,\V)$.
    When the underlying distribution is ambiguous, we write $\P_{s \sim \S}$ and $\E_{s \sim \S}$ to denote the probability
    and the expectation with respect to an independent draw $s$ from $\S$.
    We identify the set of signed distributions over
    sets of $d$ elements with the simplex of $\R^d$.

\section{First Contribution: \\
\phantom{2~}Adaptive Conformal Inference [ACI] as a Repeated Game} \label{sec:setting}

We (ultimately) work in a sequential prediction setting.
At each round $t \in [T]$, after observing
vector-valued covariates $x_t \in \X$,
where $\X \subseteq \R^p$,
the learner outputs a prediction set $C(x_t)$
for the unobserved scalar outcome $y_t \in \R$,
which is then revealed before proceeding to the next round.
A historical dataset $(x_t,y_t)_{t \in \cD}$,
indexed by a finite set $\cD$, is available
and used to perform conformal inference and
construct the prediction sets.
Note the different sets of indices, $\cD$
and $[T]$, for historical and test data.

We first recall two settings of conformal inference:
a classic batch setting known as split conformal prediction,
where prediction sets are actually constructed
simultaneously for the batch $[T]$ of test points,
and an online setting
known as adaptive conformal inference,
where prediction sets are produced sequentially as new outcomes are revealed.
Although the batch setting is not the focus of this paper,
we briefly recall it for the sake of a more pedagogical exposition.

\subsection{Batch Setting of Conformal Inference: Split Conformal Prediction [SCP]}
\label{sec:SCP}

    In split conformal prediction~\citep{vovk2005algorithmic,lei2018distribution},
    the historical data is split into a training set, indexed by $\Dtrain$
    and a calibration set, indexed by $\Dcalib$, of respective cardinalities
    $T_{\train}$ and $T_{\calib}$.
    A regression function $\widehat{\mu}$
    is learned on the training set while
    non-conformity scores $\{s_t : t \in \Dcalib \}$,
    measuring how well $y_t$ is predicted by $\widehat{\mu}(x_t)$,
    are computed on the calibration set.
    We focus on the standard choice of absolute residuals
    as non-conformity scores, which leads to conformal sets that are intervals.
    Thus, for all $t \in \Dcalib$,
    the non-conformity score is $s_t = |y_t - \widehat{\mu}(x_t)|$ and,
    assuming that the scores are almost surely distinct, we denote their order statistics by $s_{(1)} < \cdots < s_{(T_{\calib})}$.
    We further set $s_{(0)} = 0$ and $s_{(T_{\calib}+1)} = L$,
    where $L < \infty$ is a known upper bound on the absolute residuals
    such that $s_{(T_{\calib})} < L$; see Remark~\ref{rk:L}.

    We denote the batch of test points by
    $(x_t,y_t)_{t\in\Dtest}$, where $\Dtest=[T]$.
    For each test point, we define the corresponding
    non-conformity score as $s_t = |y_t-\widehat{\mu}(x_t)|$,
    following the same definition as for the calibration scores.

    For each $t \in [T]$ and for a given miscoverage rate $\alpha \in [0,1]$,
    the conformal prediction interval is then
	\[
	C_{1-\alpha}(x_t) = \Big[\widehat{\mu}(x_t) \pm  s_{\big(\lceil(T_{\calib}+1)(1-\alpha)\rceil \big)} \Big]\,.
	\]
	By construction, conformal prediction intervals are nested in $\alpha$, i.e., $1-\alpha \leq 1-\alpha^{\prime}$ implies $C_{1-\alpha}(x_t) \subseteq C_{1-\alpha^{\prime}}(x_t)$, a direct consequence of $s_{\big(\lceil(T_{\calib}+1)(1-\alpha)\rceil \big)} \leq s_{\big(\lceil(T_{\calib}+1)(1-\alpha^{\prime})\rceil \big)}$.

	\begin{fact}
    \label{fact:2}
    For all $x \in \X$,
    the family $\bigl( C_{1-\alpha}(x) \bigr)_{\alpha \in [0,1]}$ of conformal intervals
    has $C_0(x) = \emptyset$ as smallest element, with length $\Leb\big(C_0(x) \big) = 0$,
    and $C_1(x) = \big[\widehat{\mu}(x) \pm  L\big]$ as largest element,
    with length $\Leb\big(C_1(x)\big) = 2L$.
    These two extreme cases
    corresponding to $1-\alpha\in\{0,1\}$
    will play a role in the subsequent sections.
    \end{fact}
	
 We can now state the central result from \cite{vovk2005algorithmic}, that holds under
 a stochastic assumption: if calibration scores $(s_t)_{t \in \Dcalib}$ and test scores $(s_t)_{t \in [T]}$ are exchangeable\footnote{Random variables $U_1,\ldots,U_n$
 are said exchangeable if for all permutations $\sigma : [n] \to [n]$,
 the $n$--tuples $(U_1,\ldots,U_n)$ and $(U_{\sigma(1)},\ldots,U_{\sigma(n)})$ have
 the same distribution.} (in particular, if the corresponding pairs
 of covariates--outcomes are exchangeable), the above procedure guarantees that for all $\alpha \in [0,1]$,
	\begin{equation}
    \label{eq:guar-SCP}
	\forall t \in [T] \, ,\qquad \mathbb{P}\bigl( y_t \in C_{1-\alpha}(x_t) \bigr) \geq 1-\alpha\,.
    \end{equation}
    Actually, up to some minor assumption of calibration scores being all distinct, the
    above probability lies in the interval between $1-\alpha$ and $1-\alpha+1/(T_{\calib}+1)$,
    that is, is approximately equal to $1-\alpha$.

    \begin{remark} \label{rk:L}
       The literature typically sets $s_{(T_{\calib}+1)} = +\infty$, which is fine to study validity
       but complicates efficiency analyses \citep{zaffran2022adaptive}. We rather assume
       that there exists a known constant $L$ such that, for all $t \in \Dcalib \cup [T]$, one has $\mathbb{P}\big( |y_t - \widehat{\mu}(x_t)| < L \big) = 1$.
       In practical applications, \citet{zaffran2022adaptive} suggest to set $L = 2 \max \{s_t : t \in \Dcalib \}$.
    \end{remark}

\subsection{Online Setting of Conformal Inference: Adaptive Conformal Inference [ACI] \hspace{-2cm} \ }

In the online setting, no stochastic assumption is made on the
calibration and test scores,
and the learner can adapt the choice of prediction intervals
sequentially based on past observations.

A natural extension of the setting of conformal inference with these
two characteristics was introduced by \citet{gibbs2021adaptive} under the name
of adaptive conformal inference.
For theoretical purposes, we consider the simplest
online setting of \citet{gibbs2021adaptive},
and assume that the training set and calibration set
are fixed, exactly as in SCP (see Section~\ref{sec:SCP}).

At each round, these fixed sets induce a
family of nested prediction intervals of the form
\begin{equation}
\label{eq:C-int-ACI}
C_{1-\alpha}(\,\cdot\,) = \Big[\widehat{\mu}(\,\cdot\,) \pm  s_{\big(\lceil(T_{\calib}+1)(1-\alpha)\rceil \big)} \Big]\,.
\end{equation}
Then, test data $(x_t, y_t)_{t \in [T]}$ is handled in an online fashion: at round $t \in [T]$,
the covariates $x_t$ are revealed, the learner picks some miscoverage rate $\alpha_t \in [0,1]$
based on past information and outputs the conformal set $C_{1-\alpha_t}(x_t)$.
After that, the outcome $y_t$ is revealed and the learner observes whether $y_t \in C_{1-\alpha_t}(x_t)$
or $y_t \not\in C_{1-\alpha_t}(x_t)$.

\begin{remark}
More complex settings of online conformal prediction exist,
where the training and calibration sets, and thus the form of the conformal intervals,
are updated over time, which is preferred in practice; see \citet{gibbs2024conformal}.
\end{remark}

No assumption is made on the generating mechanism
of test data $(x_t, y_t)_{t \in [T]}$.
It can even be thought of as being picked
by a malicious adversary.
The probabilistic guarantee~\eqref{eq:guar-SCP}
holding for each test point is thus replaced
by an average guarantee over all test points,
as stated next.

The primary objective of the learner is to ensure coverage for a fraction about $1-\alpha$ of the test data points---a criterion called
\emph{validity}:
\begin{equation}
\label{eq:guar-ACI}
\mbox{ensure that} \qquad
\frac{1}{T}\sum_{t=1}^T \indic_{y_t \in C_{1-\alpha_t}(x_t)} \qquad \mbox{is larger than about} \quad 1-\alpha.
\end{equation}
In this article, we will also be interested in a secondary objective consisting of
keeping the intervals $C_{1-\alpha_t}(x_t)$ as small as possible on average---a criterion called \emph{efficiency}.
This will of course call for a validity criteria
of the form ``about $1-\alpha$'' rather than just
``larger than about $1-\alpha$'' in~\eqref{eq:guar-ACI}.

In a nutshell, in the version of online conformal
prediction described above,
the learner picks a sequence of data-dependent miscoverage levels $\alpha_t \in [0,1]$
to ensure approximate coverage on average.
This setting is now commonly referred to as adaptive conformal inference [ACI], although the term originally denotes
the specific algorithm of \citet[restated in Example~\ref{ex:ACI} below]{gibbs2021adaptive}.

\subsection{Adaptive Conformal Inference [ACI] as a Repeated Finite Game}
\label{sec:ACI-repeated-game}

Our first main result is a game-theoretic interpretation of ACI as a repeated finite game.
To do so, we first formulate ACI as a repeated game with continuous action sets
and then construct a reformulation thereof as a repeated game with finite action sets.

\paragraph{Direct formulation of ACI as a game with continuous action sets.}
    We first see ACI as a game where, at stage $t \geq 1$, the learner picks a prediction set $C_{1-\alpha_t}(x_t)$,
    or alternatively, a miscoverage level $\alpha_t \in [0,1]$, while the opponent picks the observation~$y_t \in \R$.
    In terms of validity, the question is whether $y_t \in C_{1-\alpha_t}(x_t)$ or $y_t \not\in C_{1-\alpha_t}(x_t)$.
    In terms of efficiency, we evaluate the prediction set through its length $\Leb\,\bigl(C_{1-\alpha_t}(x_t)\bigr)$.

    \begin{remark}
    Note that by construction, the length is independent of the covariate $x_t \in \X$, which is why we will simply write
    $\Leb\,(C_{1-\alpha_t})$ in the sequel.
    \end{remark}

\paragraph{Two helpful properties to reduce action sets to finite sets.}
    We next state two properties that are helpful to reformulate this game with continuous action sets into a finite game.

    We first show that the discrete nature of the possible lengths of the prediction sets, dictated by the scores in the calibration set
    (fixed throughout),
    entails that the learner actually picks miscoverage levels $a_t$ in the following discrete set $\cA$,
    instead of general miscoverage levels $\alpha_t \in [0,1]$:
    \[
    \cA = \biggl\{ \frac{k}{T_{\calib}+1} : k \in \llbracket 0, \, T_{\calib}+1 \rrbracket \biggr\}\,.
    \]
    Indeed, we define the rounding function $r:[0,1] \to \cA$ of $[0,1]$ onto the finite grid $\cA$ as
    \[
    r : \alpha \in [0,1] \longmapsto 1-\dfrac{\big\lceil(T_{\calib}+1)(1-\alpha)\big\rceil}{T_{\calib}+1}\,,
    \]
    i.e., $r(\alpha)$ is the largest value inferior than or equal to $\alpha$ in $\cA$. The following
    property shows that we can assume, with no loss of generality, namely, up to replacing
    $\alpha_t$ by $a_t = r(\alpha_t)$, that the learner picks miscoverage levels $a_t \in \cA$.
    The property holds by the construction~\eqref{eq:C-int-ACI} of prediction intervals based on order statistics of non-conformity scores.

    \begin{property}\label{prop:rounding}
    For all $\alpha \in [0,1]$, for all $x \in \X$, we have $C_{1-\alpha}(x) = C_{1-r(\alpha)}(x)$.
    \end{property}

    Second, we show that it is equivalent for the opponent to choose an outcome $y_t \in \R$ or some action $b_t$
    in a finite set $\cB$ to be specified.
    Indeed, the question is whether the outcome $y_t$ picked by the opponent belongs, or not, to $C_{1-a_t}(x_t)$.
    The property stated in the next lemma, introduced by \citet{gibbs2024conformal}, and slightly reworked here, shows that given $\widehat{\mu}(x_t)$
    and the historical data,
    the observation $y_t$ can be replaced by an action $b_t \in \cB$,
    where
    \[
    \cB = \biggl\{ \frac{k}{T_{\calib}+1} : k \in \llbracket 1, \, T_{\calib}+1 \rrbracket \biggr\} = \cA \setminus \{ 0 \}\,.
    \]
    Note that $\cB$ does not contain the value $0$, unlike $\cA$.
	
    \begin{lemma}[{from \citealp{gibbs2024conformal},
    with a slight adaptation}]\label{lm:bt} The following equalities define a scalar $b_t \in \cB$,
    thus $b_t > 0$:
    \begin{align*}
        b_t = \sup\bigl\{\beta \in [0,1] : y_t \in C_{1-\beta}(x_t)\bigr\} = \min\bigl\{\beta \in \cB : y_t \notin C_{1-\beta}(x_t)\bigr\}
        \,.
    \end{align*}
    As a consequence, $\indic_{y_t \notin C_{1-a_t}(x_t)} = \indic_{b_t \leq a_t}$.
    \end{lemma}

    \begin{proof}
        Fix $t \geq 1$ and define $b_t$ as the supremum of the set $\{\beta \in [0,1] : y_t \in C_{1-\beta}(x_t)\}$,
        which is non-empty since $y_t \in C_1(x_t)$ by Fact~\ref{fact:2}.
        That conformal prediction sets are nested ensures that
        \begin{equation}
        \label{eq:sup-inf}
        b_t = \sup\bigl\{\beta \in [0,1] : y_t \in C_{1-\beta}(x_t)\bigr\} = \inf\bigl\{\beta \in [0,1] : y_t \notin C_{1-\beta}(x_t)\bigr\} \,,
        \end{equation}
        and we show next that the infimum is a minimum.

        First, if $b_t = 1$, then $b_t \in \cB$; as $y_t \notin C_{1-1}(x_t) = \emptyset$, the infimum in~\eqref{eq:sup-inf} is a minimum,
        at the value $\beta = 1$.

        Second, if $b_t < 1$, then there exists $\beta \in (b_t,1]$ such that $r(\beta) = r(b_t)$.
        On the one hand, by Property~\ref{prop:rounding}, we have $C_{1-\beta}(x_t) = C_{1-b_t}(x_t)$.
        On the other hand, since $\beta > b_t$ and by definition of $b_t$ as a supremum,
        $y_t \notin C_{1-\beta}(x_t)$. Thus, $y_t \notin C_{1-b_t}(x_t)$, which proves
        that the infimum in~\eqref{eq:sup-inf} is a minimum, at the value $\beta = b_t$.

        Finally, we show that $b_t \in \cB$.
        Note that $b_t > 0$ since $y_t \notin C_{1-b_t}(x_t)$ whereas $y_t \in C_{1}(x_t)$.
        On the one hand, the characterization of $b_t > 0$ as an infimum implies that $y_t \in C_{1-\beta}(x_t)$ for all $0 \leq \beta < b_t$.
        On the other hand, the rounding function is such that $0 \leq r(b_t) \leq b_t$ and $C_{1-b_t}(x_t) = C_{1-r(b_t)}(x_t)$.
        Therefore, using again that $y_t \notin C_{1-b_t}(x_t)$, we obtain $b_t = r(b_t)$,
        and thus $b_t \in \cB$, since $\cB = r\bigl([0,1]\bigr) \setminus \{0\}$.

        The consequence stated at the end of the lemma follows from the characterization
        of $b_t$ as a minimum.
    \end{proof}

    The lemma above allows to rewrite the update scheme of the ACI algorithm from \citet{gibbs2021adaptive} in terms of the player's and opponent's actions as follows.

    \begin{example}\label{ex:ACI} The update scheme of the ACI algorithm from \citet{gibbs2021adaptive} can be written as
	\begin{equation}\label{eq:ACI}
		\alpha_{t+1} = \alpha_t - \gamma \big(\indic_{b_t \leq a_t}-\alpha\big)
		\quad \mbox{and} \quad
		a_{t+1} =
		\begin{cases}
			\quad \quad r(\alpha_{t+1}) \,, & \mbox{if $\alpha_{t+1}\in[0,1]$}, \\
			\min\bigl\{1, \max\{\alpha_{t+1},0\}\bigr\} \,, & \mbox{otherwise},
		\end{cases}
	\end{equation}
	where $\alpha_{t+1}$ serves as an intermediate value for choosing $a_{t+1} \in \cA$, and $\gamma > 0$ is a fixed step size.
    \end{example}

    \paragraph{Reformulation of ACI as a repeated finite game.}
    Based on all elements above, we now reformulate ACI as the following repeated finite game.
    Note that the reward obtained is formed by a pair in $\{0,1\} \times [0,2L]$, as we are interested
    in both validity and efficiency guarantees.

    \begin{theorem}[first main result] \label{th:finite-game}
    The setting of adaptive conformal inference [ACI] may be formulated as a repeated finite game,
    where at each round $t \geq 1$, the learner picks $a_t \in \cA$, the opponent picks $b_t \in \cB$, the
    learner receives the reward $m(a_t,b_t)$ given by the vector-valued payoff function $m : \cA \times \cB \to \{0,1\} \times [0,2L]$
    defined by
    \begin{equation}
    \label{eq:finite-game}
    m: (a,b) \in \cA\times \cB\longmapsto m(a,b) = \begin{pmatrix} \indic_{b \leq a } \\ \Leb\,(C_{1-a}) \end{pmatrix}.
    \end{equation}
    \end{theorem}

\subsection{Links Between the Batch and Online Settings (Between SCP and ACI)}

    The above formalism is suited for adversarial sequences of test observations $(x_t, y_t)_{t \in [T]}$ but
    it also encompasses the exchangeable setting:
    the latter corresponds to an opponent player picking actions $b_t$ uniformly at random in $\cB$, as stated next.
    We denote by $\bq_{\unif}$ the uniform distribution on $\cB$:
    \[
    \bq_{\unif} = \frac{1}{T_{\calib}+1} \sum_{b \in \cB} \bdelta_b\,.
    \]

\begin{lemma} \label{lm:qunif}
    If the calibration scores $(s_k)_{k \in \Dcalib}$ and the test score $s_t$ are exchangeable
    and almost-surely distinct,
    then $b_t \sim \bq_{\unif}$.
\end{lemma}

The proof of Lemma~\ref{lm:qunif} relies on the following
result, not requiring any stochastic assumption.

    \begin{lemma} \label{lm:factbt}
    When the calibration scores $(s_k)_{k \in \Dcalib}$ are all distinct,
    the scalar $b_t$ defined in Lemma~\ref{lm:bt} satisfies
    \begin{align*}
    \forall b \in \cB, \qquad \{b_t = b\} = \Big\{s_{\big( (T_{\calib}+1)(1-b)\big)} < s_t \leq s_{\big( (T_{\calib}+1)(1-b) +1\big)}\Big\} \,.
    \end{align*}
    \end{lemma}

    \begin{proof}
    The assumption of distinct values of the calibration scores is used to avoid ties:
    it entails that for all $b \in \cB$,
    \[
    \big\{y_t \notin C_{1-b}(x_t)\big\} = \Big\{s_t > s_{\big( (T_{\calib}+1)(1-b)\big)}\Big\} \,.
    \]
    Thus, by the definition of $b_t$ as some minimum in Lemma~\ref{lm:bt} (and since $0$ does not belong to $\cB$),
    for all $b \in \cB$,
    \begin{align*}
    \bigl\{b_t = b\bigr\} &= \bigl\{ y_t \notin C_{1-b}(x_t) \bigr\} \setminus \bigl\{ y_t \notin C_{1-b+1/(T_{\calib}+1)}(x_t) \bigr\} \\
    &= \Big\{s_{\big( (T_{\calib}+1)(1-b)\big)} < s_t \leq s_{\big( (T_{\calib}+1)(1-b) +1\big)}\Big\} \,. \\[-1.5cm]
    \end{align*}
    \
\end{proof}

We are now able to prove Lemma~\ref{lm:qunif}. \medskip

\begin{proof} \textbf{of Lemma~\ref{lm:qunif}}
    By exchangeability and the assumption of almost-surely distinct scores, $s_t$ is equally likely to fall into any of the $T_{\calib} +1$ intervals defined by the scores computed on the calibration set:
    \begin{align*}
    \forall k \in \llbracket 0 , T_{\calib} \rrbracket, \qquad
    \mathbb{P}\big(s_t \in (s_{(k)}, s_{(k+1)}]\big) =
    \mathbb{P}\big(s_t \in (s_{(k)}, s_{(k+1)})\big) = \frac{1}{T_{\calib}+1} \,.
    \end{align*}
    Now, Lemma~\ref{lm:factbt} exactly ensures that each event $\{b_t = b\}$, where $b \in \cB$, is of the form
    of the events considered above,
    thus these $T_{\calib}+1$ events are equally likely.
\end{proof}

\section{Second Contribution: Efficiency under a Validity Constraint}
\label{sec:objectives}

We now leverage the reduction of ACI to the repeated finite game of Theorem~\ref{th:finite-game}
involving payoffs in terms of coverage and length of conformal intervals.

Our aims are twofold: \eqref{eq:coverage} achieve validity i.e.,
$y_t \in C_{1-a_t}(x_t)$ at least $100(1-\alpha)\%$ of the time
on average and in the limit, as stated earlier in~\eqref{eq:guar-ACI};
and \eqref{eq:length} produce efficient prediction sets,
i.e., prediction sets with small lengths,
on average and in the limit.

\paragraph{Formal statement of the objectives.}
These two objectives are now formally stated
based on a target length $\ell_\alpha$ as
	\begin{align}
	\mbox{\underline{Validity}: \qquad ensure} \qquad & \limsup_{T \to \infty} \frac{1}{T}\sum_{t=1}^T \indic_{y_t \notin C_{1-a_t}(x_t)}  = \limsup_{T \to \infty} \frac{1}{T}\sum_{t=1}^T \indic_{b_t \leq a_t} \leq \alpha \tag{$\star$} \,; \label{eq:coverage} \\
	\mbox{\underline{Efficiency}: \qquad ensure} \qquad & \limsup_{T \to \infty} \frac{1}{T}\sum_{t=1}^T \Leb\big(C_{1-a_t} \big) \leq \ell_\alpha \tag{$\star\star$} \label{eq:length} \, .
	\end{align}
The validity objective~\eqref{eq:coverage} is expressed
in terms of an empirical miscoverage smaller than $\alpha$,
which corresponds to an empirical coverage larger than $1-\alpha$.

We strive for a target length $\ell_\alpha$ as small as possible. This target will be smaller when stronger stochastic assumptions
are made on the feature-outcome pairs (or equivalently, on the scores).

\paragraph{Reformulation of these objectives as an approachability problem.}
With the vector-valued payoff function introduced in Theorem~\ref{th:finite-game}, we may reformulate
the objectives \eqref{eq:coverage} and \eqref{eq:length} as approaching
the closed convex set $[0,\alpha] \times [0,\ell_\alpha]$, i.e.,
\begin{multline}
\label{eq:appr-obj}
\mbox{\underline{Validity and Efficiency}:} \\ \mbox{ensure} \qquad
\frac{1}{T}\sum_{t=1}^T \begin{pmatrix} \indic_{b_t \leq a_t} \\ \Leb\,(C_{1-a_t}) \end{pmatrix}
\longrightarrow [0,\alpha] \times [0,\ell_\alpha] \qquad \mbox{as $T \to \infty$}, ~~~~ \tag{$\star\,\mbox{\&}\star\!\star$}
\end{multline}
in the sense that $d\bigl(\overline{m}_T, [0,\alpha] \times [0,\ell_\alpha] \bigr) \to 0$ as $T \to \infty$, where
\[
\overline{m}_T \eqdef
\frac{1}{T}\sum_{t=1}^T m(a_t,b_t) \eqdef \frac{1}{T}\sum_{t=1}^T \begin{pmatrix} \indic_{b_t \leq a_t} \\ \Leb\,(C_{1-a_t}) \end{pmatrix}\,.
\]
Convergences of the above form are called approachability results
and were first studied by \citet{blackwell}, see Appendix~\ref{sec:Blk}.


\paragraph{The question at hand, and a tool to answer it.}
The question at hand is therefore what target lengths $\ell_\alpha$
may be achieved depending on the data-generation mechanisms
considered.
We first discuss two extreme of such mechanisms: (i) the favorable exchangeable setting
of Definition~\ref{def:exchangeablesetting}, and (ii) the worst-case adversarial setting
of Definition~\ref{def:advsetting}.
We next consider a general framework based on the notion of $Q$\emph{--restricted opponent's play}, introduced by \citet{bernstein2014opportunistic}
and recalled in Definition~\ref{def:Q-restricted},
which encompasses practically relevant intermediate settings.

\citet{bernstein2014opportunistic} provide a theory called opportunistic approachability
to identify target sets that are approachable under the considered $Q$--restrictions; this theory will be the
key tool used to answer the question of setting the target lengths $\ell_\alpha$.

\subsection{Several Degrees of Stochasticity for the Opponent}
\label{sec:sev-degr-stoch}

We first define the two extreme data-generating mechanisms.
The assumption of exchangeability is a typical assumption
in the batch conformal prediction setting of Section~\ref{sec:SCP}.
It entails that the actions $b_t$ of the opponent in the
associated repeated finite game of Theorem~\ref{th:finite-game} are uniformly distributed.

    \begin{definition}[exchangeable setting] \label{def:exchangeablesetting}
    The exchangeable setting corresponds to the fact that at each round $t \geq 1$, calibration scores $(s_k)_{k \in \Dcalib}$ and the test score $s_t$ are exchangeable and almost-surely distinct.
    By Lemma~\ref{lm:qunif}, we have $b_t \sim \bq_{\unif}$.
    \end{definition}

The adversarial setting refers to not assuming anything on the test observations,
which, in turn, corresponds to not assuming anything on the actions $b_t$
in the associated repeated finite game of Theorem~\ref{th:finite-game}.

    \begin{definition}[adversarial setting] \label{def:advsetting}
    The adversarial setting corresponds to the fact that at each round $t \geq 1$, the opponent freely selects the test observations $(x_t,y_t)$, possibly based on the outcomes of past rounds (with the only constraint that the resulting test score $s_t$ lies in $[0,L]$).
    Therefore, nothing can be assumed on the test score $s_t \in [0,L]$
    or on the opponent's action $b_t$.
    \end{definition}

A framework interpolating the previous two settings
was introduced by \citet{bernstein2014opportunistic}, in the form of
$Q$--restricted opponent's play.
In the definition stated next,
$\bq_t$ denotes the distribution of the outcome $b_t$
in the associated repeated finite game of Theorem~\ref{th:finite-game}.
This distribution may depend on the results of the past rounds:
$\bq_t$ is a predictable random variable.

\begin{definition}[setting of $Q$--restricted opponent's play]
    \label{def:Q-restricted}
    Fix a subset $Q \subseteq \Delta(\cB)$.
    Let $\bq_t$ denote the distribution of $b_t$ at round $t \geq 1$,
    given the results of past rounds.
    We say that an opponent's play is $Q$--restricted whenever it satisfies
    \[
    \lim_{T \to \infty} \frac{1}{T}\sum_{t=1}^T d(\bq_t,Q) = 0 \as
    \]
    \end{definition}

Note that $Q$ is generally unknown to the learner, and that the Cesàro-mean convergence of Definition~\ref{def:Q-restricted} is a weaker condition than the convergence of $d(\bq_t,Q)$ to zero (and, in particular, the fact that $\bq_t \in Q$ for all $t \geq 1$).

The exchangeable setting and the adversarial setting
are two extreme cases thereof of
Definition~\ref{def:Q-restricted}, see Examples~\ref{ex:echangeable} and~\ref{ex:adv} below.

    \begin{example} \label{ex:echangeable}
    In the exchangeable setting, the opponent's play is, in particular, $\bigl\{\bq_{\unif}\bigr\}\!$--res\-tric\-ted, where $\bq_{\unif}$ is the uniform distribution on $\cB$ (see Lemma~\ref{lm:qunif}).
    Actually, $\bigl\{\bq_{\unif}\bigr\}\!$--res\-tric\-tion may also be achieved when the exchangeability assumption
    between the calibration scores and the test scores $s_t$ is violated for a sublinear number of rounds $t$.
    \end{example}

    \begin{example} \label{ex:adv}
    The adversarial setting imposes no condition on the opponent, which corresponds to a $\Delta(\cB)$--restricted player.
    \end{example}

Example~\ref{ex:adv} above explains why the assertion that the opponent's play is $Q$--restricted for some unknown $Q \subseteq \Delta(\cB)$
is not a real assumption: this is always the case. Sometimes (e.g., in Example~\ref{ex:echangeable}), this set $Q$ is much smaller than the set $\Delta(\cB)$
of all distributions.
Section~\ref{ex:Q-restr} will provide further examples of realistic small sets~$Q$.

\subsection{Informal Statements of Validity--Efficiency Results Achieved}

We state informally weaker results that the ones achieved in Section~\ref{sec:main-results-OA}:
here, convergences of $\overline{m}_T$, i.e., empirical average miscoverages and lengths,
to sets of the form $[0,\alpha] \times [0,\ell_{\alpha,Q}]$,
instead of the more precise convergences to subsets $\cC_{\alpha,Q}$ thereof as shown in Section~\ref{sec:main-results-OA}.

\paragraph{Achievable average length under a $Q$--restriction.}
Given a distribution $\bq = (q_b)_{b \in \cB}$ in $\Delta(\cB)$,
we denote by $\quant(\alpha,\bq)$ the largest action in $\cA$ satisfying, on average over $\bq$, the validity condition,
i.e., with miscoverage at most $\alpha$:
\begin{align}
\quant(\alpha,\bq) = \max\biggl\{ a \in \cA: \sum_{b \in \cB} q_b \, \indic_{b \leq a} \leq \alpha\biggr\} \label{eq:quant} \,.
\end{align}

\paragraph{Target length: worst-case length under $Q$.}
Conceptually, a natural target would be the worst-case length for a valid prediction set associated with a mixed action in $Q$, namely,
\[
\max_{\bq \in Q} \Leb\,\bigl(C_{1-\quant(\alpha,\bq)}\bigr) \,.
\]
However, potential discontinuities of $m^{\star}$ on the boundary of $Q$ may arise, which is why,
for technical reasons, the target considered is rather
(see Lemma~\ref{lm:targetgeneral})
\[
\ell_{\alpha,Q} = \lim_{\vphantom{d(\bq,Q)}\eta \downarrow 0} \max_{\,\bq\,:\,d(\bq,Q) \leq \eta}
\Leb\,\bigl(C_{1-\quant(\alpha,\bq)}\bigr) \,,
\]
where the maximum is taken over all distributions $\bq \in \Delta(\cB)$
in an $\eta$--neighborhood of $Q$.
When there is no ambiguity,
we suppress the statement $\bq \in \Delta(\cB)$
in the maximum,
as it is understood throughout.

Section~\ref{sec:main-results-OA} will establish formally approachability results
that are stronger than convergence to $[0,\alpha] \times [0, \ell_{\alpha,Q}]$
(the convergence will be to a smaller target denoted by $\cC_{\alpha,Q}$);
but for now, we state informally the former convergence, together
with its implications in various scenarios.

\begin{informal}[second main series of results]
\label{inf:2}
The BO-ACI strategy defined in Section~\ref{sec:main-results-OA},
tuned with $\alpha \in (0,1)$, ensures that
for all convex subsets $Q \subseteq \Delta(\cB)$,
for all $Q$--restricted opponents in the sense of
Definition~\ref{def:Q-restricted},
\[
\overline{m}_T = \frac{1}{T}\sum_{t=1}^T \begin{pmatrix} \indic_{b_t \leq a_t} \\ \Leb\,(C_{1-a_t}) \end{pmatrix}
\longrightarrow [0,\alpha] \times [0,\ell_{\alpha,Q}] \qquad \mbox{a.s., \quad as $T \to \infty$;}
\]
put differently, this strategy satisfies the validity~\eqref{eq:coverage} and efficiency~\eqref{eq:length}
objectives, with a target length $\ell_{\alpha,Q}$ whenever the opponent is $Q$--restricted.

This target length equals, for instance, \smallskip
\begin{itemize}
\item $\ell_{\alpha,\{\bq_{\unif}\}} = \Leb\,\bigl(C_{1-\alpha}\bigr)$
when the opponent is $\bigl\{\bq_{\unif}\bigr\}$--restricted and when $\alpha \not\in \cA$,
which is, in particular, the case when data is exchangeable (see Example~\ref{ex:echangeable}); \smallskip
\item $\Leb\,\bigl(C_{1-\alpha+1/(T_{\calib}+1)}\bigr)$ when data is almost exchangeable
and possibly comes with dependencies,
with a level of non-exchangeability that is sufficiently small; \smallskip
\item $\Leb\big(C_{1-a^{\S,\V}}\big)$ in the single shift distribution setting,
when test scores are i.i.d.\ according to a distribution $\V$
while calibration scores were i.i.d.\ according to a distribution $\S$, and where
$\bq^{\S,\V}$ denotes the distribution induced on the $b_t$ and $a^{\S,\V} = \quant\bigl(\alpha,\bq^{\S,\V}\bigr)$ is
the largest element $a \in \cA$ such that $C_{1-a}$ satisfies the $1-\alpha$ validity condition on average
under $\bq^{\S,\V}$; \smallskip
\item $\ell_{\alpha,\Delta(\cB)} = 2L$ in the adversarial setting
of Definition~\ref{def:advsetting} and Example~\ref{ex:adv}.
\end{itemize}
\end{informal}

\begin{remark}
\label{rk:Q-prim-strat}
We insist again that the considered strategy is unaware of the value of $Q$
and even of the existence of a potential $Q$--restriction,
but turns out to exploit it if it holds. This justifies
the name of ``opportunistic approachability'' for the underlying
theory used (see details in Appendix~\ref{sec:opp-app}).

If instead $Q$ were known to the learner,
one may want to adapt the classic, primal, approachability
strategy to directly approach $\C_Q$, and not only $\C$:
we perform this in Appendix~\ref{app:Qprimal}
for the sake of completeness, and also discuss the strategy thus obtained
in the exchangeable case, where $Q = \{ \bq_{\unif} \}$.
\end{remark}

\paragraph{Outline of the rest of the article.}
Section~\ref{sec:main-results-OA} proves the first part
of Informal Result~\ref{inf:2}, namely, the general convergence
result. Section~\ref{ex:Q-restr} provides concrete and realistic
scenarios of $Q$--restrictions, together with the computation of the associated
target length $\ell_{\alpha,Q}$; put differently,
Section~\ref{ex:Q-restr} goes over the second part
of Informal Result~\ref{inf:2}.

\section{Application of Approachability Theory under $Q$--Restricted Play \\ \phantom{4~}(a.k.a.\ Opportunistic Approachability)}
\label{sec:main-results-OA}

This section is devoted to formally state and prove the first part of
Informal Result~\eqref{inf:2}, consisting of the general convergence result.

We resort to the general theory of approachability under $Q$--restricted play from \citet{bernstein2014opportunistic},
also known as opportunistic approachability, which is recalled in Appendix~\ref{sec:opp-app}.
In this section, we restate specifically for ACI the results that can be achieved
(and provide in appendix short and direct proofs thereof, tailored to the specific
setting considered and bypassing some technical considerations needed for the general
theory by \citealp{bernstein2014opportunistic}).
To do so, we define in Section~\ref{sec:def-cC-alpha-Q} the target set
$\cC_{\alpha,Q}$ to be approached, we
provide in Section~\ref{sec:calibr} some reminder
on calibrated forecasters, which are used as a subroutine
for the main strategy, introduced and studied in
Section~\ref{sec:strat-BOACI}.

\subsection{Definition of Target Sets $\cC_{\alpha,Q}$ Tailored to $Q$--Restricted Opponents}
\label{sec:def-cC-alpha-Q}

\paragraph{The set $\cC_\alpha = [0,\alpha] \times [0,2L]$ is always approachable.}
The starting point is to note that we may always ensure validity~\eqref{eq:coverage}, e.g., by alternating
between $a_t = 0$ a fraction $1-\alpha$ of the time and $a_t = 1$ a fraction $\alpha$ of the time.
Given the boundedness assumptions issued (see Remark~\ref{rk:L}),
setting $\ell_\alpha = 2L$ corresponds to ignoring the efficiency constraint~\eqref{eq:length}.
Thus, the target set $\cC_\alpha = [0,\alpha] \times [0,2L]$ can always be approached
in the sense of~\eqref{eq:appr-obj}.

\paragraph{More ambitious target sets $\cC_{\alpha,Q}$ based on $Q$--restrictions.}
Now, the point of opportunistic approachability
is to design a general strategy of the learner not only approaching a given approachable set,
like $\cC_\alpha$, but actually approaching a smaller target set $\cC_{\alpha,Q} \subseteq \C_{\alpha}$ whenever the opponent is $Q$--restricted, with $Q \subseteq \Delta(\cB)$ being unknown to the learner.

More formally, we consider the bilinear extension of the vector payoff function $m$ defined in Theorem~\ref{th:finite-game}:
for all probability distributions $\bp = (p_a)_{a \in \cA} \in \Delta(\cA)$ and $\bq = (q_b)_{b \in \cB} \in \Delta(\cB)$,
\[
m(\bp,\bq) = \sum_{a\in \cA}\sum_{b\in \cB} p_{a} \, q_{b} \, m(a,b)
= \begin{pmatrix} \displaystyle{\sum_{a \in \cA} \sum_{b \in \cB} p_a \, q_b \, \indic_{b \leq a }} \vspace{.2cm} \\
\displaystyle{\sum_{a \in \cA} p_a \, \Leb\,(C_{1-a})} \end{pmatrix}.
\]
This extended payoff function $m$ and the target set $\cC_\alpha = [0,\alpha] \times [0,2L]$ are such that
\begin{align}
\forall \bq \in \Delta(\cB), \quad \exists \, p^{\star}(\bq) \in \Delta(\cA) \quad \mbox{s.t.} \quad
m\bigl(p^{\star}(\bq),\bq\bigr) \in \cC_\alpha \,; \label{eq:responseandgoalfunctions-m}
\end{align}
indeed, and for instance, a natural choice for $p^{\star}(\bq)$ is
the largest pure action in $\cA$ satisfying, on average over $\bq$, the validity condition with miscoverage [mscv] at most $\alpha$:
\begin{align}
p^\star(\bq) = \bdelta_{\quant(\alpha,\bq)}\,, \qquad \mbox{where} \qquad \quant(\alpha,\bq) = \max\biggl\{ a \in \cA: \sum_{b \in \cB} q_b \, \indic_{b \leq a} \leq \alpha\biggr\} \label{eq:response} \,.
\end{align}
Other choices are possible for $p^{\star}(\bq)$ but we fix the above one.

Equation~\eqref{eq:response} defines a so-called response function $p^\star : \Delta(\cB) \to \Delta(\cA)$,
only returning Dirac masses, from which, based on~\eqref{eq:responseandgoalfunctions-m}, we may in turn construct the goal function
\begin{equation}
m^{\star} : \bq \in \Delta(\cB) \longmapsto
m^{\star}(\bq) \eqdef m\bigl(p^{\star}(\bq),\bq\bigr)
=    \begin{pmatrix}
        \displaystyle{\sum_{b \in \cB} q_b \, \indic_{b \leq \quant(\alpha,\bq)}} \vspace{.2cm} \\
        \displaystyle{\Leb\,\bigl(C_{1-\quant(\alpha,\bq)}\bigr)}
    \end{pmatrix} \in \cC_\alpha = [0,\alpha] \times [0,2L] \,. \label{eq:goal}
\end{equation}

We exploit this choice of $m^\star$ to define smaller target sets $\cC_{\alpha,Q}$ for $Q$--restricted opponents. Namely,
$\cC_{\alpha,Q}$ is based on the image of $Q$ under the goal function $m^{\star}$,
with some additional twists required for some technical reasons: taking closures of convex hulls and considering the non-increasing limit of
images of $\eta$--neighborhoods of~$Q$.

\begin{definition}[closed convex image under $m^\star$]
    The closed convex image of a convex subset $Q \subseteq \Delta(\cB)$ under the goal function $m^{\star}$ is defined as
    \[
        \C_{\alpha,Q} = \bigcap_{\eta>0}\cl\biggl(\conv\Bigl(\bigl\{m^{\star}(\bq) : d(\bq,Q) \leq \eta\bigr\}\Bigr)\biggr) \,,
    \]
    and satisfies $\C_{\alpha,Q} \subseteq \cC_{\alpha}$.
\end{definition}

The approachability result of~\eqref{eq:appr-obj} will actually be in terms
of a convergence to $\C_{\alpha,Q}$. This will be sharper than a mere
convergence to the target set $[0,\alpha] \times [0,\ell_{\alpha,Q}]$ initially
considered in~\eqref{eq:appr-obj}, as follows from the lemma stated next.

\begin{lemma}\label{lm:targetgeneral}
    The closed convex image of a convex subset $Q \subseteq \Delta(\cB)$ under the goal function $m^{\star}$ defined in~\eqref{eq:goal} satisfies
    \[
    \C_{\alpha,Q} \subseteq [0,\alpha] \times [0,\ell_{\alpha,Q}]\,, \qquad \mbox{where} \qquad
    \ell_{\alpha,Q} = \lim_{\vphantom{d(\bq,Q)}\eta \downarrow 0} \max_{\,\bq\,:\,d(\bq,Q) \leq \eta}
\Leb\,\bigl(C_{1-\quant(\alpha,\bq)}\bigr)\,.
    \]
\end{lemma}

\begin{proof}
We argue component by component: we denote by $\Pi_j : (u_1,u_2) \mapsto u_j$ the
extraction of the $j$--th component of a vector $(u_1,u_2) \in \R^2$, where $j \in \{1,2\}$.

For the first components, Equation~\eqref{eq:goal} states that $\Pi_1\bigl( m^\star(\bq) \bigr) \in [0,\alpha]$
for all $\bq \in \Delta(\cB)$; since $[0,\alpha]$ is a closed convex set, the definition of $\cC_{\alpha,Q}$
entails that $\Pi_1(\cC_{\alpha,Q}) \subseteq [0,\alpha]$.

For the second components, fix first $\eta > 0$.
Recall that $\Pi_2\bigl(m^{\star}(\bq)\bigr) = \Leb\,\bigl(C_{1-\quant(\alpha,\bq)}\bigr)$ for all $\bq \in \Delta(\cB)$,
where $\Leb\,\bigl(C_{1-\quant(\alpha,\,\cdot\,)}\bigr)$ takes finitely many values.
We thus have, for all $\bq \in \Delta(\cB)$ such that $d(\bq,Q) \leq \eta$,
\[
\Pi_2\bigl(m^{\star}(\bq)\bigr) \in \left[0, \,\, \max_{\bq' \,:\,d(\bq',Q) \leq \eta} \Leb\,\bigl(C_{1-\quant(\alpha,\bq')}\bigr) \right].
\]
The latter interval is a closed convex set, and
$\smash{\C_{\alpha,Q} \subseteq \cl\biggl(\conv\Bigl(\bigl\{m^{\star}(\bq) : d(\bq,Q) \leq \eta\bigr\}\Bigr)\biggr)}$,
so that
\[
\Pi_2(\cC_{\alpha,Q}) \subseteq \left[0, \,\, \max_{\bq' \,:\,d(\bq',Q) \leq \eta} \Leb\,\bigl(C_{1-\quant(\alpha,\bq')}\bigr) \right]\,.
\]
The proof is concluded by taking the non-increasing limit as $\eta \downarrow 0$.
\end{proof}

\subsection{Reminder on Calibration}
\label{sec:calibr}

The main strategy, described in Section~\ref{sec:strat-BOACI},
relies on a subroutine consisting in a calibrated forecaster.
This subsection therefore provides a reminder on calibration.
The setting and learning protocol for calibration are described next,
for the finite set $\cB$
of Section~\ref{sec:ACI-repeated-game}.

\paragraph{Setting of calibration and related objectives.}
At each round $t \geq 1$, the forecaster and the opponent simultaneously pick a mixed action $\bz_t \in \Delta(\cB)$
and a pure action $b_t \in \cB$, respectively, where $b_t$ is picked possibly independently at random based on a mixed action denoted by $\bq_t \in \Delta(\cB)$.
The mixed actions $\bz_t$ and $\bq_t$ may depend on the information gathered in the previous rounds.
At the end of the round, the forecaster observes $b_t$.

The goal of the forecaster is to produce calibrated forecasts, in the following sense.
For interpretations of these definitions and constructive proofs of existence, we refer the reader, e.g.,
to \citet{mannor2010geometric}, whose approach relies on Blackwell's approachability theorem,
and to the references cited therein (above all, \citealp{FoVo98}, \citealp{FuLe99},
\citealp{foster1999proof}).

\begin{definition}[calibrated forecaster]\label{def:calibrated}
A forecaster is said to be calibrated if, for every strategy of the opponent, the forecasts $(\bz_t)_{t \geq 1}$ output satisfy:
		\[
		\mbox{for all Borel--measurable sets} \ G \subseteq \Delta(\cB) \,, \qquad \lim_{T \to \infty} \left\lVert \frac{1}{T}\sum_{t=1}^T \indic_{\bz_t \in G} \, \bigl(\bz_t-\bdelta_{b_t}\bigr) \right\rVert = 0 \as
		\]
\end{definition}

There exist calibrated forecasters. They are typically constructed based on a sequence of $\epsilon$--calibrated forecasters, which we define
in Section~\ref{sec:boaci-epscalibr}.

The following lemma, stated by \citet[Lemma 6]{bernstein2014opportunistic} in the
case of a calibrated forecaster, guarantees that the $Q$--restriction property of the mixed actions
$\bq_t$ taken by the opponent transfers to calibrated forecasts $\bz_t$ thereof.

\begin{lemma} \label{lm:restricted-calibr}
Consider any convex subset $Q \subseteq \Delta(\cB)$, and any $Q$--restricted opponent. If a forecaster is calibrated, the forecasts $(\bz_t)_{t \geq 1}$ output satisfy, almost surely:
    \[
        \lim_{T \to \infty} \frac{1}{T}\sum_{t=1}^{T} d(\bz_t,Q) = 0\,.
    \]
\end{lemma}

For the sake of completeness, a proof of the lemma above (and of a similar result for $\epsilon$--calibrated forecasters)
may be found in Appendix~\ref{sec:proof-calibr-Q-restr}.

\subsection{Strategy BO-ACI: Blackwell Opportunistic Adaptive Conformal Inference}
\label{sec:strat-BOACI}

In this section, we present our algorithm, referred to as Blackwell Opportunistic Adaptive Conformal Inference [BO-ACI].
The strategy relies on an auxiliary forecaster $\mathcal{F}$, which at each round $t$, outputs, possibly at random, a probabilistic forecast $\bz_t \in \Delta(\cB)$ of the result $b_t \in \cB$ of the opponent's pick for $(x_t,y_t)$.
Based on this forecast, the learner then deterministically selects the pure action $a_t \in \cA$ such that $p^\star(\bz_t) = \bdelta_{a_t}$,
i.e., $a_t = \quant(\alpha,\bz_t)$, and finally outputs the prediction set $C_{1-a_t}(x_t)$.

The description above is summarized in Algorithm~\ref{algo:blackwell}. \medskip

It is a standard trick in sequential learning to use a calibrated auxiliary forecaster $\mathcal{F}$ and optimally respond to
the probabilistic forecasts output (through $p^\star$ in our strategy).
This approach was first introduced by \citet{perchet2009calibration} in the context of Blackwell approachability, and was later incorporated as a central component of the general construction
of opportunistically approaching strategies by \citet{bernstein2014opportunistic}.

The BO-ACI strategy of Algorithm~\ref{algo:blackwell} is a simple instance of these more general strategies.
In Appendix~\ref{sec:proofs-main} we propose a simple, direct,
and self-contained proof of its performance, as stated next in
Theorem~\ref{th:main}, exploiting the specificities of adaptive conformal inference.

\begin{algorithm}[t]
	\caption{\label{algo:blackwell} Blackwell Opportunistic Adaptive Conformal Inference [BO-ACI]}
	\begin{algorithmic}
		\STATE\textbf{Inputs:} miscoverage level $\alpha$; auxiliary sequential forecaster $\mathcal{F}$ with outputs in $\Delta(\cB)$
        \STATE\textbf{Initialization:} play $a_1 = \quant(\alpha,\bq_{\unif})$ \\
		\STATE\textbf{At time steps $t=2, 3, \, \ldots\,$:} \\
		\quad 1. Get from $\mathcal{F}$ a probabilistic forecast $\bz_t \in \Delta(\cB)$ of the opponent's pure action \\			
		\quad 2. Observe $x_t$, play $a_t = \quant(\alpha,\bz_t)$, and thus output the prediction set $C_{1-a_t}(x_t)$ \\
		\quad 3. Observe $y_t$, deduce the corresponding $b_t$ from Lemma~\ref{lm:bt}, and feed $\mathcal{F}$ with $b_t$
	\end{algorithmic}
\end{algorithm}

\begin{restatable}[Opportunistic approachability for adaptive conformal inference]{theorem}{thmain}\label{th:main}
    The BO-ACI strategy of Algorithm~\ref{algo:blackwell}, when run with an auxiliary calibrated forecaster~$\cF$
    and tuned with $\alpha \in (0,1)$, ensures that for all convex subsets $Q \subseteq \Delta(\cB)$, for all
    $Q$--restricted opponents in the sense of Definition~\ref{def:Q-restricted},
    \[
    \overline{m}_T = \frac{1}{T}\sum_{t=1}^T \begin{pmatrix} \indic_{b_t \leq a_t} \\ \Leb\,(C_{1-a_t}) \end{pmatrix}
    \longrightarrow \cC_{\alpha,Q} \qquad \mbox{a.s., \quad as $T \to \infty$.}
    \]
\end{restatable}

Theorem~\ref{th:main}, together with Lemma~\ref{lm:targetgeneral},
proves the first part of Informal Result~\ref{inf:2}.
Theorem~\ref{th:main} is proved in Appendix~\ref{sec:proof-th-main}:
we provide therein a concise, carefully
simplified and revisited, self-contained, proof
(by adapting the proof by \citealp{bernstein2014opportunistic}
to the specific application considered).

\section{Examples of Target Lengths $\ell_{\alpha,Q}$ Achieved Depending on the Degree of Stochasticity $Q$}
\label{ex:Q-restr}

This section is devoted to formally define the cases stated in the second part of
Informal Result~\eqref{inf:2} and review which target lengths $\ell_{\alpha,Q}$ may
be achieved in each case. We will not consider the same order as stated therein
and will rather proceed in increasing order of complexity of the modeling or of
numbers of concepts needed.

The various corollaries stated in this section are obtained by considering particular choices of $Q$ in Theorem~\ref{th:main},
each corresponding to one of the informal results presented earlier in the article.
These specializations illustrate how the general theorem recovers the different regimes of interest within a unified framework.

\paragraph{A useful tool.}
We first state a corollary of Theorem~\ref{th:main}
in the case of $Q = \{\bq\}$ being a singleton:
the average payoff vectors $\overline{m}_T$ obtained by the BO-ACI strategy then converge to a point whenever $m^\star$ is continuous at $\bq$.

\begin{corollary}
\label{cor:single}
    Con\-si\-der $\alpha \in (0,1)$ and $\bq \in \Delta(\cB)$ such that $m^{\star}$ is continuous at $\bq$.
    The BO-ACI strategy of Algorithm~\ref{algo:blackwell}, when run with an auxiliary calibrated forecaster~$\cF$, is such that, in particular,
    if the opponent is $\{\bq\}$--restricted, i.e.,
    \[
    \lim_{T \to \infty} \frac{1}{T} \sum_{t=1}^T \Arrowvert \bq_t - \bq \Arrowvert = 0\,, \qquad \mbox{then} \qquad \lim_{T \to \infty} \overline{m}_T =
    \lim_{T \to \infty} \frac{1}{T}\sum_{t=1}^T \begin{pmatrix} \indic_{b_t \leq a_t} \\ \Leb\,(C_{1-a_t}) \end{pmatrix} = m^{\star}(\bq) \as
    \]
\end{corollary}

\begin{proof}
    We apply Theorem~\ref{th:main} with $Q = \{\bq\}$; it suffices to note that
    \[
    \C_{\alpha,\{\bq\}} = \bigcap_{\eta > 0} \cl\biggl(\conv\Bigl(\bigl\{m^{\star}(\bq') : d(\bq',\bq) \leq \eta\bigr\}\Bigr)\biggr) = \bigl\{m^{\star}(\bq)\bigr\}
    \]
    by continuity of $m^{\star}$ at $\bq$.
\end{proof}

\subsection{Exchangeable Case}
\label{sec:exch-main}

The case of exchangeable data (see Definition~\ref{def:exchangeablesetting})
is encompassed in the case of a $\{ \bq_{\unif} \}$--restricted opponent, see
Example~\ref{ex:echangeable}. We deal with the latter, more general, case
and show in the course of the proof of Corollary~\ref{cor:main} that the average vector payoffs $\overline{m}_T$ converge to
the set $\C_{\alpha,\{\bq_{\unif}\}}$, and that the latter admits a closed-form expression
as a singleton:
\begin{equation}
\label{eq:Calphaunif}
\C_{\alpha,\{\bq_{\unif}\}} = \Bigl\{ \bigl( r(\alpha),\, \Leb\,(C_{1-\alpha}) \bigr)^{\!\transp} \Bigr\}\,.
\end{equation}
The miscoverage level is only $r(\alpha) \leq \alpha$, for technical reasons due to rounding.

Before we do so, and for the sake of comparison to the literature,
we introduce the concept of general strategies:
strategies that guarantee validity in all settings.
It of course suffices to ask this requirement in the adversarial setting.

    \begin{definition}[general strategies]
    \label{def:general}
    A strategy for outputting conformal prediction sets is call general at miscoverage level $\alpha$ if it ensures validity at this level in the adversarial setting:
    for all sequences of data $(x_t,y_t)_{t \geq 1}$ picked by an opponent,
    \[
    \limsup_{T \to \infty} \frac{1}{T}\sum_{t=1}^T \indic_{y_t \notin C_{1-a_t}(x_t)} \leq \alpha\,.
    \]
    \end{definition}

By Theorem~\ref{th:main}, together with Lemma~\ref{lm:targetgeneral}
and the consideration of $Q = \Delta(\cB)$,
the BO-ACI strategy of Algorithm~\ref{algo:blackwell} is a general strategy,
which turns out to also guarantee
the desirable objective in the exchangeable setting
given by controlling~\eqref{eq:length} by the length of~$C_{1-\alpha}$.
This is stated next in Corollary~\ref{cor:main}.

To the best of our knowledge,
the BO-ACI strategy of Algorithm~\ref{algo:blackwell}
is the first strategy to simultaneously achieve
these two properties.
In particular, no general strategy considered in the literature is known to
also satisfy the efficiency guarantee $\Leb\,\bigl(C_{1-\alpha}\bigr)$
in the exchangeable setting.
In fact, the existing analysis of the ACI algorithm (of \citealp{gibbs2021adaptive},
restated in Example~\ref{ex:ACI})
illustrates the opposite phenomenon:
\citet{zaffran2022adaptive} showed that it
yields a strictly larger empirical average length
than $\Leb\,(C_{1-\alpha})$ in the exchangeable setting.
The additional length depends linearly on
the step size $\gamma$ used in the ACI algorithm,
on the nominal miscoverage level $\alpha$,
and on the underlying distribution of the calibration scores.
They interpret the result as indicating that the ACI algorithm
of \citet{gibbs2021adaptive} adapts to a potentially adversarial opponent
but at the cost of a larger
average length in the exchangeable setting.
(We note however that the proof of this result relies on a simplifying assumption
introduced to bypass technical difficulties arising
from the finiteness of the calibration set,
and therefore considers a slightly different setting from ours.)

\begin{restatable}[exchangeable setting]{corollarry}{cormain}\label{cor:main}
    Con\-si\-der ${\alpha \notin \cA}$.
    The BO-ACI strategy of Algorithm~\ref{algo:blackwell}, when run with an auxiliary calibrated forecaster~$\cF$
    and tuned with $\alpha$,
    is a general strategy at miscoverage level $\alpha$ (see Definition~\ref{def:general}) and it also
    satisfies, against a $\{ \bq_{\unif} \}$--restricted opponent, and thus, in the exchangeable setting,
    the following efficiency guarantee:
    \begin{equation}
    \label{eq:corollaryexch}
    \lim_{T \to \infty} \frac{1}{T}\sum_{t=1}^T \Leb\big(C_{1-a_t} \big) = \Leb\big(C_{1-\alpha}\big) \as
    \end{equation}
\end{restatable}

The proof of Corollary~\ref{cor:main}
is located in Appendix~\ref{sec:proof-cor:main};
it consists of showing that
the continuity required in Corollary~\ref{cor:single} is always satisfied for $\alpha \notin \cA$,
and of establishing the closed-form expression~\eqref{eq:Calphaunif}.

    \begin{remark}\label{rk:informalexchangeable}
        If $\alpha \in \cA$, then the asymptotic efficiency guarantee in~\eqref{eq:corollaryexch} deteriorates to $\Leb\,\bigl(C_{1-\alpha+1/(T_{\calib}+1)}\bigr)$
        as an upper bound, due to a potential discontinuity jump.
        We obtain this result as a special case of Corollary~\ref{cor:mainepsilon}
        stated later in this article; see the comments after its statement.
    \end{remark}

\subsection{Some Intermediate Scenarios}
\label{sec:CSA}

The exchangeability considered in the previous section is a relaxation of i.i.d.\
data in the direction of lack of independence.
We now study a different relaxation, in the direction of lack of identical distribution,
under independence assumptions. More precisely, for all three scenarios
considered at the beginning of each subsection, we issue the following stochastic assumptions.

\paragraph{Common stochastic assumptions.}
We assume that the calibration scores are drawn i.i.d.\ from some (unknown) diffuse distribution
    $\S$ with support contained in $[0,L]$. At time $t \geq 1$, the test score $s_t$ is drawn independently at random according to some (unknown) distribution
    $\S_t$ that may depend on the previous observations, and with support contained in $[0,L]$.
    To emphasize the dependence on the distributions, we introduce the notation $\bq_t = \bq^{\S,\S_t}$ for the distribution of the opponent's action $b_t$ at
    time $t \geq 1$. In particular, given that $\S$ is diffuse and that scores are therefore almost-surely distinct,
    $\bq^{\S,\S} = \bq_{\unif}$ is the uniform distribution on $\cB$.

\paragraph{Relaxation of these stochastic assumptions.}
The stochastic assumptions above are only needed to motivate the
associated $Q$--restrictions of the opponent's play that will be
stated in Assumptions~\ref{ass:single} and~\ref{ass:Qeps}.
The latter will cover more general cases,
not requiring strict independence assumptions.

\subsubsection{Single Distributional Shift}

This section studies, in particular, a setting of single distributional shift.
We assume, in this setting, that $\S_t = \V$ for all $t \geq 1$
    for some constant distribution $\V$, i.e., conditionally on the training and calibration data, the test scores are i.i.d.\ according to $\V$.
    This setting covers both the fully i.i.d.\ case where $\V = \S$ and situations where the test distribution $\V$ differs from the calibration
    distribution $\S$ but is stationary. It also includes cases where calibration data may have been inadvertently incorporated into the
    training process, thereby altering the effective distribution underlying subsequent observations.

    We may generalize this setting through the following assumption.

    \begin{assumption}\label{ass:single}
    The opponent's play is $\bigl\{\bq^{\S,\V}\bigr\}$--restricted.
    \end{assumption}

    Under this assumption, the desirable efficiency bound is $\Leb\big(C_{1-a^{\S,\V}}\big)$,
    the length of the interval associated with $a^{\S,\V} = \quant\bigl(\alpha,\bq^{\S,\V}\bigr)$,
    the largest element $a \in \cA$ such that $C_{1-a}$ satisfies the $1-\alpha$ validity condition on average
    under $\bq^{\S,\V}$.
    The counterpart to (and extension of) Corollary~\ref{cor:main}
    under this assumption is the following.

    \begin{corollary}[single distributional shift]\label{cor:lshift}
    Con\-si\-der $\alpha \in (0,1)$.
    The BO-ACI strategy of Algorithm~\ref{algo:blackwell}, when run with an auxiliary calibrated forecaster~$\cF$
    and tuned with $\alpha$,
    is a general strategy at miscoverage level $\alpha$ (see Definition~\ref{def:general}) and it also
    satisfies, against a $\bigl\{\bq^{\S,\V}\bigr\}$--restricted opponent, and thus, in single distributional shift setting
    described above, the following efficiency guarantee: if $m^\star$ is continuous at $\bq^{\S,\V}$,
    \begin{equation}
		\lim_{T \to \infty} \frac{1}{T}\sum_{t=1}^T \Leb\big(C_{1-a_t} \big) = \Leb\big(C_{1-a^{\S,\V}}\big) \as \label{eq:corollarysingle}
	\end{equation}
	\end{corollary}

    The first part of Corollary~\ref{cor:lshift} (that BO-ACI is a general strategy) was already stated
    right after Definition~\ref{def:general}.
    The second part is a mere application of Corollary~\ref{cor:single}, thanks to the explicit
    continuity assumption (which was unnecessary in Corollary~\ref{cor:main}).

\subsubsection{Almost-Exchangeable Setting}
\label{sec:almost-exch-1}

Our motivating example, taking place in the stochastic setting
described at the beginning of Section~\ref{sec:CSA}
corresponds to the case where the distributions $\S_t$ may differ slightly from~$\S$
and change over time, but with controlled variations, in some sense to be made precise.
To describe it with the required detail, we first establish a useful connection in terms of distances between mixed actions in $\Delta(\cB)$ and distributions of scores.

    \begin{lemma}\label{lm:distances}
    In the stochastic setting
described at the beginning of Section~\ref{sec:CSA},
        for any $t \geq 1$, we have that
        \[
        \Bigl\lVert \bq_{\mathrm{unif}} - \bq^{\S,\S_t}\Bigr\rVert_1 \leq 2 \, \dtv(\S,\S_t) \as
        \]
    \end{lemma}

    \begin{proof}
        The test score $s_t$ is an independent draw from $\S_t$ conditionally on the training, calibration, and past test observations.
        Because $\S$ is diffuse, Lemma~\ref{lm:factbt} is applicable and ensures that for all $b \in \cB$,
        \[
        q^{\S,\S_t}_b = \P_{s_t \sim \S_t}(b_t=b) = \P_{s_t \sim \S_t}\Bigl(s_{\bigl((T_{\calib}+1)(1-b)\bigr)} < s_t \leq s_{\bigl((T_{\calib}+1)(1-b)+1\bigr)}\Bigr) \as
        \]
        The same argument applied to $s \sim \S$ implies that
        \[
            \Bigl\lVert \bq^{\S,\S} - \bq^{\S,\S_t}\Bigr\rVert_1 = \sum_{i = 0}^{T_{\calib}} \biggl| \P_{s \sim \S}\Bigl(s \in \bigl(s_{(i)}, s_{(i+1)}\bigr]\Bigr) - \P_{s_t \sim \S_t}\Bigl(s_t \in \bigl(s_{(i)}, s_{(i+1)}\bigr]\Bigr) \biggr| \as
        \]
        Therefore, the definition of the total variation distance, together with classic properties of histograms applied to the (almost-surely distinct) order statistics $s_{(i)}$, gives that
        \[
            \Bigl\lVert \bq^{\S,\S} - \bq^{\S,\S_t}\Bigr\rVert_1 \leq 2 \, \dtv(\S,\S_t) \as
        \]
        The proof is concluded by recalling that $\bq^{\S,\S} = \bq_{\mathrm{unif}}$.
    \end{proof}

    We can now define the $\epsilon$--exchangeable setting (or, by abuse, the almost-exchangeable setting)
    in terms of a convergence condition on the total variation distance
    between the distribution $\S$ of the calibration scores and the distributions $(\S_t)_{t \geq 1}$ of the test scores.
    This setting notably encompasses the case of independent but non-identically distributed test scores, provided that the total variation distances $\dtv(\S,\S_t)$ are small (inferior to some threshold $\epsilon/2$) except for a sublinear number of rounds $t$.

    \begin{example} \label{ex:eps-exch}
    The $\epsilon$--exchangeable setting, where $\epsilon \in [0,2]$, corresponds to the case where
    \begin{equation} \label{eq:def-eps-exch-setting}
     \lim_{T \to \infty} \frac{1}{T}\sum_{t=1}^T \indic_{\bigl\{\dtv(\S,\S_t) > \epsilon / 2\bigr\}} = 0 \as
    \end{equation}
    In this case, the opponent's play is $Q_\epsilon$--restricted, where $Q_\epsilon$ is the
    $\ell_1$--ball in $\Delta(\cB)$ centered at $\bq_{\mathrm{unif}}$
    and with radius $\epsilon$, namely,
    \begin{equation} \label{eq:defQeps-inex}
        Q_\epsilon = \Bigl\{ \bq \in \Delta(\cB) : \bigl\lVert\bq - \bq_{\mathrm{unif}} \bigr\rVert_1 \leq \epsilon \Bigr\}\,.
    \end{equation}
    \end{example}

    \begin{proof}
    Indeed, by definition of $Q_\epsilon$, we have that
    \begin{align*}
        \frac{1}{T} \sum_{t=1}^T d(\bq_t,Q_\epsilon) &= \frac{1}{T} \sum_{t=1}^T \indic_{\bigl\{\lVert\bq_t - \bq_{\mathrm{unif}} \rVert_1 > \epsilon \bigr\}}\bigl(\lVert\bq_t - \bq_{\mathrm{unif}} \rVert_1 - \epsilon \bigr) \\
         &\leq \frac{2}{T} \sum_{t=1}^T \indic_{\bigl\{\dtv(\S,\S_t) > \epsilon / 2\bigr\}}  \as,
    \end{align*}
    where the inequality follows from the fact that $\ell_1$--distances between distributions
    are smaller than~2, together with Lemma~\ref{lm:distances} (to bound the indicator function by another indicator function).
    The desired convergence to~0 follows from the assumption~\eqref{eq:def-eps-exch-setting}.
    \end{proof}

    Based on the previous example,
    we will more generally consider the following assumption in the rest of this section.

    \begin{assumption} \label{ass:Qeps}
    The opponent's play is $Q_\epsilon$--restricted for some $\epsilon \in [0,2]$, where
    $Q_\epsilon$ is defined in~\eqref{eq:defQeps-inex}.
    \end{assumption}

We state next the associated efficiency guarantee, in Corollary~\ref{cor:mainepsilon}.
Its proof is also located in Appendix~\ref{sec:proof-cor:main}.

\begin{restatable}[$\epsilon$--exchangeable setting]{corollarry}{cormainepsilon}\label{cor:mainepsilon}
    Con\-si\-der ${\alpha \in \bigl[1/(T_{\calib}+1),\,1\bigr)}$.
    The BO-ACI strategy of Algorithm~\ref{algo:blackwell}, when run with an auxiliary calibrated forecaster~$\cF$
    and tuned with $\alpha$,
    is a general strategy at miscoverage level $\alpha$ (see Definition~\ref{def:general}) and it also
    satisfies, against a $Q_{\epsilon}$--restricted opponent,
    where $\epsilon \in \bigl[0,\,1/(T_{\calib}+1)\bigr)$,
    and thus, in the $\epsilon$--exchangeable setting
    of Example~\ref{ex:eps-exch}, the following efficiency guarantee:
    \[
    \limsup_{T \to \infty} \frac{1}{T}\sum_{t=1}^T \Leb\big(C_{1-a_t} \big) \leq \Leb\,\bigl(C_{1-\alpha+1/(T_{\calib}+1)}\bigr) \as
    \]
\end{restatable}

\paragraph{Comparison to the exchangeable case and discussion.}
A comparison of Corollaries~\ref{cor:main} and~\ref{cor:mainepsilon}
reveals that the mild price to pay for not facing fully
exchangeable scores is a replacement of the ideal length
$\Leb\,\bigl(C_{1-\alpha}\bigr)$ by
the smallest possible strictly larger value
$\Leb\,\bigl(C_{1-\alpha+1/(T_{\calib}+1)}\bigr)$,
providing that the departure from exchangeability is small enough.

Note that Corollary~\ref{cor:mainepsilon}
also encompasses the result of Remark~\ref{rk:informalexchangeable}, i.e., $\alpha \in \cA$ in the exchangeable setting.
Indeed, the opponent's play is in particular $Q_{\epsilon}$--restricted for any $\epsilon > 0$ in the exchangeable setting.

Also, the specific case of $\alpha < 1/(T_{\calib}+1)$ is not covered by Corollary~\ref{cor:mainepsilon}, but is actually pathological and coincides with the result stated in Corollary~\ref{cor:main} for $\alpha = 0$.

\subsubsection{Almost-exchangeable Setting, Continued: with Explicit Dependencies}

    In this section, we show that similar results as in the previous motivating
    example of Section~\ref{sec:almost-exch-1}
    can be obtained under a typical dependence structure,
    which notably encompasses some cases of time series.
    In particular, we show that Assumption~\ref{ass:Qeps} is also satisfied.

    We consider again the stochastic setting described at the beginning of Section~\ref{sec:CSA}
    and recall that the distributions $\S_t$ may depend on previous observations.
    Before formalizing this final, more complex, motivating example,
    we need to establish a similar relationship as in Lemma~\ref{lm:distances} but with respect to a distance between distributions that
    is more adapted to the dependent setting, namely, the Kolmogorov distance $\dK$.
    The latter is defined as the supremum norm between cumulative distribution functions.

    \begin{lemma}\label{lm:distancesbis}
    In the stochastic setting described at the beginning of Section~\ref{sec:CSA},
    for any $t \geq 1$, we have that
        \[
        \Bigl\lVert \bq_{\mathrm{unif}} - \bq^{\S,\S_t}\Bigr\rVert_1 \leq 2 (T_{\calib}+1) \, \dK(\S,\S_t) \as,
        \]
        where $\dK$ denotes the Kolmogorov distance between distributions.
    \end{lemma}

    \begin{proof}
        Recall that, under the same assumptions, we showed in the proof of Lemma~\ref{lm:distances} that
        \[
            \Bigl\lVert \bq_{\mathrm{unif}} - \bq^{\S,\S_t}\Bigr\rVert_1 = \sum_{i = 0}^{T_{\calib}} \biggl| \P_{s \sim \S}\Bigl(s \in \bigl(s_{(i)}, s_{(i+1)}\bigr]\Bigr) - \P_{s_t \sim \S_t}\Bigl(s_t \in \bigl(s_{(i)}, s_{(i+1)}\bigr]\Bigr) \biggr| \as
        \]
        Therefore, several triangular inequalities, the fact that the supports of $\S$ and $\S_t$ are both included in $[0,L] = \smash{\bigl[s_{(0)}, s_{(T_{\calib}+1)}\bigr]}$,
        and the definition of the Kolmogorov distance as the supremum norm between the cumulative distribution functions, entail that
        \[
            \Bigl\lVert \bq_{\mathrm{unif}} - \bq^{\S,\S_t}\Bigr\rVert_1 \leq 2 \sum_{i = 0}^{T_{\calib}} \Bigl| \P_{s \sim \S}\bigl(s \leq s_{(i)}\bigr) - \P_{s_t \sim \S_t}\bigl(s_t \leq s_{(i)}\bigr) \Bigr| \leq 2 (T_{\calib}+1) \, \dK(\S,\S_t) \as,
        \]
        which is the claimed bound.
    \end{proof}

    By the exact same arguments as for the $\epsilon$--exchangeable setting of Section~\ref{sec:almost-exch-1},
    we get that the opponent's play is $Q_{\epsilon}$--restricted provided that the following condition holds:
    \begin{equation} \label{eq:def-eps-exch-dep-setting}
        \lim_{T \to \infty} \frac{1}{T}\sum_{t=1}^T \indic_{\bigl\{\dK(\S,\S_t) > \epsilon / (2T_{\calib}+2)\bigr\}} = 0 \as
    \end{equation}

    We now exhibit a natural dependence structure under which such a control on the Kolmogorov distances arises.
    Following \citet{zaffran2022adaptive}, we consider the case of test scores being absolute values of a clipped auto-regressive process of order~1.
    This modeling is classic in time series forecasting, where forecast errors $r_t = y_t - \widehat{\mu}(x_t)$, also known as residuals,
    are often temporally correlated. Such a dependence may stem either from non-stationarity in the data-generating process or from model misspecification,
    when the predictor $\widehat{\mu}$ fails to capture all relevant dynamics.
    Although not formally equivalent, it is natural to assume that the calibration scores, which are i.i.d.\ according to $\S$,
    are the absolute value of underlying calibration residuals that are i.i.d.\ with distribution $\Xi$, supported on $[-L,L]$.
    In line with this perspective, we model the series of test residuals as an auto-regressive process (e.g., of order~1) with
    i.i.d.\ innovations drawn from $\Xi$.

    \begin{example}\label{ex:dependent}
    Let $(\xi_t)_{t \geq 1}$ be an i.i.d.\ innovation process with distribution $\Xi$ (the one of the calibration residuals), and
    let $\phi > 0$ be a (typically small) auto-regressive parameter.
    We consider the setting where the test scores are absolute values of the clipped $\operatorname{AR}(1)$ process $(r_t)_{t \geq 1}$
    initialized at $r_1 = \xi_1$ and recursively defined by
        \[
        \forall t \geq 2\,, \qquad r_t = \max\bigl\{ \min \{\phi \, r_{t-1} + \xi_t , \, L \}, \, -L \bigr\}\,,
        \qquad \mbox{i.e.,} \qquad s_t = |r_t|\,.
        \]
        We denote by $\S_t$ the distribution of $s_t$ conditional to past observations
        (training data and previous residual $r_{t-1}$).
        Note that the support of $\S_t$  is included in $[0,L]$.

        If we further assume that the density of $\S$ is upper bounded by some constant $\lVert f_{\S} \rVert_{\infty}$, then
        we get the following bound on the Kolmogorov distances:
        \[
        \forall t \geq 1\,, \qquad \dK(\S,\S_t) \leq L \phi \, \lVert f_{\S} \rVert_{\infty} \as
        \]
    \end{example}

    \begin{proof}
    First, the construction of the scores, a triangular inequality, and the fact that the support of $\S_{t-1}$ is included in $[0,L]$ entail that
    \[
    s_t \leq | \phi \, r_{t-1} + \xi_t| \leq |\xi_t| + L \phi  \as
    \]
    Second, either $| \phi \, r_{t-1} + \xi_t| \leq L $ or $| \phi \, r_{t-1} + \xi_t| > L$; by a reverse triangular inequality in the first case,
    and the fact that $s_t = L$ while $|\xi_t| \in [0,L]$ in the second case, we get that
    \[
    |\xi_t| - L \phi \leq s_t \as
    \]
    Therefore, we have that $|\xi_t| - L \phi \leq s_t \leq |\xi_t| + L \phi$ almost surely;
    hence, since $|\xi_t| \sim \S$, we obtain
        \[
          \forall u \in [0,L]\,, \qquad \P_{s \sim \S}(s \leq u - L \phi)  \leq \P_{s_t \sim \S_t}(s_t \leq u) \leq \P_{s \sim \S}(s \leq u + L \phi) \as
        \]
        Based on the above inequalities, we can now control the Kolmogorov distance between $\S$ and $\S_t$ by substituting the cumulative distribution
        function of $\S_t$ by the cumulative distribution function of $\S$ evaluated at two close points, namely,
        \begin{align*}
            \dK(\S, \S_t) &\eqdef \sup_{u \in [0,L]} \max\Bigl\{\P_{s_t \sim \S_t}(s_t \leq u) - \P_{s \sim \S}(s \leq u), \, \P_{s \sim \S}(s \leq u) -\P_{s_t \sim \S_t}(s_t \leq u) \Bigr\} \\
            &\leq \sup_{u \in [0,L]} \max\Bigl\{\P_{s \sim \S}\bigl(s \in (u,u + L \phi] \bigr), \,  \P_{s \sim \S}\bigl(s \in (u - L \phi , u ]\bigr) \Bigr\} \as
        \end{align*}
        Finally, the assumption on the density of $\S$, together with the above inequality, entails that
        \[
            \dK(\S, \S_t) \leq L \phi \, \lVert f_{\S} \rVert_{\infty} \as \vspace{-.95cm}
        \]

    \end{proof}

        In the dependent setting depicted in Example~\ref{ex:dependent}, we showed that the Kolmogorov distances are controlled in the sense of~\eqref{eq:def-eps-exch-dep-setting},
        with $\epsilon = L\phi \, \lVert f_{\S} \rVert_{\infty}$. This quantity is small under weak dependence (small $\phi$) and for a non-concentrated distribution $\S$ (small $\lVert f_{\S} \rVert_{\infty}$).
        Consequently, the opponent's play is $Q_{\epsilon}$--restricted by Lemma~\ref{lm:distancesbis}, so that Assumption~\ref{ass:Qeps}
        and Corollary~\ref{cor:mainepsilon} are applicable.

\subsection{On Efficiency Objectives in the Adversarial Setting}

This section is devoted to the adversarial setting of Definition~\ref{def:advsetting}
and Example~\ref{ex:adv}, which corresponds to $Q = \Delta(\cB)$,
i.e., no restriction on the opponent's play.

\paragraph{Point of this section.}
We set $p^\star$ and $m^\star$
in a simple way in~\eqref{eq:responseandgoalfunctions-m}
and~\eqref{eq:response} and picked, in particular,
a non-randomized response function $p^\star$ outputting Dirac masses.

These definitions are such that there exist distributions $\bq \in \Delta(\cB)$
with $\quant(\alpha,\bq) = 0$ and thus $m^\star(\bq) = (0,\,2L)^\top$.

There also exist other distributions $\bq \in \Delta(\cB)$
with $\quant(\alpha,\bq) = 1/(T_{\calib}+1)$ and $m^\star(\bq) = \bigl( \alpha,\,2 s_{(T_{\calib})} \bigr)^\top$.
The latter quantity could be close to $(\alpha,2L)^\top$
so that, with the simple definitions considered,
Theorem~\ref{th:main} does not seem to guarantee anything more
specific than a convergence of $\overline{m}_T$
to $[0,\alpha] \times [0,2L]$ in the adversarial setting.

Yet, we have to acknowledge that convergence
to $[0,\alpha] \times \bigl[0,(1-\alpha)2L\bigr]$ may be achieved,
with a randomized $p^\star$, and that this convergence is the
best achievable one. Both results are formally stated below
in Lemma~\ref{lm:adversarial}.
    This lemma highlights a fundamental limitation of the fully adversarial setting: while validity can always be satisfied,
    efficiency cannot simultaneously be controlled against any opponent in a non-trivial way, i.e.,
    by a bound strictly smaller than $(1-\alpha)2L$.

However, for the sake of a simpler exposition, we rather considered in Section~\ref{sec:def-cC-alpha-Q}
the target set $\C_\alpha = [0, \alpha] \times [0, 2L]$,
i.e., we neglected the $1-\alpha$ term in the efficiency component. The set $\C_\alpha$ is also approachable
by Lemma~\ref{lm:adversarial} (and with the same simple strategy).
Remark~\ref{rk:pstar-rando} hints at the complications that would arise from considering the smaller
target set $[0, \alpha] \times \bigl[0, (1-\alpha)2L\bigr]$.
We therefore preferred to obtain slightly suboptimal efficiency results in the
adversarial case rather than picking a more complex $p^\star$ than in~\eqref{eq:response},
at the cost of additional technicalities in the analysis.

    \begin{lemma}\label{lm:adversarial}
        The target set
        $[0,\alpha] \times \bigl\{ (1-\alpha)2L \bigr\}$ and thus the target set
        $[0, \alpha] \times \bigl[0, (1-\alpha)2L\bigr]$ are approachable by the learner;
        put differently, there exists a (simple) strategy of the learner such that,
        the following inequalities hold against any opponent, almost surely:
        \[
        \limsup_{T \to \infty} \frac{1}{T}\sum_{t=1}^T \indic_{b_t \leq a_t} \leq \alpha
        \qquad \mbox{and} \qquad
	    \limsup_{T \to \infty} \frac{1}{T}\sum_{t=1}^T \Leb\big(C_{1-a_t} \big) \leq (1-\alpha)2L \,.
        \]
        Conversely, for all $\epsilon >0$, the target set $[0, \alpha] \times \bigl[0, (1-\alpha)2L - \epsilon\bigr]$ is not approachable by the learner;
        put differently, the inequalities above cannot hold almost surely against any opponent
        when the second condition is replaced by $\leq (1-\alpha)2L - \epsilon$.
	\end{lemma}

    Lemma~\ref{lm:adversarial} basically indicates that in a fully adversarial setting,
    only trivial results may be achieved when it comes to efficiency.
    This is why we restricted the behavior of the opponent, in a way already considered
    by the literature of online learning, that
    encompasses both exchangeable and adversarial settings as extreme cases; namely,
    through $Q$--restrictions, see Section~\ref{sec:sev-degr-stoch}. \medskip

    \begin{proof}
        For the first statement of the lemma, consider the following strategy\footnote{The associated conformal sets correspond to the typical example of a poor but valid sequence of prediction sets that is used in the literature of conformal inference to emphasize that validity alone is not a sufficiently interesting objective. } for the learner:
        at each round $t \geq 1$, pick $a_t = 0$ with probability $1-\alpha$ and $a_t = 1$ with probability $\alpha$, i.e., play $\bp_t = (1-\alpha)\bdelta_{0} + \alpha\bdelta_1$.
        By the boundedness assumption on the scores, $y_t \in C_{1}(x_t)$ with probability~$1$. Also, by design, the sets $C_{0}(x_t)$ and $C_{1}(x_t)$
        have respective lengths~$0$ and~$2L$.
        Thus, $m(\bp_t, b_t) \in [0,\alpha] \times \bigl\{ (1-\alpha)2L \bigr\}$ for all round $t \geq 1$.
        By martingale convergence (e.g., the Hoeffding-Azuma inequality together with the Borel-Cantelli lemma),
        \[
        \limsup_{T \to \infty} \Biggl| \frac{1}{T}\sum_{t=1}^T m(a_t, b_t) - \underbrace{\frac{1}{T}\sum_{t=1}^T m(\bp_t, b_t)}_{\in [0,\alpha] \times \{ (1-\alpha)2L \}} \Biggr| = 0 \as
        \]
        which concludes the proof of the first statement.

        For the second statement of the lemma, we fix $\epsilon > 0$ and show that Blackwell's dual condition is not satisfied for the target set
        $\C^{(\epsilon)} = [0, \alpha] \times \bigl[0, (1-\alpha)2L - \epsilon\bigr]$, which yields the desired lack of
        approachability, by Theorem~\ref{th:blackwell} of Appendix~\ref{sec:Blk}.
        We consider $\bq = \bdelta_{1/(T_{\calib}+1)}$ and proceed by contradiction:
        we assume that there exists $\bp = (p_a)_{a\in\cA} \in \Delta(\cA)$ such that $m(\bp,\bq) \in \C^{(\epsilon)}$.
        Since $s_{(T_{\calib})} < L$ by the choice of $L$ (see Remark~\ref{rk:L}),
        we can consider $(x,y)$ such that $s_{(T_{\calib})} < | y - \widehat{\mu}(x) | < L$,
        e.g., by taking any $x \in \X$ and letting $y = \widehat{\mu}(x) + \bigl(s_{(T_{\calib})}+L\bigr)/2$.
        This choice of $(x,y)$ indeed corresponds to $\bq = \bdelta_{1/(T_{\calib}+1)}$ and is such that the only conformal set that contains $y$ is $C_1(x)$.
        Therefore $m(\bp,\bq) \in \C^{(\epsilon)}$ implies in particular that the prediction set $C_1(x)$ is selected with probability at least $1-\alpha$, i.e., $p_0 \geq 1-\alpha$.
        In that case, the second component of $m(\bp,\bq)$ is at least $(1-\alpha)\Leb\,(C_1) = (1-\alpha)2L$, thus contradicting the fact that $m(\bp,\bq) \in \C^{(\epsilon)}$.
    \end{proof}

    \begin{remark}
    \label{rk:pstar-rando}
    With $[0, \alpha] \times \bigl[0, (1-\alpha)2L\bigr]$ as a target set,
    a suitable response function should satisfy, for all $\bq \in \Delta(\cB)$,
    \[
    p_{\mbox{\tiny \rm rand}}^{\star}(\bq) \in \argmin \! \left\{ \sum_{a \in \cA} p_a \Leb\,\bigl(C_{1-a}\bigr)
    \ \ \mbox{over} \ \ \bp \in \Delta(\cA)
    \ \ \mbox{subject to} \ \ \sum_{a \in \cA}\sum_{b \in \cB} p_a \, q_{b} \, \indic_{b \leq \alpha} \leq \alpha \right\} .
    \]
    Such a response function $p_{\mbox{\tiny \rm rand}}^{\star}$ would lead to randomized actions,
    as opposed to the Dirac-mass-valued function $p^\star$ considered
    with the target set $\C = [0, \alpha] \times [0, 2L]$.
    Our analysis could be adapted to handle this slightly more complex case
    (the original version of \citealp{bernstein2014opportunistic} does so)
    but this would not, in general,
    bring any substantial improvement in the efficiency guarantees achieved,
    except perhaps in the adversarial setting, while leading to less elementary proofs
    in Appendix~\ref{sec:proofs-main}.
    \end{remark}

\section{BO-ACI Strategy Based on $\epsilon$--Calibrated Forecasters}
\label{sec:boaci-epscalibr}

In this section, we extend the theoretical guarantees of the BO-ACI strategy of Algorithm~\ref{algo:blackwell} to the case where the auxiliary forecaster is only $\epsilon$--calibrated as per the following definitions,
starting with the notion of disjoint $\epsilon$--covering.

\begin{definition}[disjoint $\epsilon$--covering]\label{def:disjointcovering}
Given $\epsilon>0$, a finite partition $\cP = \bigl\{P_k : k \in [N]\bigr\}$ of $\Delta(\cB)$ into measurable subsets is a disjoint $\epsilon$--covering if, for every $k \in [N]$, there exists $\bc_k \in P_k$ such that
\[
P_k \subseteq B(\bc_k,\epsilon) \eqdef \bigl\{ \bq \in \Delta(\cB) : \lVert \bc_k - \bq \rVert_1 \leq \epsilon\bigr\} \,.
\]
We refer to $\bigl\{\bc_k : k \in [N] \bigr\}$ as centers associated with the disjoint $\epsilon$--covering $\cP$.
\end{definition}

\begin{definition}[$\epsilon$--calibrated forecaster]\label{def:epsiloncalibrated}
Given $\epsilon>0$, and a disjoint $\epsilon$--covering $\cP$ with associated centers $\bigl\{\bc_k : k \in [N] \bigr\}$, a forecaster is said to be $\epsilon$--calibrated if, for every strategy of the opponent, its forecasts $(\bz_t)_{t \geq 1}$ take values in the set of centers,
and if, denoting for each $t\ge1$ by $k_t\in[N]$ the index such that $\bz_t = \bc_{k_t}$, the following holds:

\[
\limsup_{T \to \infty} \sum_{k=1}^{N} \left\lVert \frac{1}{T} \sum_{t=1}^{T} \indic_{k_t = k} \, \bigl(\bc_k - \bdelta_{b_t}\bigr) \right\rVert \leq \epsilon \as \]
\end{definition}

For technical reasons and with no loss of generality, we assume that the underlying
disjoint $\epsilon$--covering $\cP = \bigl\{P_k : k \in [N]\bigr\}$ of $\Delta(\cB)$
used by the auxiliary $\epsilon$--calibrated forecaster $\cF$
is compatible with $p^\star$, or equivalently, with $\quant(\alpha,\,\cdot\,)$, in the following sense:
\begin{equation}
\label{eq:def-compat}
\forall k \in [N], \ \ \exists a \in \cA \quad \mbox{s.t.} \quad
P_k \subseteq R_a \eqdef \bigl\{\bq \in \Delta(\cB) : \quant(\alpha,\bq) = a \bigr\}\,,
\end{equation}
i.e., $p^\star$ and $\quant(\alpha,\,\cdot\,)$ take constant values on each $P_k$.

This compatibility indeed comes with no loss of generality as we may transform any
disjoint $\epsilon$--covering $\cP' = \bigl\{P'_k : k \in [N']\bigr\}$
into
\[
\cP = \bigl\{P_k' \cap R_a : k \in [N'] \ \mbox{and} \ a \in \cA \ \mbox{s.t.} \ P_k' \cap R_a \ne \emptyset \bigr\}
\]
to make it compatible with $p^\star$ and $\quant(\alpha,\,\cdot\,)$. We critically use here that $\cA$ is finite.

We can now state the adaptation of Theorem~\ref{th:main} to an $\epsilon$--calibrated forecaster with respect to a disjoint $\epsilon$--covering
compatible with $p^\star$.
We strengthen the assumption of $Q$--restriction of the opponent
used in Definition~\ref{def:opp-app} and Theorem~\ref{th:main}:
we do not only assume that $d(\bq_t,Q)$ vanishes on Ces{\'a}ro average
but instead that $d(\bq_t,Q) = 0$, i.e., $\bq_t \in \cl(Q)$, for all $t \geq 1$.

\begin{restatable}[Opportunistic $\epsilon$--approachability for adaptive conformal inference]{theorem}{thmainepsilon}\label{th:mainepsilon}
    Fix $\epsilon \in (0,1)$.
    The BO-ACI strategy of Algorithm~\ref{algo:blackwell}, when tuned with $\alpha \in (0,1)$
    and run with an auxiliary $\epsilon$--calibrated forecaster $\cF$ relying on
    a disjoint $\epsilon$--covering compatible with $p^\star$ in the sense of~\eqref{eq:def-compat}, is an $\epsilon$--opportunistically approaching
    strategy in the following sense: for any opponent, for any convex subset $Q \subseteq \Delta(\cB)$,
    \begin{align*}
    & \mbox{for all } t \geq 1\,, \quad d(\bq_t,Q) = 0 \as \\
    \mbox{entails} \qquad\quad & \limsup_{T \to \infty} d\Bigl(\overline{m}_T,\C_{\alpha,Q}^{\sqrt{\epsilon}}\Bigr) \leq (2L+1) \bigl(\epsilon + 2\sqrt{\epsilon} \bigr) \as,
    \end{align*}
    where $\smash{\displaystyle{\C_{\alpha,Q}^{\sqrt{\epsilon}} = \cl\biggl(\conv\Bigl(\bigl\{m^{\star}(\bq) : d(\bq,Q) \leq \sqrt{\epsilon} \bigr\}\Bigr)\biggr)}}$.
\end{restatable}

\section{Discussion and Future Work}

In this work, we formulate adaptive
conformal inference [ACI] as a repeated
two-player game with finite action sets
and vector-valued payoffs.
This opens a toolbox containing
Blackwell approachability
and its opportunistic extension, which we leverage
to construct a strategy that
guarantees validity in all cases at the desired level
and adapts the efficiency
guarantees to the
underlying, unknown, degree of stochasticity of the
opponent.
This ``best of many worlds'' property
provides a unified perspective on
exchangeable, adversarial,
and intermediate settings for adaptive conformal inference.

\paragraph{Future work and research perspectives.}
A first natural extension of this work includes
establishing finite-time convergence rates
for the validity and efficiency guarantees.
The validity criterion, when considered in isolation,
can be satisfied at a rate of $1/T$,
as proved directly by \citet{gibbs2021adaptive}
and as can be obtained also by the characterization
of slow versus fast rates for approachability by
\citet{mannor2013approachability}.
What would be the (optimal) rates
when simultaneously considering both criteria, validity and efficiency?
The analysis of the BO-ACI strategy relies, on the one hand, on the use of a calibrated auxiliary
forecaster: calibration may be achieved at a rate of $T^{-1/(1+|\cB|)}$.
However, \citet{marinov2026efficient} recently showed that
this calibration subroutine could be bypassed, which improves
the associated rates to $T^{-1/3}$ or $T^{-1/4}$,
depending on the (lack of) computational efficiency of the
approach.
The second important element is the convergence
stated in Definition~\ref{def:Q-restricted} of $Q$--restricted
opponents: it should come with a specific rate.
Both rates may then be combined to obtain
rates on the $m$--approachability of the sets $\cC_{\alpha,Q}$,
but it is highly unclear whether one would obtain
the best achievable rates doing so, in our specific problem
of validity and efficiency in ACI.

More generally,
the results presented in this article should be put
in perspective with the indicated
computationally efficient
variant of opportunistic approachability
developed by \citet{marinov2026efficient}\footnote{This work
was led in parallel to ours; the first
version of the present submission
was posted in October 2025 (see \url{https://arxiv.org/abs/2510.15824})
before any version of \citet{marinov2026efficient},
only publicly posted in February 2026.},
which builds on recent developments
in polynomial-time $\epsilon$--calibration
by \citet{peng2025high} and \citet{fishelson2026high}.

Third and finally, it would be interesting to extend
our framework to
handle sequential updates of the training and calibration sets;
this could be performed by considering time-dependent payoff
function~$m_t$.


\acks{This work was conducted as part of the Inria--EDF challenge.}

\clearpage
\appendix

\section*{Outline of the Appendices}

The appendices provide first the proofs of all results
stated in the main body:
\begin{itemize}
\item Appendix~\ref{sec:proofs-main} provides elementary proofs
of the general convergence result of Section~\ref{sec:main-results-OA},
of its instantiation to specific scenarios performed in
Section~\ref{ex:Q-restr}, and of the extension to $\epsilon$--calibration
proposed in Section~\ref{sec:boaci-epscalibr}.
\end{itemize}

They then provide reminders on (and occasional extensions of)
approachability theory and opportunistic approachability,
together with the rederivations of some facts.
These sections may of course be skipped by a reader familiar
with approachability theory.
More specifically,
\begin{itemize}
\item Appendix~\ref{app:reminder} offers a general exposition of and reminder on
approachability theory and its extension by \citet{bernstein2014opportunistic}
consisting of opportunistic approachability, with mere statements of
results and proofs omitted;
\item Appendix~\ref{app:Qprimal}
goes over Remark~\ref{rk:Q-prim-strat}
and shows that when the set $Q$ of restricted play is known to the learner,
opportunistic approachability is not needed and that
the classic primal approachability strategy can be adapted to guarantee convergence to $\C_Q$ instead of $\C$,
with a focus on the exchangeable case $Q = \{ \bq_{\unif} \}$;
\item Appendix~\ref{app:go-over-Bernstein} revisits the theory of
opportunistic approachability established by \citet{bernstein2014opportunistic}
for so-called D-sets: it performs two technical twists
(one technical correction, consisting of taking closures,
while showing the necessity thereof; and one relaxation,
consisting of not requiring piecewise continuity of
the response function) and provides a simplified
and less technical analysis than in the original reference
    (while keeping its spirit; simplifications take place at the margins).
\end{itemize}

\section{Proofs of the Main Results, Located in Sections~\ref{sec:main-results-OA}--\ref{ex:Q-restr}--\ref{sec:boaci-epscalibr}}
\label{sec:proofs-main}

This appendix proves all the results (lemmas, theorems, corollaries)
located in Sections~\ref{sec:main-results-OA}--\ref{ex:Q-restr}--\ref{sec:boaci-epscalibr},
and does so through elementary and self-contained
proofs.

\subsection{Proof of Lemma~\ref{lm:restricted-calibr} of Section~\ref{sec:calibr}}
\label{sec:proof-calibr-Q-restr}

The following lemma encompasses and generalizes Lemma~\ref{lm:restricted-calibr} of Section~\ref{sec:calibr}: it indeed covers
both the calibrated case of Lemma~\ref{lm:restricted-calibr}
and an $\epsilon$--calibrated result that will be useful when proving
Theorem~\ref{th:mainepsilon} of Section~\ref{sec:boaci-epscalibr}.

\begin{lemma}[Calibration under $Q$--restricted opponent]
\label{lm:restricted}
Consider any convex subset $Q \subseteq \Delta(\cB)$, and any $Q$--restricted opponent. If a forecaster is calibrated, respectively, $\epsilon$--calibrated, the forecasts $(\bz_t)_{t \geq 1}$ output satisfy, almost surely:
    \[
        \lim_{T \to \infty} \frac{1}{T}\sum_{t=1}^{T} d(\bz_t,Q) = 0\,, \qquad \mbox{respectively,} \qquad \limsup_{T \to \infty} \frac{1}{T}\sum_{t=1}^{T} d(\bz_t,Q) \leq \epsilon \,.
    \]
\end{lemma}

Lemma~\ref{lm:restricted} was stated and proved, in the calibrated case, by
\citet[Lemma 6]{bernstein2014opportunistic}.
We provide below an alternative, elementary proof of this result,
also covering the case of $\epsilon$--calibration. \medskip

\begin{proof}
    Fix a convex subset $Q \subseteq \Delta(\cB)$ and a sequence of mixed actions $(\bq_t)_{t \geq 1}$ output by a $Q$--restricted opponent.
    We decompose the proof in three steps: we show first a general result; second, we instantiate it for an $\epsilon$--calibrated forecaster; third, we use it repeatedly to get the result for a calibrated forecaster. \bigskip

    \emph{Step 1: a general result.} We fix $\epsilon >0$. We show that, for any finite set $\bigl\{\bu_k \in \Delta(\cB) : k \in [N]\bigr\}$, and for all sequences of indices $(k_t)_{t \geq 1}$ picked by the forecaster,
\begin{equation}
\begin{split}
    \limsup_{T \to \infty} \sum_{k=1}^{N} \left\lVert \frac{1}{T} \sum_{t=1}^{T} \indic_{k_t = k} \, \bigl(\bu_k - \bdelta_{b_t}\bigr) \right\rVert & \leq \epsilon \as \\ \mbox{entails} \hspace{5cm} \limsup_{T \to \infty} \frac{1}{T}\sum_{t=1}^{T} d(\bu_{k_t},Q) & \leq \epsilon \as \label{eq:generalcalibration}
\end{split}
\end{equation}
By martingale convergence (e.g., repeated applications of the Hoeffding-Azuma inequality together with the Borel-Cantelli lemma), the left--hand side of \eqref{eq:generalcalibration} is equivalent to
\begin{align}
    \limsup_{T \to \infty} \sum_{k=1}^{N} \left\lVert \frac{1}{T} \sum_{t=1}^{T} \indic_{k_t = k} \, (\bu_k - \bq_t) \right\rVert \leq \epsilon \as \label{eq:calibrationmixed}
\end{align}
We now introduce the empirical averages of the opponent's mixed actions over the rounds when $k$ is picked:
\begin{align*}
        \overline{\bq}_k(T) =
        \begin{cases}
        \displaystyle\frac{1}{N_k(T)}\sum_{t=1}^{T} \indic_{k_t = k} \, \bq_t \qquad &\mbox{if } N_k(T) \eqdef \displaystyle\sum_{t=1}^{T} \indic_{k_t = k} \geq 1 \,,\\
        \bb{0} &\mbox{if } N_k(T) = 0\,.
        \end{cases}
\end{align*}
By grouping according to the values of $k_t$ for the first equality, by a triangular inequality for the inequality,
and by definition of the $\overline{\bq}_k(T)$ for the second equality,
the target quantity in~\eqref{eq:generalcalibration} can be rewritten as
\begin{align*}
    \frac{1}{T}\sum_{t=1}^{T} d(\bu_{k_t},Q) &= \sum_{k=1}^{N} \frac{N_k(T)}{T} \, d(\bu_k,Q) \\
    & \leq \sum_{k=1}^{N} \frac{N_k(T)}{T} \, \bigl\lVert \bu_k - \overline{\bq}_k(T) \bigr\rVert + \sum_{k=1}^{N} \frac{N_k(T)}{T} \, d\bigl(\overline{\bq}_k(T),Q\bigr) \\
    &= \underbrace{\sum_{k=1}^{N} \left\lVert \frac{1}{T}\sum_{t=1}^T \indic_{k_t = k} \, (\bu_{k} - \bq_t) \right\rVert }_{\limsup ... \leq \epsilon \, \scriptsize\mbox{ by \eqref{eq:calibrationmixed}}} + \underbrace{\sum_{k=1}^{N} \frac{N_k(T)}{T} \, d\bigl(\overline{\bq}_k(T),Q\bigr)}_{\longrightarrow \,0 \,\scriptsize\mbox{ by $Q$--restriction}} \,.
\end{align*}
We substituted the limit behavior~\eqref{eq:calibrationmixed} for the first sum, while the second sum can be seen to vanish,
as follows. Indeed,
by convexity of $Q$, the function $d(\,\cdot\,,Q)$ is also convex (\citealp[Example 2.25]{rockafellar1998variational}), and thus,
for $k$ such that $N_k(T) \geq 1$,
\begin{align*}
    d\bigl(\overline{\bq}_k(T),Q \bigr) = d\Biggl(\frac{1}{N_k(T)}\sum_{t=1}^{T}\indic_{k_t=k} \, \bq_t \, , Q \Biggr) \leq \frac{1}{N_k(T)} \sum_{t=1}^{T}\indic_{k_t=k} \, d\bigl(\bq_t , Q \bigr) \,.
\end{align*}
Therefore,
\[
    \sum_{k=1}^{N} \frac{N_k(T)}{T} \, d\big(\overline{\bq}_k(T),Q\big) \leq \frac{1}{T}\sum_{t=1}^{T} d(\bq_t, Q) \,,
\]
where the right-hand side vanishes by the assumption of $Q$--restricted opponent. \bigskip

\emph{Step 2.} The stated claim for an $\epsilon$--calibrated forecaster follows immediately by applying~\eqref{eq:generalcalibration} to the set of centers from Definition~\ref{def:epsiloncalibrated} (i.e., $\bu_k = \bc_k$). \bigskip

\emph{Step 3.} The application to a calibrated forecaster is a slightly more complex; we actually show that its forecasts satisfy:
\begin{align}
\label{eq:step3-Qrestr}
   \forall \, \epsilon >0\,, \qquad \limsup_{T \to \infty} \frac{1}{T}\sum_{t=1}^{T} d(\bz_t, Q) \leq 2 \epsilon\,.
\end{align}
To that end, fix $\epsilon >0$. We consider a disjoint $\epsilon$--covering $\cP_{\epsilon} =\bigl\{P_k : k \in [N_\epsilon] \bigr\}$ as per Definition~\ref{def:disjointcovering}, with associated centers $\bigl\{ \bc_k : k \in [N_\epsilon]\bigr\}$. We can now define the mapping $\bz \mapsto \kappa_{\epsilon}(\bz) \in [N_\epsilon]$ given by $\bz \in P_{\kappa_{\epsilon}(\bz)}$. (Note that the mapping is well defined because $\cP_{\epsilon}$ is a partition.)
By construction of $\kappa_{\epsilon}$ and the definition of $\cP_{\epsilon}$,
\begin{align}
    \forall \, t\geq 1 \,, \qquad \bigl\{\kappa_{\epsilon}(\bz_t) = k \bigr\} = \bigl\{\bz_t \in P_k \bigr\}  \subseteq \bigl\{ \lVert \bz_t -\bc_k \rVert_1 \leq \epsilon\bigr\} \,. \label{eq:coveringproperties}
\end{align}
The calibration property of Definition~\ref{def:calibrated}, used with $G = P_k$ for a given $k\in[N_\epsilon]$, ensures that
\[
    \lim_{T \to \infty} \left\lVert \frac{1}{T} \sum_{t=1}^{T} \indic_{\kappa_{\epsilon}(\bz_t) = k} \, \bigl(\bz_t - \bdelta_{b_t}\bigr) \right\rVert =
    \lim_{T \to \infty} \left\lVert \frac{1}{T} \sum_{t=1}^{T} \indic_{\bz_t \in P_k} \, \bigl(\bz_t - \bdelta_{b_t}\bigr) \right\rVert
    = 0 \as
\]
Summing the above limit over the finitely many $k\in[N_\epsilon]$, we get
\[
\lim_{T \to \infty} \,\, \sum_{k = 1}^{N_\epsilon} \left\lVert \frac{1}{T} \sum_{t=1}^{T} \indic_{\kappa_{\epsilon}(\bz_t) = k} \, \bigl(\bz_t - \bdelta_{b_t}\bigr) \right\rVert = 0 \as
\]
Repeated applications of the triangular inequality based on~\eqref{eq:coveringproperties},
and the fact that the $\ell_1$--norm is larger than the Euclidean norm, then ensure that
\[
    \limsup_{T \to \infty} \sum_{k = 1}^{N_\epsilon} \left\lVert \frac{1}{T} \sum_{t=1}^{T} \indic_{\kappa_{\epsilon}(\bz_t) = k} \, \bigl(\bc_k - \bdelta_{b_t}\bigr) \right\rVert \leq \epsilon \as
\]
The limit above corresponds to the premise of the implication~\eqref{eq:generalcalibration}; the latter therefore yields that
\[
\frac{1}{T}\sum_{t=1}^T d(\bc_{\kappa_{\epsilon}(\bz_t)}, Q) \leq \epsilon \as
\]
Now, repeated applications of the triangular inequality, again based on~\eqref{eq:coveringproperties},
guarantee that
\[
\frac{1}{T}\sum_{t=1}^T d(\bz_t, Q) \leq \frac{1}{T}\sum_{t=1}^T d(\bc_{\kappa_{\epsilon}(\bz_t)}, Q) + \epsilon \,,
\]
so that property~\eqref{eq:step3-Qrestr} follows by collecting all elements together.
\end{proof}

\subsection{Proof of Theorem~\ref{th:main} of Section~\ref{sec:strat-BOACI}}
\label{sec:proof-th-main}

\thmain*

This theorem could be obtained as an instance of the general theory
of opportunistic approachability by
\citet{bernstein2014opportunistic}, which we recall in
Appendix~\ref{app:reminder}. However, when focusing on the specific
payoff function $m$, response function $p^\star$ (which is piecewise constant),
and goal function $m^\star$, we may bypass many technicalities required
to build a general theory. We provide below a concise, carefully
simplified and revisited, self-contained, proof. \medskip

\begin{proof}
    Fix a convex subset $Q \subseteq \Delta(\cB)$ and a sequence of mixed actions $(\bq_t)_{t \geq 1}$ output by a $Q$--restricted opponent; we show that $d\bigl(\overline{m}_T,\C_{\alpha,Q}\bigr) \to 0$ a.s.

    To do so, we consider the following upper bound, based on a triangular inequality:
    \[
        d\bigl(\overline{m}_T,\C_{\alpha,Q}\bigr) \leq \left\lVert \overline{m}_T-\frac{1}{T}\sum_{t=1}^T m(a_t,\bz_t)\right\rVert + d\!\left(\frac{1}{T}\sum_{t=1}^T m(a_t,\bz_t), \C_{\alpha,Q}\right) \,.
    \]
    Each of the two steps of the proof consists of proving that each of the two summands vanishes almost surely in the limit. \bigskip

    \emph{Step 1.}
    First, we recall that the BO-ACI strategy of Algorithm~\ref{algo:blackwell}
    only picks pure actions $a_t \in \cA$, given by $a_t = \quant(\alpha,\bz_t)$.
    Hence, we can define a partition $(R_a)_{a \in \cA}$ of $\Delta(\cB)$ based on the inverse
    images of $\quant(\alpha,\,\cdot\,)$, namely, $R_a = \bigl\{\bq \in \Delta(\cB) : \quant(\alpha,\bq) = a \bigr\}$.
    The action $a_t$ picked is such that $\bz_t \in R_{a_t}$, and we may group
    rounds according to the $(R_a)_{a \in \cA}$:
    \begin{align*}
	\frac{1}{T}\sum_{t=1}^T m(a_t,\bz_t)  =
    \sum_{a \in \cA } \frac{1}{T}\sum_{t=1}^T \indic_{\bz_t \in R_a} \, m(a,\bz_t)
    &= \sum_{a \in \cA} M_a \left(\frac{1}{T}\sum_{t=1}^T \indic_{\bz_t \in R_a} \, \bz_t\right) \\
    \mbox{and} \hspace{6cm} \overline{m}_T  = \frac{1}{T}\sum_{t=1}^T m(a_t,b_t)
    &= \sum_{a \in \cA} M_a \left(\frac{1}{T}\sum_{t=1}^T \indic_{\bz_t \in R_a} \, \bdelta_{b_t} \right),
    \end{align*}
    where the existence of a matrix $M_a$ such that $M_a\,\bz = m(a,\bz)$ for all $\bz \in \Delta(\cB)$ follows from the linear extension of $m$.
    Second, the calibration property of Definition~\ref{def:calibrated}, used with $G = R_a$,
    exactly guarantees that for each $a \in \cA$,
    \[
    \lim_{T \to \infty} \left\lVert \frac{1}{T}\sum_{t=1}^T \indic_{\bz_t \in R_a} \, \bigl(\bz_t-\bdelta_{b_t}\bigr) \right\rVert = 0 \as
    \]
    Collecting all elements together, resorting to a triangular inequality,
    and using the upper bound $\lVert m(a,b) \rVert \leq 2L+1$ for all $(a,b) \in \cA \times \cB$,
    we therefore have
    \[
    \left\lVert \overline{m}_T-\frac{1}{T}\sum_{t=1}^T m(a_t,\bz_t)\right\rVert
    \leq \sum_{a \in \cA} (2L+1) \,\, \left\lVert \frac{1}{T}\sum_{t=1}^T \indic_{\bz_t \in R_a} \, \bigl(\bz_t-\bdelta_{b_t}\bigr) \right\rVert \longrightarrow 0 \as,
    \]
    as announced. \bigskip

    \emph{Step 2.} We use the definition of $a_t$ as $p^\star(\bz_t) = \bdelta_{a_t}$ for the rewriting $m(a_t,\bz_t) = m^{\star}(\bz_t)$; we show that
    \[
    d\!\left(\frac{1}{T}\sum_{t=1}^T m(a_t,\bz_t), \C_{\alpha,Q}\right) = d\!\left(\frac{1}{T}\sum_{t=1}^T m^{\star}(\bz_t), \C_{\alpha,Q}\right) \longrightarrow 0 \as
    \]
    \noindent Fix $\eta >0$ for now and introduce the sets of rounds where the forecasts $\bz_t$ belong to the $\eta$--neighborhood of $Q$:
    \[
    \forall T \geq 1, \qquad \V_T(\eta) = \bigl\{t \in [T] : d(\bz_t,Q) \leq \eta\bigr\}\,.
    \]
    Their cardinalities $\bigl| \V_T(\eta) \bigr|$ satisfy $T - \bigl| \V_T(\eta) \bigr| = o(T)$; indeed,
    by Markov's inequality, and then by the $Q$--restriction property of the calibrated forecaster established in Lemma~\ref{lm:restricted},
    \begin{align}
        \frac{T-\bigl| \V_T(\eta) \bigr|}{T} = \frac{1}{T}\sum_{t=1}^{T} \indic_{d(\bz_t,Q) > \eta} \leq \frac{1}{\eta} \,\left(\frac{1}{T}\sum_{t=1}^{T} d(\bz_t,Q) \right) \xrightarrow[T \to \infty]{} 0 \as \label{eq:markov}
    \end{align}
    A triangular inequality entails that
    \begin{multline*}
    d\!\left(\frac{1}{T}\sum_{t=1}^T m^{\star}(\bz_t), \C_{\alpha,Q}\right) \\
    \leq \underbrace{\Biggl\lVert \frac{1}{T}\sum_{t=1}^{T}m^{\star}(\bz_t)-\frac{1}{\bigl| \V_T(\eta) \bigr|}\sum_{t \in \V_T(\eta)\!\!\!\!}m^{\star}(\bz_t) \Biggr\rVert}_{\leq
    2(2L+1)(T - |\V_T(\eta)|)/T}
    \,\,+\,\, d\Biggl(\frac{1}{\bigl| \V_T(\eta) \bigr|}\sum_{t \in \V_T(\eta)\!\!\!\!} m^{\star}(\bz_t) , \,\, \C_{\alpha,Q}\Biggr) \,,
    \end{multline*}
    where the upper bound indicated for the first term is obtained by the $(2L+1)$--boundedness of $m$ thus of $m^\star$ (see Step~1)
    and the triangle inequality
    \begin{align*}
    \lefteqn{\Biggl\lVert \frac{1}{T}\sum_{t=1}^{T} m^{\star}(\bz_t)-\frac{1}{\bigl| \V_T(\eta) \bigr|}\sum_{t \in \V_T(\eta)\!\!\!\!}m^{\star}(\bz_t) \Biggr\rVert} \\
    & \leq \frac{1}{T} \sum_{\,t \notin \V_T(\eta)\!}\bigl\lVert m^{\star}(\bz_t) \bigr\rVert + \Biggl(\frac{1}{\bigl| \V_T(\eta) \bigr|}-\frac{1}{T}\Biggr)\sum_{t \in \V_T(\eta)\!\!\!\!} \bigl\lVert m^{\star}(\bz_t) \bigr\rVert \\
    & \leq (2L+1) \left( \frac{T - \bigl| \V_T(\eta) \bigr|}{T} + \left( \frac{1}{\bigl| \V_T(\eta) \bigr|}-\frac{1}{T} \right) \bigl| \V_T(\eta) \bigr| \right)
    = 2 \, (2L+1) \frac{T - \bigl| \V_T(\eta) \bigr|}{T}\,.
    \end{align*}
    We also note that the set $\C_{\alpha,Q}$ is convex as the limit of convex sets.
    Thus, the function $d(\,\cdot\,,\C_{\alpha,Q})$ is also convex (\citealp[Example 2.25]{rockafellar1998variational}), so that,
    for all $T \geq 1$,
    \[
    d\Biggl(\frac{1}{\bigl| \V_T(\eta) \bigr|}\sum_{t \in \V_T(\eta)\!\!\!\!} m^{\star}(\bz_t) , \,\, \C_{\alpha,Q}\Biggr)
    \leq M_\eta \eqdef \sup \Bigl\{ d\bigl(m^\star(\bz), \, \C_{\alpha,Q}\bigr) : \bz \in \Delta(\cB) \ \ \mbox{s.t.} \ \ d(\bz,Q) \leq \eta \Bigr\}\,.
    \]
    Collecting all elements, we proved so far that for all $\eta > 0$,
    \begin{equation}
    \label{eq:ccl-pr-thmain-step2}
    \limsup_{T \to \infty} \ d\!\left(\frac{1}{T}\sum_{t=1}^T m^{\star}(\bz_t), \C_{\alpha,Q}\right) \leq M_\eta \as\,,
    \end{equation}
    and it now suffices to show that $M_\eta \to 0$ as $\eta \downarrow 0$.
    The proof for the aforementioned convergence relies on topological arguments,
    detailed in Appendix~\ref{app:geometric} below, which concludes the proof.
    \end{proof}

    \subsection{An Auxiliary Topological Result} \label{app:geometric}

    For each $\eta > 0$, we introduce
    \[
    \C_{\alpha,Q}^\eta = \cl\biggl(\conv\Bigl(\bigl\{m^{\star}(\bq) : d(\bq,Q) \leq \eta\bigr\}\Bigr)\biggr)
    \]
    and note that, by definitions and nestedness of the sets $\C_{\alpha,Q}^\eta$,
    \[
    M_\eta \leq \sup_{\bc \in \C_{\alpha,Q}^\eta} d(\bc,\C_{\alpha,Q}) \qquad \mbox{and} \qquad
    \C_{\alpha,Q} = \bigcap_{n \geq 1} \C_{\alpha,Q}^{1/n}\,.
    \]
    In addition, the sets $\C_{\alpha,Q}^{1/n}$ are bounded (since $m^\star$ is bounded) and closed thus compact.
    The proof of $M_{1/n} \to 0$, which corresponds to $M_\eta \to 0$ by nestedness, is therefore provided by
    the lemma below.

    \begin{lemma}[{see, e.g., \citealp[page~30]{henrot2005variation}}]\label{lm:Hausdorff}
    A nested sequence of non-empty compact subsets $\bigl(K_n\bigr)_{n \geq 1}$ of $\R^d$
    converges to their intersection $\smash{\displaystyle{K = \bigcap_{n \geq 1} K_n}}$ in the Hausdorff distance:
    \[
    d_H(K_n,K) = \max\!\left\{ \sup_{\bk \in K_n} d(\bk,K), \,\, \sup_{\bk \in K} d(\bk,K_n) \right\} \underset{n \to \infty}{\longrightarrow} 0 \,.
    \]
    \end{lemma}

    \begin{proof}
    As $K \subseteq K_n$ for all $n \geq 1$, we have $d(\bk, K_n) = 0$ for all $\bk \in K$, thus $\smash{\displaystyle{\sup_{\bk \in K} d_H(\bk,K_n)} = 0}$;
    we are anyway interested in the other supremum.

    By continuity of the point-to-set distance function $d(\,\cdot\,, K)$ and by compactness of the non-empty sets $K_n$, there exists a sequence of points $(\bk_n)_{n \geq 1}$ such that $\bk_n \in K_n$ and $\sup_{\bk \in K_n} d(\bk, K) = d(\bk_n, K)$ for all $n \geq 1$. By the Bolzano-Weierstrass theorem, we can extract a subsequence $\bk_{\varphi(n)}$ converging to some $\bk_\infty \in \R^d$, where $\varphi$ is a non-decreasing application
    from positive integers to positive integers.
    The sequence of sets is nested, therefore, first, $\bigl(d(\bk_n, K)\bigr)_{n \geq 1}$ is non-increasing thus converging,
    and second, that limit can only be the one of $d(\bk_{\varphi(n)}, K)$, which, by continuity, is $d(\bk_\infty,K)$.
    Now, again by nestedness, for each $n \geq 1$, the points $\bk_{\varphi(N)}$, with $N \geq n$, belong to $K_{\varphi(n)}$,
    which is closed; therefore, their limit $\bk_\infty$
    also belongs to $K_{\varphi(n)}$. This being true for all $n \geq 1$, and the sets being nested, we have that $\bk_\infty \in K$.
    Collecting all elements, we thus proved
    \[
    d(\bk_n, K) \underset{n \to \infty}{\longrightarrow} d(\bk_\infty,K) = 0\,,
    \]
    which concludes the proof.
    \end{proof}

    \subsection{Proof of Corollaries~\ref{cor:main} and~\ref{cor:mainepsilon} of Sections~\ref{sec:exch-main} and~\ref{sec:CSA}}
    \label{sec:proof-cor:main}

    \cormain*

    \begin{proof}
    We establish below the continuity of $m^{\star}$ at $\bq_{\unif}$ and also provide its value
    at that point.
    Both facts, together with Corollary~\ref{cor:single}, ensure the convergence
    \[
    \lim_{T \to \infty} \overline{m}_T = m^{\star}(\bq_{\unif}) = \begin{pmatrix}
    r(\alpha) \\
    \Leb\,\bigl(C_{1-\alpha}\bigr)
    \end{pmatrix} \,.
    \]

    For the continuity,
    recall that the response function is defined as $p^\star : \bq \longmapsto \bdelta_{\quant(\alpha, \bq)}$, where the function $\quant$ is defined as
    \[
        \quant(\alpha,\,\cdot\,) : \bq \in \Delta(\cB) \longmapsto \max \biggl\{a \in \cA : \sum_{b \in \cB} q_b \, \indic_{b \leq a}\leq \alpha\biggr\} \,.
    \]
    Since $\cB = \cA \setminus \{0\}$,
    \begin{equation}
    \label{eq:sumralpha}
    \forall a \in \cA, \qquad \sum_{b \in \cB} \frac{1}{|\cB|} \, \indic_{b \leq a} = a\,, \qquad \mbox{so that} \qquad
    \quant(\alpha, \bq_{\unif}) = r(\alpha)\,.
    \end{equation}
    Take $\epsilon = \min\bigl\{\alpha - r(\alpha), \, r(\alpha) - \alpha + 1/|\cB| \bigr\}$: since $\alpha \notin \cA$, we have $\epsilon > 0$.
    We now show that for all $\bq \in \Delta(\cB)$ with $\lVert \bq - \bq_{\unif} \rVert_1 < \epsilon$, we have $\quant(\alpha, \bq) = r(\alpha)
    =\quant(\alpha, \bq_{\unif}) $, which will ensure the continuity of $p^{\star}$, thus of $m^{\star}$ given the definition of the latter.
    For all $\bq$ such that $\lVert \bq - \bq_{\unif} \rVert_1 < \epsilon$, by a first triangular inequality,
    \[
    \sum_{b \in \cB} q_b \, \indic_{b \leq r(\alpha)} < \frac{1}{|\cB|} \sum_{b \in \cB} \indic_{b \leq r(\alpha)} + \epsilon = r(\alpha) + \epsilon \leq \alpha \,.
    \]
    Similarly, we have by another triangular inequality that
    \[
    \sum_{b \in \cB} q_b \, \indic_{b \leq r(\alpha) + 1/|\cB|} > \frac{1}{|\cB|}  \sum_{b \in \cB} \indic_{b \leq r(\alpha) + 1/|\cB|} - \epsilon = r(\alpha) + 1/|\cB| - \epsilon \geq \alpha \,,
    \]
    so that $\quant(\alpha, \bq) = r(\alpha)$, as announced.

    The value of $\bigl\{m^{\star}(\bq_{\unif}) \bigr\}$ follows: its first component equals $r(\alpha)$ given
    the left statement of~\eqref{eq:sumralpha}, while its second component equals $\Leb\,\bigl(C_{1-r(\alpha)}\bigr)
    = \Leb\,\bigl(C_{1-\alpha}\bigr)$.
    \end{proof}

    \cormainepsilon*

    \begin{proof}
        We recall that $|\cB| = T_{\calib} + 1$ and will rather write $|\cB|$ in the proof.
        Fix $\alpha,\,\epsilon$ as in the statement, i.e., $\alpha \in \bigl[1/|\cB|,\,1\bigr)$ and $\epsilon \in \bigl[0,\,1/|\cB|\bigr)$.
        We apply Theorem~\ref{th:main} with $Q_{\epsilon} = \bigl\{\bq \in \Delta(\cB) : \lVert \bq - \bq_{\unif} \rVert_1 \leq \epsilon \bigr\}$ and get the convergence
        \[
        \overline{m}_T \longrightarrow \C_{Q_{\epsilon}} \as,
        \]
        where, by Lemma~\ref{lm:targetgeneral}, we have $\displaystyle{\C_{Q_{\epsilon}} \subseteq [0, \alpha] \times \biggl[ 0, \,\, \lim_{\vphantom{d(\bq,Q_{\epsilon})}\eta \downarrow 0} \max_{\bq : d(\bq, Q_{\epsilon}) \leq \eta} \Leb\big(C_{1-\quant(\alpha, \bq)}\big) \biggr]}$.

        \smallskip

        We now show that the limit involved in the definition of the above set is smaller than $\Leb\,\bigl(C_{1-\alpha+1/|\cB|}\bigr)$, which will conclude the proof.
        By a triangular inequality together with a standard $\ell_1$--Euclidean norm
        inequality, we get that, for all $\eta >0$,
        \[
            \Bigl\{ \bq \in \Delta(\cB) : d\bigl(\bq, Q_{\epsilon}\bigr) \leq \eta \Bigr\} \subseteq \Bigl\{\bq \in \Delta(\cB) : \bigl\lVert \bq - \bq_{\unif} \bigr\rVert_1 \leq \epsilon + \sqrt{|\cB|}\,\eta \Bigr\} \,.
        \]
        Recall that $\epsilon < 1/|\cB|$; therefore $\epsilon + \sqrt{|\cB|}\,\eta < 1/|\cB|$ for $\eta >0$ small enough, and thus
        \[
        \lim_{\vphantom{d(\bq, Q_{\epsilon})}\eta \downarrow 0} \max_{\bq : d(\bq, Q_{\epsilon}) < \eta} \Leb\,\bigl(C_{1-\quant(\alpha, \bq)}\bigr) \leq \max_{\bq : \lVert \bq - \bq_{\unif} \rVert_1 < 1/|\cB|} \Leb\,\bigl(C_{1-\quant(\alpha, \bq)}\bigr) \,.
        \]
        We fix $\bq$ such that $\lVert \bq - \bq_{\unif} \rVert_1 < 1/|\cB|$ and resort to an argument similar as the final calculations of the proof of Corollary~\ref{cor:main}:
        \[
        \sum_{b \in \cB} q_b \, \indic_{b \leq r(\alpha) - 1/|\cB|} < \frac{1}{|\cB|} +
        \sum_{b \in \cB} \frac{1}{|\cB|} \, \indic_{b \leq r(\alpha) - 1/|\cB|} = \frac{1}{|\cB|} + r(\alpha) - \frac{1}{|\cB|} \leq \alpha \,.
        \]
        Therefore, $\quant(\alpha, \bq) \geq r(\alpha) - 1/|\cB| \geq 0$, so that
        \[
            \Leb\,\bigl(C_{1-\quant(\alpha, \bq)}\bigr) \leq \Leb\,\bigl(C_{1-r(\alpha)+1/|\cB|}\bigr) =  \Leb\,\bigl(C_{1-\alpha+1/|\cB|}\bigr) \,.
        \]
        The property above holds for all $\bq$ such that $\lVert \bq - \bq_{\unif} \rVert_1 < 1/|\cB|$, which concludes the proof.

    \end{proof}

    \subsection{Proof of Theorem~\ref{th:mainepsilon} of Section~\ref{sec:boaci-epscalibr}}
    \label{sec:eps-cal}

The proof is obtained by adapting the proof provided for Theorem~\ref{th:main} in Appendix~\ref{sec:proof-th-main},
and is thus ultimately also an adaptation of the proof techniques by
\citet{bernstein2014opportunistic}.

\thmainepsilon*

\begin{proof}
The proof follows the same structure as the one of Theorem~\ref{th:main},
whose notation we continue to use (see Appendix~\ref{sec:proof-th-main}).
In particular, we start with the triangular inequality
\begin{equation}
\label{eq:decomp-epsilon-proofmaineps}
    d\bigl(\overline{m}_T,\C_{\alpha,Q}^{\sqrt{\epsilon}}\bigr) \leq
    \underbrace{\left\lVert \overline{m}_T-\frac{1}{T}\sum_{t=1}^T m(a_t,\bz_t)\right\rVert}_{\limsup ... \ \leq (2L+1)\epsilon} +
    \underbrace{d\!\left(\frac{1}{T}\sum_{t=1}^T m(a_t,\bz_t), \C_{\alpha,Q}^{\sqrt{\epsilon}}\right)}_{\limsup ... \ \leq 2(2L+1)\sqrt{\epsilon}} \,;
\end{equation}
then, each step term is almost surely bounded by the limit quantities indicated above. \medskip

\emph{Step 1.}
The following two rewritings only depended on the BO-ACI strategy of Algorithm~\ref{algo:blackwell},
not on the auxiliary forecasting strategy, and they thus still hold:
\[
\overline{m}_T = \sum_{a \in \cA} m\!\!\left(a, \,\, \frac{1}{T}\sum_{t=1}^T \indic_{\bz_t \in R_a} \, \bdelta_{b_t} \right)
\quad \mbox{and} \quad
\frac{1}{T}\sum_{t=1}^T m(a_t,\bz_t) = \sum_{a \in \cA} m\!\!\left(a, \,\, \frac{1}{T}\sum_{t=1}^T \indic_{\bz_t \in R_a} \, \bz_t\right),
\]
where $\Arrowvert m(a,b) \Arrowvert \leq 2L+1$ for all $(a,b)$, so that
\[
\left\lVert \overline{m}_T-\frac{1}{T}\sum_{t=1}^T m(a_t,\bz_t)\right\rVert
\leq (2L+1) \sum_{a \in \cA} \left\lVert \frac{1}{T}\sum_{t=1}^T \indic_{\bz_t \in R_a} \, \bigl(\bz_t-\bdelta_{b_t}\bigr) \right\rVert .
\]
We now use the properties of the auxiliary forecasting strategy:
the compatibility property~\eqref{eq:def-compat} of the disjoint $\epsilon$--covering $\cP = \bigl\{P_k : k \in [N]\bigr\}$
implies that
\[
\forall a \in \cA, \qquad R_a = \bigcup_{k \,:\, P_k \subseteq R_a}  P_k \,,
\]
so that, by triangular inequalities for the first inequality,
by Definition~\ref{def:epsiloncalibrated} of $\epsilon$--calibration
for the equality and the final inequality,
    \begin{align*}
    \limsup_{T \to \infty} \sum_{a \in \cA} \left\lVert \frac{1}{T}\sum_{t=1}^T \indic_{\bz_t \in R_a} \, \bigl(\bz_t-\bdelta_{b_t}\bigr) \right\rVert
    & \leq \limsup_{T \to \infty} \smash{\sum_{k=1}^N \overbrace{\sum_{a \in \cA} \indic_{P_k \subseteq R_a}}^{=1}} \left\lVert \frac{1}{T}\sum_{t=1}^T \indic_{\bz_t \in P_k} \, \bigl(\bz_t-\bdelta_{b_t}\bigr) \right\rVert \\
    & = \limsup_{T \to \infty} \sum_{k=1}^N \left\lVert \frac{1}{T}\sum_{t=1}^T \indic_{k_t = k} \, \bigl(\bc_k-\bdelta_{b_t}\bigr) \right\rVert
    \leq \epsilon \as
    \end{align*}
Collecting all elements, we proved the first limsup in~\eqref{eq:decomp-epsilon-proofmaineps}.

\emph{Step 2.}
For the second terms in~\eqref{eq:decomp-epsilon-proofmaineps},
we first note that by definition of action $a_t$ played
and by $\epsilon$--calibration,
\[
\forall t \geq 1, \qquad m(a_t,\bz_t) = m^{\star}(\bz_t) = m(a_t,\bc_{k_t})\,.
\]
We introduce the set of indices for which $\bc_k$ belongs to the $\sqrt{\epsilon}$--neighborhood of $Q$,
\[
\K_\epsilon = \bigl\{k \in [N] : d(\bc_k, Q) \leq \sqrt{\epsilon}\bigr\}\,,
\]
and show later that
\begin{equation}
\label{eq:frequency}
\limsup_{T \to \infty} \frac{T - N_{\K_\epsilon}(T)}{T} \leq \sqrt{\epsilon} \as,
\qquad \mbox{where} \qquad
N_{\K_\epsilon}(T) = \sum_{k \in \K_\epsilon} N_k(T)\,.
\end{equation}
Given that $\epsilon \in (0,1)$ by assumption, we have
$\limsup N_{\K_\epsilon}(T) \geq 1$,
and for $T$ such that $N_{\K_\epsilon}(T) \geq 1$,
\begin{multline*}
d\!\left(\frac{1}{T}\sum_{t=1}^T m(a_t,\bz_t), \C_{\alpha,Q}^{\sqrt{\epsilon}}\right)
= d\!\left(\frac{1}{T}\sum_{k \in [N]} N_k(T)\, m^\star(\bc_k), \, \C_{\alpha,Q}^{\sqrt{\epsilon}}\right) \\
\leq \underbrace{\left\lVert \frac{1}{T} \!\!\!\sum_{k \in [N]} \!\! N_k(T) m^\star(\bc_k) \!- \!\frac{1}{N_{\K_\epsilon}\!(T)} \! \sum_{k \in \K_\epsilon} \!\! N_k(T)  m^{\star}(\bc_k) \right\rVert}_{\limsup \, \ldots \ \leq 2(2L+1)\sqrt{\epsilon}}
\!+ \underbrace{d\!\left(\!\frac{1}{N_{\K_\epsilon}\!(T)} \! \sum_{k \in \K_\epsilon} \!\! N_k(T) \, m^{\star}(\bc_k), \, \C_{\alpha,Q}^{\sqrt{\epsilon}}\!\right)}_{= \,0 \, \scriptsize\mbox{by definition of $\C_{\alpha,Q}^{\sqrt{\epsilon}}$ and $\K_\epsilon$}}\!.
\end{multline*}
The first term above is bounded as in the proof of Theorem~\ref{th:main}, with
$N_{\K_\epsilon}(T)$ playing the role of $\bigl| \V_T(\eta) \bigr|$,
by using the limit~\eqref{eq:frequency}, which we prove next.

Introduce (as in the proof of Lemma~\ref{lm:restricted}) the notation,
for each $k \in [N]$,
    \[
    \overline{\bq}_k(T) = \begin{cases}
    \displaystyle\frac{1}{N_k(T)}\sum_{t=1}^{T} \indic_{k_t=k} \, \bq_t \,, \qquad &
    \mbox{if } N_k(T) \eqdef \displaystyle{\sum_{t=1}^{T} \indic_{k_t=k} \geq 1} \,, \\
    \bb{0} \,, \qquad &\mbox{otherwise} \,.
    \end{cases}
    \]
    For all $k,T$ such that $N_k(T) \geq 1$, we have $d(\overline{\bq}_k(T), Q) = 0$ almost surely.
    Indeed, the set $Q$, thus the function $d(\,\cdot\,,Q)$, are convex, and by assumption $d(\bq_\tau,Q) = 0$ almost surely for all $\tau$.
    Hence, by a triangular inequality, for all $k,T$ such that $N_k(T) \geq 1$,
    \[
    \sqrt{\epsilon} \, \indic_{d(\bc_k, Q) > \sqrt{\epsilon}}
    \leq d(\bc_k, Q) \leq \lVert\bc_k - \overline{\bq}_k(T)\rVert + d\bigl(\overline{\bq}_k(T), Q\bigr)
    = \lVert\bc_k - \overline{\bq}_k(T)\rVert \as
    \]
    We actually get for all pairs $k,T$, even when $N_k(T) = 0$, that
    \[
    \sqrt{\epsilon} \, N_k(T) \, \indic_{d(\bc_k, Q) > \sqrt{\epsilon}}
    \leq N_k(T) \, \lVert\bc_k - \overline{\bq}_k(T)\rVert \as
    \]
    Therefore,
    \begin{align*}
    \sqrt{\epsilon} \,\, \frac{T - N_{\K_\epsilon}(T)}{T} & =
    \sqrt{\epsilon} \sum_{k=1}^{N} \frac{N_k(T)}{T} \indic_{d(\bc_k, Q) > \sqrt{\epsilon}} \\
    & \leq \sum_{k=1}^{N} \frac{N_k(T)}{T} \, \lVert\bc_k - \overline{\bq}_k(T)\rVert
    = \sum_{k=1}^{N} \left\lVert \frac{1}{T} \sum_{t=1}^{T} \indic_{k_t = k} \, \bigl(\bc_k - \bq_t \bigr) \right\rVert \,.
    \end{align*}
    Now, by the $\epsilon$--calibration property of Definition~\ref{def:epsiloncalibrated},
    \[
    \limsup_{T \to \infty} \sum_{k=1}^{N} \left\lVert \frac{1}{T} \sum_{t=1}^{T} \indic_{k_t = k} \, \bigl(\bc_k - \bdelta_{b_t} \bigr) \right\rVert \leq \epsilon \as,
    \]
    which by martingale convergence (e.g., repeated applications of the Hoeffding-Azuma inequality together with the Borel-Cantelli lemma) is equivalent to
    \[
    \limsup_{T \to \infty} \sum_{k=1}^{N} \left\lVert \frac{1}{T} \sum_{t=1}^{T} \indic_{k_t = k} \, \bigl(\bc_k - \bq_t \bigr) \right\rVert \leq \epsilon \as
    \]
    Collecting all elements, we proved the limit~\eqref{eq:frequency}, which concludes the proof.
\end{proof}

\section{Reminder of Classic Results in Online Adversarial Learning}
\label{app:reminder}

This appendix restates classic results in online adversarial learning,
i.e., in the settings where a learner faces an opponent reacting to
the actions picked.
More precisely,
\begin{itemize}
\item Appendix~\ref{sec:Blk} is devoted to Blackwell approachability
in its original form (\citealp{blackwell});
\item Appendix~\ref{sec:opp-app} presents the
extension by \citet{bernstein2014opportunistic} consisting of opportunistic approachability
and which forms the key tool used in designing and analyzing the
BO-ACI strategy of Algorithm~\ref{algo:blackwell}.
\end{itemize}
We provide here mere statements of the results
but Appendices~\ref{app:Qprimal} and
Appendices~\ref{app:go-over-Bernstein} offer
detailed proofs.
The considered vector-valued payoff function will be denoted by
$\mfm$, to distinguish it from the specific function $m$
considered in the main body and in Appendix~\ref{sec:proofs-main}.

\subsection{Blackwell Approachability in Its Original Form}
\label{sec:Blk}

We recall Blackwell's approachability theory
for general finite action sets $\cA$ and $\cB$ and for general vector-valued payoff functions $\mfm : \cA \times \cB \to \R^d$.
In the main body, this theory is applied to the specific payoff function~\eqref{eq:finite-game}
but this appendix considers a general bounded vector-valued payoff function $\mfm$.

At each round $t \geq 1$, the learner and the opponent simultaneously pick actions $a_t \in \cA$
and $b_t \in \cA$. The learner obtains the payoff $\mfm(a_t,b_t)$ and observes it;
unless explicitly stated (e.g., in the main body and in Appendix~\ref{app:go-over-Bernstein}),
the learner does not have access to the action $b_t$ of the opponent.
The learner wants the average payoff to converge to a closed convex set $\cC \subset \R^d$,
referred to as the target set.
Formally, the aim of the learner is to ensure that
\begin{equation}
\label{eq:def-appr}
\frac{1}{T} \sum_{t=1}^T \mfm(a_t,b_t) \eqdef
\overline{\mfm}_T \longrightarrow \cC \quad \mbox{a.s.} \qquad \mbox{in the sense} \qquad
\lim_{T\to\infty} d\bigl(\overline{\mfm}_T, \cC\bigr) = 0\,,
\end{equation}
where, for $u \in \R^d$, the distance of $u$ to the closed convex set $\cC$ equals $d(u,\cC) =
\displaystyle{\min_{c \in \cC} \Arrowvert u-c \Arrowvert}$.

The aim of the opponent is to prevent this convergence. Actually,
the proof of Blackwell's theorem (Theorem~\ref{th:blackwell} below) shows that
either the learner can approach $\cC$ or the opponent can approach the complement of a
neighborhood of $\cC$, i.e., make sure that
\[
\liminf_{T\to\infty} d\bigl(\overline{\mfm}_T, \cC\bigr) \geq \gamma > 0 \qquad \mbox{a.s.,}
\]
for some $\gamma > 0$, irrespective of the strategy of the learner.

Blackwell's theorem provides a necessary and sufficient condition for the approachability
of a closed convex set $\cC$ in a finite game, based on the bilinear extension of $\mfm$:
for all probability distributions $\bp = (p_a)_{a \in \cA} \in \Delta(\cA)$ and $\bq = (q_b)_{b \in \cB} \in \Delta(\cB)$,
\[
\mfm(\bp,\bq) = \sum_{a\in \cA}\sum_{b\in \cB} p_{a} \, q_{b} \, \mfm(a,b) \, .
\]

    \begin{restatable}[Approachability theorem, {\citealp{blackwell}}]{theorem}{thblackwell}
    \label{th:blackwell}
		A closed convex subset $\cC$ of~$\mathbb{R}^d$ is approachable by the learner, i.e., \eqref{eq:def-appr} may hold, if and only if
        \begin{align}
            \forall \bq \in \Delta(\cB), \ \exists \bp \in \Delta(\cA) \quad \text{s.t.} \quad \mfm(\bp, \bq) \in \cC \,. \label{eq:dual}
        \end{align}
	\end{restatable}

The condition~\eqref{eq:dual} is referred to as Blackwell's dual condition.
The proof of Theorem~\ref{th:blackwell} is constructive: the explicit algorithm introduced therein
is based on a geometric argument and is referred to as the primal Blackwell's approachability strategy.
It consists, at round~$t$, of drawing $a_t$ independently at random according to a distribution $\bp_t$ satisfying
\begin{align}
\bp_t \in \argmin_{\bp \in \Delta(\cA)} \bmax_{\bq \in \Delta(\cB)} \bigl\langle \overline{\mfm}_{t-1} - \projC(\overline{\mfm}_{t-1}), \mfm( \bb{p},\bb{q}) \bigr\rangle \, , 	\label{eq:primal}
\end{align}
where $\projC$ is the projection in Euclidean norm onto $\cC$.

\begin{remark}
Theorem~\ref{th:blackwell} mostly indicates whether a sequential-learning problem evaluated through average payoffs
is feasible or not, but in general, the so-called primal approachability strategy~\eqref{eq:primal} cannot be computed efficiently
(often due to the projection onto $\cC$) nor is the strategy achieving the optimal rates of convergence.
\end{remark}

\paragraph{Elements of proofs for Theorem~\ref{th:blackwell}.}
The sufficient part follows from arguments located in Appendix~\ref{app:Qprimal}:
the strategy~\eqref{eq:primal} is studied therein, for the case $Q = \Delta(\cB)$.

The necessity part works by contradiction: when there exists $\bq^\dash \in \Delta(\cB)$
such that $\mfm(\bp, \bq^\dash) \not\in \cC$ for all $\bp \in \Delta(\cA)$,
then, since $\cC$ is closed, there exists $\gamma > 0$ such that
\[
\bigl\{ \mfm(\bp, \bq^\dash) : \bp \in \Delta(\cA) \bigr\} \subseteq \overline{\cC^\gamma}\,,
\]
where $\cC^\gamma$ is the $\gamma$--neighborhood of $\cC$ in Euclidean norm
and $\overline{\cC^\gamma}$ is the complement of $\cC^\gamma$.
Assume that the opponent draws the actions $b_t$ i.i.d.\ according to $\bq^\dash$.
For all strategies of the learner,
by martingale convergence (e.g., repeated applications of Hoeffding's inequality together with the Borel-Cantelli lemma),
\[
\lim_{T \to \infty} \left\lVert \mfm\bigl(\overline{\bp}_T,\bq^\dash\bigr) - \frac{1}{T} \sum_{t=1} \mfm(a_t,b_t) \right\rVert = 0 \as,
\]
where $\overline{\bp}_T$ denotes the empirical distribution of the actions $a_t$ played between rounds $t=1$ and $t=T$.
Thus,
\[
\frac{1}{T} \sum_{t=1} \mfm(a_t,b_t) \longrightarrow \overline{\cC^\gamma}  \as, \qquad \mbox{as $T \to \infty$.}
\]
In particular, the convergence to $\cC$ stated in \eqref{eq:def-appr} cannot hold.

\subsection{Opportunistic Approachability by \citet{bernstein2014opportunistic}}
\label{sec:opp-app}

This section summarizes the setting and main result of \citet{bernstein2014opportunistic}.
The aim is to design a general strategy of the learner not only approaching a given approachable set $\C \subset \R^d$, but actually approaching a smaller target set $\C_Q \subseteq \C$ whenever the opponent is $Q$--restricted, with $Q \subseteq \Delta(\cB)$.
We recall that the concept of $Q$--restricted opponent's play was introduced in Definition~\ref{def:Q-restricted}. \medskip

More formally, we consider the general framework of Section~\ref{sec:Blk} for a closed convex approachable set~$\cC$.
The dual condition~\eqref{eq:dual} in Blackwell's approachability theorem (see Theorem~\ref{th:blackwell}) ensures the existence of a response function $\mfp^\star : \Delta(\cB) \to \Delta(\cA)$ and of an associated goal function $\mfm^{\star} : \Delta(\cB) \to \C$ such that
\begin{align}
    \forall \bq \in \Delta(\cB), \quad \mfm^{\star}(\bq) \eqdef \mfm\bigl(\mfp^\star(\bq),\bq\bigr) \in \C \,. \label{eq:responseandgoalfunctions}
\end{align}
The response function $\mfp^\star$ is not uniquely defined;
the learner should pick a response function~$\mfp^\star$ such that the resulting payoff $\mfm^{\star}(\bq)$ is desirable for any mixed action $\bq \in \Delta(\cB)$
of the opponent.

We exploit this choice of $\mfm^\star$ to define smaller target sets $\cC_Q$ for $Q$--restricted opponents:
namely, $\C_Q$ is based on the image of $Q$ under the goal function $\mfm^{\star}$,
with some additional twists required for some technical reasons (taking closures\footnote{\citet{bernstein2014opportunistic}
do not consider closures, which is not an issue in many cases but can be in some other pathological
cases; we go over these issues in Appendix~\ref{app:go-over-Bernstein}.}
of convex hulls and considering the non-increasing limit of
images of $\eta$--neighborhoods of~$Q$).

\begin{definition}[closed convex image under $\mfm^\star$] \label{def:image}
    The closed convex image of the convex subset $Q \subseteq \Delta(\cB)$ under the goal function $\mfm^{\star}$ is defined as
    \[
        \C_Q = \bigcap_{\eta>0}\cl\biggl(\conv\Bigl(\bigl\{\mfm^{\star}(\bq) : d(\bq,Q) \leq \eta\bigr\}\Bigr)\biggr) \,.
    \]
\end{definition}

\noindent
We now state formally the property that strategies are expected to satisfy in this framework.

\begin{definition}[opportunistic approachability]
    \label{def:opp-app}
    A strategy is said to be opportunistically approaching a closed convex set $\cC$ given
    a goal function $\mfm^{\star}$ if, for any convex subset $Q \subseteq \Delta(\cB)$,
    it approaches the closed convex image $\C_Q$ whenever the opponent is $Q$--restricted:
    \[
    \lim_{T \to \infty} \frac{1}{T}\sum_{t=1}^{T} d(\bq_t,Q) = 0 \as \qquad \mbox{entails} \qquad \lim_{T \to \infty} d\bigl(\overline{\mfm}_T,\C_Q\bigr) = 0  \as \smallskip
    \]
\end{definition}

Such strategies exist, and are based on auxiliary calibrated forecasters (see Section~\ref{sec:calibr} for a reminder on the latter concept).
\citet{bernstein2014opportunistic} provide a general construction and study of opportunistically approaching strategies
for general closed convex approachable sets $\cC$, that comes with some necessary technicalities and is therefore rather long and complex. (They actually even cover the more general case of so-called D--sets $\cC$.)

For the sake of completeness, we provide in Appendix~\ref{app:go-over-Bernstein}
simplified proofs for this general construction.

\section{On $Q$--Primal Strategies to Approach $\C_Q$ when $Q$ is Known}
\label{app:Qprimal}

This section goes over Remark~\ref{rk:Q-prim-strat}
and does so both in some generality (with a general vector payoff function $\mfm$)
and for the specific case of adaptive conformal inference (with the specific
vector payoff function $m$ introduced therein). \medskip

The notion of opportunistic approachability is designed to handle unknown convex subsets $Q \subseteq \Delta(\cB)$
and relies on a so-called dual strategy, exploiting a response function $\mfp^\star$, as introduced in
Section~\ref{sec:opp-app}.
If instead $Q$ were known to the learner, one may want to adapt the primal approachability strategy~\eqref{eq:primal} to
directly approach $\C_Q$, and not only $\C$.
In this section, we first define such a strategy,
show (in Appendix~\ref{sec:Qprim-mfm}) that it indeed approaches $\C_Q$ whenever the opponent is $Q$--restricted,
and characterize it in the context of adaptive conformal inference in the case of the exchangeable setting
(in Appendix~\ref{sec:Qprim-m}).

\paragraph{General statement.}
Under the Blackwell condition~\eqref{eq:dual},
the $Q$--primal strategy is defined as follows: at round $t \geq 1$, draw $a_t$ independently at random according to a distribution $\bp_t$ satisfying
\begin{align}
    \bp_t \in \argmin_{\bp \in \Delta(\cA)} \bmax_{\bq \in \cl(Q)} \,\, \langle \overline{\mfm}_{t-1} - \projCQ(\overline{\mfm}_{t-1}) , \mfm(\bp,\bq) \rangle\,. \label{eq:Qprimal}
\end{align}
This strategy only differs from the original primal Blackwell's strategy~\eqref{eq:primal}
in that the maximum is taken over $\cl(Q)$ instead of $\Delta(\cB)$
and in that the projection is taken onto $\C_Q$ instead of $\C$.
Note that the closure $\cl(Q)$ is used in~\eqref{eq:Qprimal},
mostly for the sake of elegance of the formula (to have a maximum instead of a mere supremum
over $Q$), without changing the strategy in any way.

We state next that the $Q$--primal strategy approaches $\C_Q$ whenever the opponent is $Q$--restricted.
This result generalizes the sufficient condition of Blackwell's approachability theorem (restated as Theorem~\ref{th:blackwell});
the latter indeed corresponds to the case $Q = \Delta(\cB)$.

\begin{restatable}{theorem}{thQprimal}\label{th:Qprimal}
     If the opponent's play is $Q$--restricted, then the $Q$--primal strategy defined recursively by~\eqref{eq:Qprimal}
     approaches the target set $\C_Q$, i.e.,
     \[
      \lim_{T \to \infty} \frac{1}{T}\sum_{t=1}^{T} d(\bq_t,Q) = 0 \as \qquad \mbox{entails} \qquad \overline{\mfm}_T \longrightarrow \C_Q \as
     \]
\end{restatable}

The theorem is proved below, in Appendix~\ref{sec:Qprim-mfm}.
Its proof is an adaptation of the classic proof
by~\cite{blackwell} (see also the exposition by
\citealp[Chapter 2]{mertens2015repeated}),
with some immediate adaptations
to handle
the $Q$--restriction on the opponent's play.

\subsection{Specific statement in the exchangeable setting}
\label{sec:Qprim-m}

We now instantiate the $Q$--primal strategy for adaptive conformal inference
in the case of the exchangeable setting, where $Q = \{\bq_{\unif}\}$.
We thus consider the specific vector payoff function $m$ of Theorem~\ref{th:finite-game}. \medskip

When exchangeability is known to hold, the natural strategy is the classic conformal procedure,
which plays $a_t = r(\alpha)$ at every round $t$.
We show that the $\{\bq_{\unif}\}$--primal strategy does not reduce to this procedure:
see Example~\ref{ex:qunif-primal-notreduc} below.

In contrast, the opportunistic approachability strategy of Algorithm~\ref{algo:blackwell}
adapts to this favorable regime: whenever the calibrated forecaster outputs
a forecast $\bz_t$ sufficiently close to $\bq_{\unif}$
it picks $a_t = r(\alpha)$,
exactly as in classic conformal prediction.
Indeed, the final part of the proof
of Corollary~\ref{cor:main}, located in Appendix~\ref{sec:proof-cor:main},
shows that $\quant(\alpha, \bq) = r(\alpha)$ for distributions $\bq$
in a $\min\bigl\{\alpha - r(\alpha), \, r(\alpha) - \alpha + 1/|\cB| \bigr\}$--neighborhood
of $\bq_{\unif}$ in $\ell_1$--norm.

\begin{example}[$\{\bq_{\unif}\}$--primal strategy]
\label{ex:qunif-primal-notreduc}
    Consider $\alpha \notin \cA$. In the exchangeable setting, the $\{\bq_{\unif}\}$--primal strategy
    may be written in closed form as follows (where quadrants are as in Figure~\ref{fig:decomposition}):
    \begin{itemize}
    \item If \, $\overline{m}_{t-1} = \bigl(r(\alpha), \, \Leb\,(C_{1-r(\alpha)})\bigr)^{\top}\!,$ then $\bp_t$ can be any distribution in $\Delta(\cA)$;
    \item Otherwise,
    \[
    \bp_t \begin{cases}
        = \bdelta_{0} \,, \quad & \text{if} \quad \overline{m}_{t-1} \mbox{ is in the bottom right quadrant,} \\
        = \bdelta_{1} \,, \quad & \text{if} \quad \overline{m}_{t-1} \mbox{ is in the top left quadrant,} \\
        \mbox{is any } \bp \in \Delta(\cA_t) \,, \quad & \text{if} \quad \overline{m}_{t-1} \mbox{ is in the bottom left or top right quadrant,}
    \end{cases}
    \]
    where
    \[
    \cA_t = \argmin_{a \in \cA} \Bigl\{ \bigl(\Pi_1(\overline{m}_{t-1})-r(\alpha)\bigr) a +
    \bigl(\Pi_2(\overline{m}_{t-1})-\Leb\,(C_{1-r(\alpha)})\bigr) \Leb\,(C_{1-a}) \Bigr\}
    \]
    and $\Pi_1(\,\cdot\,)$ and $\Pi_2(\,\cdot\,)$ denote respectively the
    extraction of the first and second coordinates of a vector in $\R^2$.
    \end{itemize}
\end{example}

\begin{proof}
    Fix $\alpha \notin \cA$.
    In the exchangeable setting, we have $\C_{\alpha,Q} = \{r(\alpha)\} \times \bigl\{\Leb\,(C_{1-r(\alpha)})\bigr\}$
    since $m^{\star}$ is continuous at $\bq_{\unif}$ as $\alpha \notin \cA$;
    see the proof of Corollary~\ref{cor:main}.
    Thus, we get that
    \begin{align*}
    \overline{m}_{t-1} - \projCQ(\overline{m}_{t-1}) =
    \begin{pmatrix}
        \Pi_1\bigl(\overline{m}_{t-1}\bigr) - r(\alpha) \\
        \Pi_2\bigl(\overline{m}_{t-1}\bigr) - \Leb\,\bigl(C_{1-r(\alpha)}\bigr)
    \end{pmatrix} \,.
    \end{align*}
    Resorting to the definition of the payoff function together with the elementary fact that
    \[
    \forall a \in \cA, \qquad
    \sum_{b \in \cB} \, \indic_{b \leq a} = (T_{\calib}+1) \, a\,,
    \]
    we get $m(a,\bq_{\unif}) = \bigl(a, \, \Leb\,(C_{1-a})\bigr)^{\top}\!$ for all $a \in \cA$.
    Therefore, we obtain the following rewriting, for all $\bp \in \Delta(\cA)$,
    \begin{multline*}
        \langle \overline{m}_{t-1} - \projCQ(\overline{m}_{t-1}) , m(\bp,\bq_{\unif}) \rangle \\
        = \sum_{a \in \cA} p_a \, \Bigl(\bigl(\Pi_1(\overline{m}_{t-1})-r(\alpha)\bigr) a +
        \bigl(\Pi_2(\overline{m}_{t-1})-\Leb\,(C_{1-r(\alpha)})\bigr) \Leb\,(C_{1-a}) \Bigr) \,.
    \end{multline*}
    The claimed case decomposition comes from the fact that the above expression is linear in $\bp$,
    so that $\bp_t$ can be any distribution with support on the set $\cA_t$
    of minimizers of the following function:
    \[
        a \in \cA \longmapsto \bigl(\Pi_1(\overline{m}_{t-1})-r(\alpha)\bigr) a +
        \bigl(\Pi_2(\overline{m}_{t-1})-\Leb\,(C_{1-r(\alpha)})\bigr) \Leb\,(C_{1-a}) \,.
    \]
    \begin{figure}[t]
    \usetikzlibrary{patterns}
\centering
\begin{tikzpicture}[scale=4.0, x = 2cm, y = 0.5cm]
	\def\L{1}
	\def\alphaCoord{0.58}
	\def\rCoord{0.22}

	\fill[green!15,pattern=horizontal lines,pattern color=green!55]
		(0,0) rectangle (\rCoord,2*\L*\alphaCoord);
	\fill[red!15,pattern=dots,pattern color=red!60]
		(\rCoord,0) rectangle (1,2*\L*\alphaCoord);
	\fill[blue!15,pattern=north west lines,pattern color=blue!60]
		(0,2*\L*\alphaCoord) rectangle (\rCoord,2*\L);
	\fill[green!15,pattern=horizontal lines,pattern color=green!55]
		(\rCoord,2*\L*\alphaCoord) rectangle (1,2*\L);

	\draw[thick, color=blue!60] (0,2*\L) -- (\rCoord,2*\L);
	\draw[thick, color=blue!60] (0,2*\L*\alphaCoord) -- (0,2*\L);
	\draw[thick, color=green!55] (\rCoord,2*\L) -- (1,2*\L);
	\draw[thick, color=green!55] (1,2*\L*\alphaCoord) -- (1,2*\L);
	\draw[thick, color=green!55] (0,0) -- (0,2*\L*\alphaCoord);
	\draw[thick, color=green!55] (0,0) -- (\rCoord,0);
	\draw[thick, color=red!60] (\rCoord,0) -- (1,0);
	\draw[thick, color=red!60] (1,2*\L*\alphaCoord) -- (1,0);

	\draw[thick,dashed,color=blue!60] (\rCoord,2*\L*\alphaCoord) -- (\rCoord,2*\L);
	\draw[thick,dashed,color=blue!60] (0,2*\L*\alphaCoord) -- (\rCoord,2*\L*\alphaCoord);

	\draw[thick,dashed,color=red!60] (\rCoord,2*\L*\alphaCoord) -- (1,2*\L*\alphaCoord);
	\draw[thick,dashed,color=red!60] (\rCoord,0) -- (\rCoord,2*\L*\alphaCoord);

	\fill (\rCoord,2*\L*\alphaCoord) circle (0.8pt)
	node[below right]{};

	\node[below left] at (0,0) {$0$};
	\node[below] at (1,0) {$1$};
	\node[left] at (0,2*\L) {$2L$};
	\node[left] at (0,2*\L*\alphaCoord) {$\Leb\,(C_{1-\alpha})$};
	\node[below] at (\rCoord,0) {$r(\alpha)$};

	\node[align=center] at (0.115,1.55) {top \\ left};
	\node at (0.6,0.55) {bottom right};
	\node[align=center] at (0.115,0.55) {bottom \\ left};
    \node at (0.6,1.55) {top right};
\end{tikzpicture}
    \caption{\label{fig:decomposition}
    Decomposition of $[0,1] \times [0,2L]$ into quadrants and a central point;
    the blue and red dotted lines belong to their quadrants
    while the black dot is the singleton set $\C_{\alpha,Q}$
    and does not belong to any of the four quadrants.
    The $\{\bq_{\unif}\}$--primal strategy depends on the position of
    $\overline{m}_{t-1}$ in this space. }
    \end{figure}

    In the top left and bottom right quadrants (blue and red regions in Figure~\ref{fig:decomposition}),
    the above function is strictly monotone in $a$
    (we use here also that calibration scores are all different)
    and is thus minimized by either $a = 0$ or $a = 1$.
    For instance, if the first term is negative and the second term is positive
    (i.e., $\overline{m}_{t-1}$ is in interior of the top left quadrant),
    then the above function is decreasing in $a$ and therefore minimized by $a=1$;
    this argument extends to the boundary of that quadrant (when one of the
    terms is null, which corresponds to the dotted frontier lines in
    Figure~\ref{fig:decomposition}).
    A symmetric argument holds for the bottom right quadrant and returns $a=0$.

    In the remaining cases, i.e., the green region in Figure~\ref{fig:decomposition},
    the above function is not necessarily monotone in $a$
    and may even have several minimizers, so that $\bp_t$ can be any distribution
    supported on the set of minimizers.
\end{proof}

\subsection{Proof of Theorem~\ref{th:Qprimal}}
\label{sec:Qprim-mfm}

\thQprimal*

The proof is an adaptation of the classic proof
by~\cite{blackwell} (see also the exposition by
\citealp[Chapter 2]{mertens2015repeated}),
with some immediate adaptations
to handle
the $Q$--restriction on the opponent's play. \medskip

\begin{proof}
    The proof is an immediate adaptation of the classic proof
    by~\cite{blackwell} to handle
    the $Q$--restriction on the opponent's play.

    We decompose the proof in four steps:
    we first use the dual condition together with
    a minimax argument to
    get a supporting-hyperplane inequality;
    second, we use this inequality together
    with an induction argument to
    derive an upper bound
    on the distance to the target set $\C_Q$;
    third, we invoke the $Q$--restriction property
    and get the convergence in probability (the adaptations
    are mainly located here);
    fourth, we conclude by studying the convergence of a non-negative supermartingale,
    (as in~\citealp[Chapter 2]{mertens2015repeated}), to get the
    desired almost-sure convergence. \medskip

    \emph{Step 1.} We show that,
    for all $t \geq 1$,
    the mixed action
    $\bp_{t+1} \in \Delta(\cA)$
    defined in~\eqref{eq:Qprimal} satisfies
        \begin{equation}
			\forall \bq \in \cl(Q)\,, \qquad
            \langle \overline{\mfm}_{t} - \projCQ(\overline{\mfm}_{t}), \mfm(\bp_{t+1},\bq) - \projCQ(\overline{\mfm}_{t}) \rangle \leq 0 \,. \label{eq:vanneumann}
		\end{equation}
    Fix $t \geq 1$ and $\bq \in \cl(Q)$.
    The construction of the response function
    $\mfp^{\star}$ in~\eqref{eq:responseandgoalfunctions}
    and of the target set $\C_Q$
    in~Definition~\eqref{def:image} ensure that
    $\mfm\bigl(\mfp^{\star}(\bq),\bq\bigr) \in \C_Q$.
    Hence, by the characterization of the Euclidean
    projection onto a closed convex set
    (Hilbert projection theorem applied to
    the non-empty closed convex set $\C_Q$),
    we get that,
    \[
		\langle \overline{\mfm}_{t} - \projCQ(\overline{\mfm}_{t}), \mfm\bigl( \mfp^{\star}(\bq),\bq \bigr) - \projCQ(\overline{\mfm}_{t}) \rangle \leq 0 \, .
	\]
	In particular, it implies that the
    minimum over $\bp \in \Delta(\cA)$
    of the left-hand side of the above inequality
    is also non-positive.
    The latter holds for all $\bq \in \cl(Q)$,
    which by the Von Neumann minimax theorem,
    applied to the convex and compact sets
    $\Delta(\cA)$ and $\cl(Q)$,
    entails that
    \[
		\min_{\bp \in \Delta(\cA)\vphantom{\cl(Q)}} \max_{\bq \in \cl(Q)} \,
        \langle \overline{\mfm}_{t} - \projCQ(\overline{\mfm}_{t}),
        \mfm( \bp,\bq ) - \projCQ(\overline{\mfm}_{t}) \rangle \leq 0 \, .
	\]
    The definition of $\bp_{t+1}$ in~\eqref{eq:Qprimal} concludes the first step of the proof.
\medskip

        \emph{Step 2.}
        We denote
        $d_t = d\bigl(\overline{\mfm}_{t}, \C_Q\bigr)$,
        for all $t \geq 1$,
        and show that, for all $T \geq 1$,
        \begin{align}
            \E\bigl[d_{T+1}^2 \bigr] &\leq
            \frac{2}{T+1} \lVert \mfm \rVert_{\infty}^2 +
            \frac{2}{T(T+1)}\sum_{t=1}^{T} t \, \E\bigl[\nu_t(Q) \bigr] \,. \label{eq:restprimalrecu}
        \end{align}
        where $\lVert \mfm \rVert_{\infty} = \max\bigl\{\lVert \mfm(a,b) \rVert : (a,b) \in \cA \times \cB\bigr\}$
        is the uniform bound on the payoff function
        and $\nu_t(Q) = \langle \mfm(\bp_{t+1},\bq_{t+1}) - \mfm\bigl(\bp_{t+1},\bq_{t+1}'\bigr) , \overline{\mfm}_{t}  - \projCQ(\overline{\mfm}_{t}) \rangle$,
        with $\bq_{t+1}'$ being the orthogonal projection
        of $\bq_{t+1} \in \Delta(\cB)$ onto $\cl(Q)$.
        \bigskip

        Fix $t\geq 1$.
		We first use that $d_{t+1} \leq \lVert \overline{\mfm}_{t+1} - \projCQ(\overline{\mfm}_{t})\rVert$,
        then that $\lVert \overline{\mfm}_{t+1} - \overline{\mfm}_{t} \rVert
        = \frac{1}{t+1}\lVert\mfm(a_{t+1},b_{t+1})-\overline{\mfm}_{t}\rVert \leq \frac{2}{t+1}\lVert \mfm \rVert_{\infty}$,
        and get the following recurrent inequality:
		\begin{align}
			d_{t+1}^2 \leq d_t^2 + \frac{2}{t+1} \lVert \mfm \rVert_{\infty}^2 +
            \frac{2}{t+1} \langle \mfm(a_{t+1},b_{t+1})-\overline{\mfm}_{t} , \overline{\mfm}_{t}  - \projCQ(\overline{\mfm}_{t}) \rangle \,.
            \label{eq:recurrentinequality}
	    \end{align}
        Let $c_t = \overline{\mfm}_{t}  - \projCQ(\overline{\mfm}_{t})$.
        By a simple algebraic manipulation, we get that
        \begin{multline*}
            \langle \mfm(a_{t+1},b_{t+1}) - \overline{\mfm}_{t} , c_t \rangle
            = \langle \mfm(a_{t+1},b_{t+1}) - \mfm(\bp_{t+1},\bq_{t+1}) , c_t \rangle  \\
            + \underbrace{\langle \mfm(\bp_{t+1},\bq_{t+1}) - \mfm\bigl(\bp_{t+1},\bq_{t+1}'\bigr) , c_t \rangle}_{
                \leq 0 \, \footnotesize\mbox{ by inequality~\eqref{eq:vanneumann}, obtained in Step~1}}
                            + \langle \mfm\bigl(\bp_{t+1},\bq_{t+1}'\bigr) - \overline{\mfm}_{t} , c_t \rangle - d_t^2\,.
        \end{multline*}
        Let $\F_t = \sigma\bigl((a_{\tau},b_{\tau})_{\tau \leq t}\bigr)$,
        be the natural filtration
        so that $\bp_{t+1}$, $d_t$ and $\overline{\mfm}_t$ are $\F_t$--measurable.
        Since $\mathbb{E}\bigl[\langle \mfm(a_{t+1},b_{t+1}) - \mfm(\bp_{t+1},\bq_{t+1}) , c_t \rangle \, \big| \mathcal{F}_t \bigr] = 0$
        in the above decomposition,
        we consider the conditional expectation with
        respect to $\F_t$ in~\eqref{eq:recurrentinequality},
        and get that
        \begin{align}
            \E \bigl[d_{t+1}^2 \big| \F_t \bigr] \leq \Bigl(1-\frac{2}{t+1} \Bigr)d_t^2 + \frac{2}{(t+1)^2} \lVert \mfm \rVert_{\infty}^2
            + \frac{2}{t+1} \E\bigl[ \nu_t(Q) \big| \F_t \bigr] \,. \label{eq:generalineq}
        \end{align}
        We can now unfold this induction inequality to get
        to the desired bound~\eqref{eq:restprimalrecu}.
        Let $u_t = \E\bigl[d_{t}^2 \bigr]$ and
        $v_t =  \frac{2}{(t+1)^2} \lVert \mfm \rVert_{\infty}^2 + \frac{2}{t+1}\E\bigl[ \nu_t(Q) \bigr]$.
        We take the expectation in~\eqref{eq:generalineq}, and get that
        \begin{align}
            u_{t+1} \leq \Bigl(1-\frac{2}{t+1} \Bigr) \, u_t + v_t = \frac{t-1}{t+1} \, u_t + v_t \,. \label{eq:to_unfold}
        \end{align}
        By induction, we show that the following inequality holds, for all $T \geq 1$:
        \begin{align}
    		u_{T+1} \leq \sum_{t=1}^T \frac{t(t+1)}{T(T+1)} \, v_t \label{eq:unfolded} \,.
        \end{align}
        Indeed, \eqref{eq:unfolded} holds for $T=1$
        and assume it holds at round $T \geq 1$,
        then by~\eqref{eq:to_unfold}
        \begin{align*}
            u_{T+2} &\leq \frac{T}{T+2}\sum_{t=1}^T v_t\frac{t(t+1)}{T(T+1)} + v_{T+1} \leq \sum_{t=1}^{T+1} v_t\frac{t(t+1)}{(T+1)(T+2)} \,,
        \end{align*}
        which concludes the proof by induction and also
        directly implies the desired bound~\eqref{eq:restprimalrecu}.
        \medskip

        \emph{Step 3.} We invoke the $Q$--restriction property
        on the opponent's play, and show that
        \[
            \lim_{T \to \infty} \sum_{t=1}^{T} \frac{2}{t+1} \nu_t(Q) =  0 \as\,,
            \qquad \mbox{and thus} \qquad
            \lim_{T \to \infty} \E\big[d_T^2\big] = 0 \as
        \]
        By standard norm inequalities, we have that, for all $\bq$, $\bq' \in \Delta(\cB)$,
        and all $t \geq 1$,
        \[
            \lVert \mfm(\bp_{t+1},\bq) - \mfm\bigl(\bp_{t+1}, \bq'\bigr) \rVert \leq \lVert \mfm \rVert_{\infty} \lVert \bq - \bq' \rVert_1 \leq \lVert \mfm \rVert_{\infty} \sqrt{|\cB|} \, d(\bq, \bq') \,.
        \]
        Therefore, by Cauchy-Schwarz inequality, we obtain that
        \begin{align*}
            \nu_t(Q) &= \langle \mfm(\bp_{t+1},\bq_{t+1}) - \mfm\bigl(\bp_{t+1},\bq_{t+1}'\bigr) , \overline{\mfm}_{t}  - \projCQ(\overline{\mfm}_{t}) \rangle \\
            &\leq \bigl(2 \lVert \mfm \rVert_{\infty}^2 \sqrt{|\cB|} \bigr) \, d\bigl(\bq_{t+1}, Q\bigr) \,
        \end{align*}
        As a consequence, we get that, for all $T \geq 1$,
        \[
            \sum_{t=1}^{T} \frac{2}{t+1} \nu_t(Q) \leq
            \bigl(4 \lVert \mfm \rVert_{\infty}^2 \sqrt{|\cB|} \bigr) \, \frac{1}{T}\sum_{t=1}^{T}  d\bigl(\bq_{t+1}, Q\bigr)  \,,
        \]
        which converges to $0$ almost surely by
        the $Q$--restriction property on the opponent's play.
        Thus, by the upper bound~\eqref{eq:restprimalrecu}
        obtained in Step~2, we also get that $d_T^2$
        converges to $0$ in probability, i.e.,
        $\lim_{T \to \infty} \E\bigl[d_T^2\bigr] \to 0$.
        \medskip

        \emph{Step 4.} Finally, we show that $d_t$ converges to $0$ almost surely.
        To do so, we consider the following non-negative
        supermartingale $\bigl(w_t\bigr)_{t \geq 1}$,
        that uniformly bounds $d_t^2$.
        \begin{align}
            w_t = d_{t}^2 + \sum_{\tau = t}^{\infty} \frac{2}{(\tau+1)^2} \lVert \mfm \rVert_{\infty}^2  +
            \sum_{\tau = t}^{\infty} \frac{2}{\tau+1}\E\bigl[ \nu_\tau(Q) \big| \F_t \bigr] \,. \label{eq:supermartingale}
        \end{align}
        We first check that $\bigl(w_t\bigr)_{t \geq 1}$
        is indeed a non-negative supermartingale:
        we take the conditional expectation
        in~\eqref{eq:supermartingale}, and
        use the tower rule
        $\E\bigl[ \E[\cdots | \F_{t+1}] | \F_{t} \bigr] = \E[\cdots | \F_{t}] $,
        together with the recurrent
        inequality~\eqref{eq:generalineq} obtained in Step 2
        to get that
        \begin{align*}
            \E\bigl[w_{t+1} \big| \F_t \bigr] &= \E\bigl[d_{t+1}^2 \big| \F_t \bigr]
            + \sum_{\tau = t+1}^{\infty} \frac{2}{(\tau+1)^2} \lVert \mfm \rVert_{\infty}^2
            + \sum_{\tau = t+1}^{\infty} \frac{2}{\tau+1}\E\bigl[ \nu_\tau(Q) \big| \F_t \bigr] \\
            &\leq d_t^2
            + \sum_{\tau = t}^{\infty} \frac{2}{(\tau+1)^2} \lVert \mfm \rVert_{\infty}^2
            + \sum_{\tau = t}^{\infty} \frac{2}{\tau+1}\E\bigl[ \nu_\tau(Q) \big| \F_t \bigr]
            = w_t \,.
        \end{align*}
        The supermartingale is defined for all $t \geq 1$
        as the sum of $d_t^2$,
        which converges to $0$ in probability by Step~3,
        and two remainder terms of convergent series
        (the convergence of the second series
        is also granted by Step~3),
        so that $w_t$ converges to $0$ in probability.

        Finally, by the martingale convergence theorem,
        $w_t$ converges almost surely to a finite limit,
        that is necessarily the same as the limit in probability, i.e., $0$,
        which concludes the proof since $d_t^2 \leq w_t$, for all $t \geq 1$.
\end{proof}

\section{A Discussion and Simpler Proofs for the Opportunistic \\
\phantom{Appendix~D. }Approachability Theory by \citet{bernstein2014opportunistic}}
\label{app:go-over-Bernstein}

We consider an extension of the general setting of Appendix~\ref{sec:opp-app} that is designed to handle possibly non-convex target sets
(known as D-sets), just as in~\citet{bernstein2014opportunistic}.
The point of this appendix is to reprove the main result of the latter article,
which corresponds to Theorem~\ref{th:general} of Appendix~\ref{sec:gal-theory}
(and is actually a general version of the specific convergence result
established in Theorem~\ref{th:main}).
Theorem~\ref{th:general} provides
some general construction and analysis of opportunistically approaching strategies
in the sense of Definition~\ref{def:opp-app} (extended to D-sets).

There are two differences between the original
version of the result by \citet{bernstein2014opportunistic} and our version
in Theorem~\ref{th:general}.

First, we require that closures are taken in the defining expression of $\C_Q$
(see Definition~\ref{def:image}, unlike \citealp[Definition~8]{bernstein2014opportunistic}),
i.e., we shoot for a slightly less ambitious target. However, we prove in Appendix~\ref{app:counterexample}
the necessity of taking these closures, through a counter-example,
while also explaining in
Appendix~\ref{sec:B2014-special} that taking or not taking closures is indifferent
in several natural cases.

Second, we relax the assumption on the response function $\mfp^\star$
into a mere measurability requirement (\citealp[Theorem 4]{bernstein2014opportunistic}
assumed that $\mfp^\star$ is piecewise continuous).

\paragraph{Setting, with D-sets.}
The formal setting throughout this appendix is the following.
Let $\cA$ and $\cB$ be finite action sets, $\mfm : \cA \times \cB \to \R^d$ be a general vector-valued payoff function, and $\cC \subset \R^d$ be a general D-set.

A closed set $\cC$ is called a D-set if for all $\bq \in \Delta(\cB)$, there exists $\bp \in \Delta(\cA)$ such that $\mfm(\bp,\bq) \in \cC$.
For any D-set $\cC$, we can thus define a response function $\mfp^\star : \Delta(\cB) \to \Delta(\cA)$ and an associated goal function $\mfm^{\star} : \Delta(\cB) \to \C$ such that
\begin{align}
    \forall \bq \in \Delta(\cB), \qquad \mfm^{\star}(\bq) \eqdef \mfm\bigl(\mfp^\star(\bq),\bq\bigr) \in \C \,. \label{eq:responseandgoalfunctions2}
\end{align}
This construction generalizes~\eqref{eq:responseandgoalfunctions} to non-convex target sets. The response function $\mfp^\star$ defined in Equation~\eqref{eq:responseandgoalfunctions2}
will only be required to be measurable.

A repeated game takes place.
At each round $t \geq 1$, the learner and the opponent simultaneously pick actions $a_t \in \cA$
and $b_t \in \cB$. The learner obtains the payoff $\mfm(a_t,b_t)$ and observes
$b_t$, not only $\mfm(a_t,b_t)$ as in Section~\ref{sec:Blk}.
(This is because the strategy considered in Theorem~\ref{th:general}
relies on an auxiliary calibrated strategy~$\mathcal{F}$, forecasting the $b_t$.)

\subsection{General Theory, with Closures}
\label{sec:gal-theory}

For the sake of clarity, we recast the general statement of the opportunistic approachability procedure in Algorithm~\ref{algo:general}, as introduced by \citet{bernstein2014opportunistic}.

\begin{algorithm}[t]
	\caption{\label{algo:general} General opportunistic approachability procedure}
	\begin{algorithmic}
		\STATE\textbf{Inputs:} response function $\mfp^\star$; auxiliary sequential forecaster $\mathcal{F}$ with outputs in $\Delta(\cB)$
        \STATE\textbf{Initialization:} play $a_1$ uniformly at random over $\cA$ \\
		\STATE\textbf{At time steps $t=2, 3, \, \ldots\,$:} \\
		\quad 1. Get from $\mathcal{F}$ a probabilistic forecast $\bz_t \in \Delta(\cB)$ of the opponent's pure action \\			
		\quad 2. Play $a_t \sim \mfp^\star(\bz_t)$ \\
		\quad 3. Receive $\mfm(a_t,b_t)$, observe $\mfm(a_t,b_t)$ and $b_t$,
and feed $\mathcal{F}$ with $b_t$
	\end{algorithmic}
\end{algorithm}

\begin{theorem}[General opportunistic approachability]\label{th:general}
If the response function $\mfp^\star : \Delta(\cB) \to \Delta(\cA)$
satisfying the dual condition~\eqref{eq:responseandgoalfunctions2}
is measurable, then the strategy stated in Algorithm~\ref{algo:general}, when run with a calibrated forecaster, is an opportunistically approaching strategy for the D-set $\C$, in the sense of Definition~\ref{def:opp-app}, namely:
for any convex subset $Q \subseteq \Delta(\cB)$,
whenever the opponent is $Q$--restricted,
    \[
    \overline{\mfm}_T \longrightarrow \C_Q \qquad \as, \qquad \quad \mbox{as $T \to \infty$.}
    \]
\end{theorem}

\begin{proof}
    The proof follows the same lines and uses the same notation as the proof of Theorem~\ref{th:main}
    in Appendix~\ref{sec:proof-th-main}, the only change being the choice of the partition in the first step.
    In particular, we start with a similar triangular inequality,
    \begin{equation}
    \label{eq:firstsum-general-proof}
        d\bigl(\overline{\mfm}_T,\C_Q\bigr) \leq \left\lVert \overline{\mfm}_T-\frac{1}{T}\sum_{t=1}^T \mfm\bigl(\mfp^\star(\bz_t),\bz_t\bigr)\right\rVert + d\!\left(\frac{1}{T}\sum_{t=1}^T \mfm\bigl(\mfp^\star(\bz_t),\bz_t\bigr), \C_Q\right) \,,
    \end{equation}
    where the convergence of the second term, involving the quantities $\mfm\bigl(\mfp^\star(\bz_t),\bz_t\bigr) = \mfm^\star(z_t)$,
    is ensured by Equation~\eqref{eq:ccl-pr-thmain-step2} and Appendix~\ref{app:geometric}, as both relied
    on general arguments (not involving the specific definition of $\mfm$ in Appendix~\ref{sec:proof-th-main}).

    Therefore, we only need to show the convergence of the first term in Equation~\eqref{eq:firstsum-general-proof}.
    First, by martingale convergence (e.g., the Hoeffding-Azuma inequality together with the Borel-Cantelli lemma),
    \[
    \limsup_{T \to \infty}
    \left\lVert \overline{\mfm}_T-\frac{1}{T}\sum_{t=1}^T \mfm\bigl(\mfp^\star(\bz_t),\bz_t\bigr)\right\rVert
    = \limsup_{T \to \infty}
    \left\lVert \frac{1}{T}\sum_{t=1}^T \mfm\bigl(\mfp^\star(\bz_t),b_t\bigr) - \frac{1}{T}\sum_{t=1}^T \mfm\bigl(\mfp^\star(\bz_t),\bz_t\bigr)\right\rVert.
    \]
    Fix $\epsilon > 0$, and consider a disjoint $\epsilon$--covering $\bigl\{P_k : k \in [N]\bigr\}$ of $\Delta(\cA)$ as per Definition~\ref{def:disjointcovering}, with associated centers $\bigl\{\bc_k : k \in [N]\bigr\}$.
    This is the main change with respect to the proofs of Theorems~\ref{th:main} and~\ref{th:mainepsilon}:
    the covering is now on $\Delta(\cA)$, and not on $\Delta(\cB)$.
    For each $k \in [N]$, we define the preimage of $P_k$ by $\mfp^\star$ as $Q_k = \bigl\{ \bq \in \Delta(\cB) : \mfp^\star(\bq) \in P_k \bigr\}$;
    by measurability of $\mfp^\star$, the subset $Q_k$ is also measurable.
    We group rounds according to the disjoint $\epsilon$--covering:
    \[
    \frac{1}{T} \sum_{t=1}^{T} \mfm\bigl(\mfp^\star(\bz_t),\bz_t\bigr) = \frac{1}{T} \sum_{t=1}^{T} \sum_{k=1}^N \indic_{\bz_t \in Q_k} \, \mfm\bigl(\mfp^\star(\bz_t),\bz_t\bigr) \,.
    \]
    We recall that for all $k \in [N]$,
    \[
    \forall \bq \in Q_k, \qquad \bigl\Arrowvert \mfp^\star(\bq) - \bc_k \bigr\Arrowvert_1 \leq \epsilon\,,
    \qquad \mbox{which entails} \qquad
    \bigl\Arrowvert \mfm \bigl( \mfp^\star(\bq), \bq \bigr) - \mfm(\bc_k,\bq) \bigr\Arrowvert \leq \rho\epsilon\,,
    \]
    where we denoted $\smash{\rho = \max \Bigl\{ \bigl\lVert \mfm(a,b) \bigr\rVert : (a, b) \in \cA \times \cB \Bigr\} < \infty}$.
    Therefore, by a triangular inequality,
    \[
    \left\lVert \frac{1}{T} \sum_{t=1}^{T} \sum_{k=1}^N \indic_{\bz_t \in Q_k} \, \mfm\bigl(\mfp^\star(\bz_t),\bz_t\bigr) -
    \frac{1}{T}\sum_{t=1}^T \sum_{k=1}^N \indic_{\bz_t \in Q_k} \mfm(\bc_k,\bz_t)
    \right\rVert \leq \rho \epsilon\,.
    \]
    A similar argument may be applied to the average of the $\mfm\bigl(\mfp^\star(\bz_t),b_t\bigr)$,
    and we get, by a triangular inequality, that the quantity of interest is bounded as
    \begin{align*}
    \lefteqn{\left\lVert \frac{1}{T}\sum_{t=1}^T \mfm\bigl(\mfp^\star(\bz_t),b_t\bigr) - \frac{1}{T}\sum_{t=1}^T \mfm\bigl(\mfp^\star(\bz_t),\bz_t\bigr)\right\rVert} \\
    & \leq 2\rho\epsilon +
    \left\lVert \frac{1}{T}\sum_{t=1}^T \sum_{k=1}^N \indic_{\bz_t \in Q_k} \mfm(\bc_k,b_t) - \frac{1}{T}\sum_{t=1}^T \sum_{k=1}^N \indic_{\bz_t \in Q_k} \mfm(\bc_k,\bz_t) \right\rVert \\
    & \leq 2\rho\epsilon +
    \rho \sum_{k=1}^N \left\lVert \frac{1}{T}\sum_{t=1}^T \indic_{\bz_t \in Q_k} \bigl( \bdelta_{b_t} - \bz_t \bigr) \right\rVert_1,
    \end{align*}
    where the second inequality follows again by a triangular inequality and by the boundedness of $\mfm$.
    Now, by a triangular inequality, by a standard norm inequality,
    and by $N$ applications of the definition of the calibration of $\mathcal{F}$ (this is where we use that the $Q_k$ are measurable),
    \[
    \sum_{k=1}^N \left\lVert \frac{1}{T}\sum_{t=1}^T \indic_{\bz_t \in Q_k} \bigl(\bdelta_{b_t} - \bz_t\bigr) \right\rVert_1
    \leq \sqrt{|\cB|} \, \sum_{k=1}^N \left\lVert \sum_{t=1}^{T} \indic_{\bz_t \in Q_k} \bigl(\bz_t - \bdelta_{b_t}\bigr)\right\rVert \longrightarrow 0 \as
    \]
    Collecting all elements, we showed so far that for all $\epsilon > 0$,
    \[
    \limsup_{T \to \infty}
    \left\lVert \overline{\mfm}_T-\frac{1}{T}\sum_{t=1}^T \mfm\bigl(\mfp^\star(\bz_t),\bz_t\bigr)\right\rVert
    \leq 2\rho\epsilon \as,
    \]
    from which the desired conclusion follows, by letting $\epsilon \to 0$.
\end{proof}

\subsection{Special Case where Closures Are Not Necessary}
\label{sec:B2014-special}

Theorem~\ref{th:general} and the underlying Definitions~\ref{def:image} and~\ref{def:opp-app}
entail the convergence of $\overline{\mfm}_T$ to
\[
\C_Q = \bigcap_{\eta>0}\cl\biggl(
\conv\Bigl(\bigl\{ \mfm^\star(\bq) : \bq \in \Delta(\cB) \ \ \mbox{s.t.} \ \ d(\bq,Q) \leq \eta \bigr\}\Bigr)
\biggr),
\]
where $\mfm^{\star}(\bq) = \mfm\bigl( \mfp^\star(\bq),\bq)$ and when the opponent plays in a $Q$--restricted way.

Sometimes (but not always, see Appendix~\ref{app:counterexample}), convergence to a smaller target set $\C'_Q \subseteq \C_Q$ may
be achieved; this set $\C'_Q$ is defined
by not taking closures in the definition of $\C_Q$ and is the set considered in \citet{bernstein2014opportunistic}.
Formally,
\begin{align}
\label{eq:alternative}
\C_Q' = \bigcap_{\eta > 0} \conv\Bigl(\bigl\{ \mfm^\star(\bq) : \bq \in \Delta(\cB) \ \ \mbox{s.t.} \ \ d(\bq,Q) \leq \eta \bigr\}\Bigr) \,.
\end{align}

We may indifferently consider $\C'_Q$ or $\cl(\C'_Q)$, as the distance of a point to a set is equal to the distance
of a point to the closure of that set.
The closure $\cl(\C'_Q)$ of the intersection of some sets is always contained in the intersection $\C_Q$ of the closures of these sets. The question is to determine when the converse is true.
We recall that the dimension of a convex set $C$ is the dimension of the affine subspace that it generates.

\begin{lemma}
If $\interior(\C_Q') \ne \emptyset$, then $\cl(\C_Q') = \C_Q$,
and thus, the convergence of $\overline{\mfm}_T$ to $\C_Q$ in Theorem~\ref{th:general}
corresponds exactly to the convergence of $\overline{\mfm}_T$ to $\C_Q'$.

The equality holds more generally if the dimension of $\C_Q'$
equals the limit of the dimensions of the sets $C_n$ defined in the proof.
\end{lemma}

\begin{proof}
By nestedness, it suffices to
apply the general result of Lemma~\ref{lm:closure-intersection} (which we state and prove next)
to the sequences of convex sets
\[
C_n = \conv\Bigl(\bigl\{ \mfm^\star(\bq) : \bq \in \Delta(\cB) \ \ \mbox{s.t.} \ \ d(\bq,Q) \leq 1/n \bigr\}\Bigr)\,,
\]
indexed by $n \geq 1$.
\end{proof}

    The following result is perhaps a well-known result but we could find no exactly similar statement of it; therefore,
    we provide a (simple but) complete proof thereof.

    \begin{lemma}\label{lm:closure-intersection}
        Let $\bigl(C_n\bigr)_{n \geq 1}$ be a sequence of convex subsets of $\R^d$ such that its intersection $C$ has a non-empty interior;
        then the closure of $C$ is the intersection of the closures of the $C_n$, i.e.,
        \[
        \mbox{let} \quad  C = \bigcap_{n \geq 1} C_n\,: \qquad \quad
        \interior(C) \ne \emptyset \quad \Longrightarrow \quad
        \cl(C) = \bigcap_{n \geq 1} \cl(C_n) \,.
        \]
        The result holds more generally if the convex sets $C_n$ are nested and if the limit of their dimensions as $n \to \infty$ equals the dimension of $C$.
    \end{lemma}

    \begin{proof}
    We prove the first statement of the lemma, under the condition that $\interior(C) \ne \emptyset$.
    As the intersection of closed sets always contains the closure of the intersection of sets
    (with no additional condition), it suffices to prove that
    \[
    \cl(C) \supseteq \bigcap_{n \geq 1} \cl(C_n)\,.
    \]
    We do so by contradiction: suppose that there exists $x \in \displaystyle{\bigcap_{n \geq 1} \cl(C_n)}$ such that $x \notin \cl(C)$.

        By the strong separation theorem for convex sets (e.g., \citealp[Theorem 2.39]{rockafellar1998variational}, also
        called the Hahn-Banach theorem), applied to the closed convex set $\cl(C)$ and to the singleton $\{x\}$, which is a compact convex set,
        there exists a hyperplane $\mathcal{H} = \bigl\{ u \in \R^d : h(u) = 0 \bigr\}$ parameterized by an affine function $h$ that strictly separates $x$ from $\cl(C)$, i.e.,
        there exists $\delta > 0$ such that
        \[
        \{x\} \subset \cH^+ \defeq \bigl\{ u \in \R^d : h(u) > \delta \bigr\}
        \qquad \mbox{and} \qquad
        \cl(C) \subseteq \cH^- \defeq \bigl\{ u \in \R^d : h(u) < {-\delta} \bigr\} \,.
        \]
        The idea of the proof consists in constructing a point $z$ that belongs to $\cH$ and to $\cH^-$, which is impossible.

        We use first the assumption of non-empty interior: take $y \in \interior(C)$. By definition, $C \subseteq C_n$ for all $n \geq 1$, and we thus have $y \in \interior(C_n)$ for all $n \geq 1$. Moreover, by design, $x \in \cl(C_n)$ for all $n \geq 1$, and therefore $[y,x) \subseteq \interior(C_n)$ by convexity of $C_n$ (this is a classic property, see Lemma~\ref{lm:segm} below).
        Second, $y \in C$ thus $y \in \cH^-$ while $x \in \cH^+$, so that the segment $[y,x]$ intersects
        in particular the hyperplane $\cH$: there exists $z \in (y,x) \cap \cH$.

        Finally, we use the inclusion $[y,x)\subseteq \interior(C_n)$ for all $n \geq 1$ to get that $z \in (y,x)$ also belongs to $\interior(C_n)$ thus to $C_n$, for all $n \geq 1$, and in turn, $z$ belongs to $C$. This creates the contradiction: indeed,
        $z$ belongs to $\cH$ and was proved to belong to $C \subseteq \cH^-$. \medskip

        The second statement of the lemma relies on the following facts.
        Denote by $\mathcal{H}$ the (affine) subspace generated by $C$;
        it can be identified with some $\R^{d'}$ space, up to some change of coordinates.
        By assumption, the convex sets $C_n$ are nested, thus the sequence of their dimensions is a non-increasing, integer-valued sequence.
        That its limit equals the dimension of $C$ means that there exists $N \geq 1$ such that $C$ and $C_n$ are subsets
        of $\mathcal{H}$ for all $n \geq N$. Indeed, $C$ is included in $C_n$, and thus the affine space generated by $C$ is included in the affine space generated by $C_n$; as these two spaces have the same dimension, they are equal.
        Now, $C$ has a non-empty interior when seen as a subset of $\cH$: this follows from the general fact that a non-empty convex set
        has always a non-empty interior in the affine subspace it generates (e.g., \citealp[Theorem 6.2]{rockafellar1997convex}).
        Therefore, the second statement of the lemma follows from the first one, applied to the sequence $(C_{N+n})_{n \geq 1}$ of convex subsets of
        some $\R^{d'}$ space.
    \end{proof}

\begin{lemma}[{see, e.g., \citealp[Theorem 6.1]{rockafellar1997convex}}]
\label{lm:segm}
Consider a convex $C \subseteq \R^d$ with non-empty interior; then,
\[
\forall x \in \cl(C), \ \ \forall y \in \interior(C), \qquad [y,x) \subseteq C\,.
\]
\end{lemma}

\begin{proof}
Fix $w \in [y,x)$: there exists $\lambda \in [0,1)$ such that $w = (1-\lambda) y + \lambda x$.
Since $y \in \interior(C)$, there exists $\epsilon > 0$ such that $B(y,\epsilon) \subseteq C$; in addition, the linear mapping $u \mapsto (1-\lambda) u + \lambda x$ is a homothety of center $x$ and ratio $1-\lambda$ that maps the ball $B(y,\epsilon)$ to the ball $B\bigl(w,(1-\lambda)\,\epsilon\bigr)$, which is included in $\cl(C)$ by convexity of the latter. Therefore, the ball $B\bigl(w,(1-\lambda)\,\epsilon/2\bigr)$ is included in $C$, and thus $w \in \interior(C)$, as announced.
\end{proof}

\subsection{Counter-example for the General Necessity of Closures} \label{app:counterexample}

In general, we do not have that $\cl(\C_Q')$ equals $\C_Q$.
\citet[proof of Corollary~2]{bernstein2014opportunistic} directly claim (with no proof and under the additional assumption that $\mfp^\star$
is piecewise continuous) that
\begin{equation}
\label{eq:not-OK}
\lim_{T \to \infty} \frac{1}{T}\sum_{t=1}^{T} d(\bz_t,Q) = 0 \qquad \mbox{entails} \qquad
d\!\left(\frac{1}{T}\sum_{t=1}^T \mfm^{\star}(\bz_t), \, \C'_Q\right) \longrightarrow 0 \,,
\end{equation}
where $(\bz_t)_{t \geq 1}$ is a sequence in $\Delta(\cB)$ and
where we recall that $\C'_Q$ was defined in~\eqref{eq:alternative}.

However, we found a counter-example showing that the above implication does not hold in general.
Given the results of Appendix~\ref{sec:B2014-special}, such a counter-example
requires a gap in the limit for the dimensions of the convex sets involved:
the limit of the dimensions of the
\begin{align}
    \cM_Q^\eta = \conv\Bigl(\bigl\{ \mfm^\star(\bq) : \bq \in \Delta(\cB) \ \ \mbox{s.t.} \ \ d(\bq,Q) \leq \eta \bigr\}\Bigr) \label{eq:Meta}
\end{align}
must be strictly larger than the dimension of their intersection
$\displaystyle{\C'_Q = \bigcap_{\eta > 0} \cM_Q^\eta}$.

We have $\C_Q^\eta = \cl\bigl(\cM_Q^\eta\bigr)$ with the notation of Appendix~\ref{app:geometric}
but we omit the closures in this section.
We build a simple counter-example in $\R^2$, where the $\cM_Q^\eta$ are of dimension~2 (i.e., span the entire space)
while $\C'_Q$ is of dimension~1 (i.e., it generates a line).

\begin{remark}
The counter-example was inspired by the reward function defined for ACI in Section~\ref{sec:setting};
however, for the sake of simplicity, we present it with explicit scalar values.
\end{remark}

\paragraph{Setting.}
We consider the action sets $\cA = \bigl\{ 1, 2, 3 \bigr\}$ and $\cB = \bigl\{ 1, 2, 3 \bigr\}$, and the following reward function $\mfm : \cA \times \cB \to \R^2$,
which we describe by listing all its values in a $3 \times 3$ matrix, where rows correspond
to the learner playing in $\cA$ and columns to the opponent picking actions in $\cB$:
\begin{align} \label{eq:rewardcounterexample}
    \begin{bmatrix}
    \mfm(1,1) & \mfm(1,2) & \mfm(1,3) \\
    \mfm(2,1) & \mfm(2,2) & \mfm(2,3) \\
    \mfm(3,1) & \mfm(3,2) & \mfm(3,3)
    \end{bmatrix}
      = \begin{bmatrix}
        \begin{pmatrix}
            0.1 \\
            2
        \end{pmatrix}
        &
        \begin{pmatrix}
            0.1 \\
            2
        \end{pmatrix}
        &
        \begin{pmatrix}
            0.1 \\
            2
        \end{pmatrix} \smallskip \\
        \begin{pmatrix}
            -1 \\
            3
        \end{pmatrix}
        &
        \begin{pmatrix}
            0 \\
            3
        \end{pmatrix}
        &
        \begin{pmatrix}
            0 \\
            3
        \end{pmatrix} \smallskip \\
        \begin{pmatrix}
            1 \\
            1
        \end{pmatrix}
        &
        \begin{pmatrix}
            1 \\
            1
        \end{pmatrix}
        &
        \begin{pmatrix}
            0 \\
            1
        \end{pmatrix}
    \end{bmatrix}\,.
\end{align}
The objective is to approach $\C = [-1,0.2] \times [1,3] \subseteq \R^2$.

\paragraph{Response function.}
We describe a (valid but not unique) response function $\mfp^\star$ for the row player and its associated payoff function $\mfm^{\star}$ defined by
$\bq \in \Delta(\cB) \mapsto \mfm^\star(\bq) = \mfm\bigl(\mfp^\star(\bq), \bq \bigr) \in \C$;
namely, for all $\bq \in \Delta(\cB)$,
\[
\begin{cases}
    \mfp^\star(\bq) = \bdelta_{1} \quad \mbox{and} \quad \mfm^\star(\bq) = \begin{pmatrix} 0.1 \\ 2 \end{pmatrix} \in \C & \mbox{if } q_{1} = 0 \mbox{ and } q_{1} + q_{2}  > 0.2 \,, \smallskip \\
    \mfp^\star(\bq) = \bdelta_{2} \quad \mbox{and} \quad \mfm^\star(\bq) = \begin{pmatrix} -q_1 \\ 3 \end{pmatrix} \in \C & \mbox{if } q_{1} > 0  \mbox{ and } q_{1} + q_{2}  > 0.2\,, \smallskip \\
    \mfp^\star(\bq) = \bdelta_{3} \quad \mbox{and} \quad \mfm^\star(\bq) = \begin{pmatrix} q_1+q_2 \\ 1 \end{pmatrix} \in \C & \mbox{if } q_{1} + q_{2}  \leq 0.2 \,.
\end{cases}
\]
Note that $\mfp^\star$ is piecewise constant, and thus both measurable and piecewise continuous, as required
by Theorem~\ref{th:general} and \citet[Corollary 2]{bernstein2014opportunistic}, respectively.

\paragraph{Set of $Q$--restrictions.}
Finally, we introduce the convex subset $Q$ of $\Delta(\cB)$ defined by
\[
Q = \bigl\{\bq \in \Delta(\cB) : q_{1} = 0\,, q_{2} \geq 0.2 \bigr\} \,.
\]
The core property of $Q$ is that its boundary contains the discontinuity point $\bq = (0,0.2,0.8)^{\top}$ of the response function $\mfp^\star$,
and that any neighborhood of $(0,0.2,0.8)^{\top}$ contains distributions $\bq$ leading to any of
the three possible outcomes of $\mfp^\star$.

\begin{theorem}
    In the setting above, where $\mfp^\star$ is piecewise constant, there exists a sequence $(\bz_t)_{t \geq 1}$ in~$\Delta(\cB)$ that satisfies
    \[
    \lim_{T \to \infty} \frac{1}{T}\sum_{t=1}^{T} d(\bz_t,Q) = 0 \qquad \mbox{but} \qquad \liminf_{T \to \infty} \ d\!\left(\frac{1}{T}\sum_{t=1}^T \mfm^{\star}(\bz_t), \,\C'_Q\right) > 0 \,,
    \]
    where $\displaystyle{\C'_Q = \bigcap_{\eta > 0} \cM_Q^\eta}$.
\end{theorem}

\begin{proof}
    We decompose the proof of the counter-example into two steps. First, we provide closed-form expressions for the sets $\cM_Q^{\eta}$
    and for their limit $\C_Q'$, as illustrated in Figure~\ref{fig:counter-example}.
    Second, we construct a sequence $(\bz_t)_{t \geq 1}$ that satisfies the claim. \bigskip

    \emph{Step 1: Structure.}
    We show that for $0 < \eta < 0.2$,
    \begin{equation}
    \label{eq:decomposition}
    \M_Q^{\eta} = \conv\Biggl( \Biggl[ \begin{pmatrix} -\sqrt{2/3}\,\eta \\ 3 \end{pmatrix},
    \begin{pmatrix} 0 \\ 3 \end{pmatrix} \Biggr)
    \, \cup \, \biggl\{ \begin{pmatrix} 0.1 \\ 2 \end{pmatrix} \biggr\} \, \cup \,
    \Biggl[ \begin{pmatrix} 0.2-\eta /\sqrt{2} \\ 1 \end{pmatrix} , \begin{pmatrix} 0.2 \\ 1 \end{pmatrix} \Biggr]
    \Biggr)\,.
    \end{equation}
    On the left and middle pictures of Figure~\ref{fig:counter-example},
    we illustrate what $\M_Q^{\eta}$ is the convex hull (in shaded green) of:
    the upper semi-open interval (the left endpoint is included and is marked in green while the right endpoint is not included is marked in red),
    the middle point (marked in green), and the lower closed interval
    (with both endpoints marked in green).

    \begin{figure}[t]
    \definecolor{mygreen}{RGB}{0,128,0}
\definecolor{mylightgreen}{RGB}{200,240,200}
\definecolor{myred}{RGB}{200,0,0}

\centering

\begin{minipage}{0.3\textwidth}
\centering
\begin{tikzpicture}[scale=4,xshift=0.1cm,yshift=0.05cm, x = 1cm, y = 0.9cm]

\def\alphanum{0.4}
\def\Ltop{1.0}
\def\Lone{0.66}
\def\Ltwo{0.33}
\def\eps{0.0707}
\def\epsbis{0.0816}
\pgfmathsetmacro{\Cx}{\alphanum-\eps}

\draw[->] (-0.15,0.05) -- (-0.15,1.1);
\draw[->] (-0.15,0) -- (.75,0);

\draw (-0.17,\Ltop) -- (-0.13,\Ltop);
\node[left] at (-0.17,\Ltop) {$3$};

\draw (-0.17,\Lone) -- (-0.13,\Lone);
\node[left] at (-0.17,\Lone) {$2$};

\draw (-0.17,\Ltwo) -- (-0.13,\Ltwo);
\node[left] at (-0.17,\Ltwo) {$1$};

\draw (\alphanum,0.02) -- (\alphanum,-0.02);
\node[below] at (\alphanum,0) {\small$0.2$};

\draw (0.2,0.02) -- (0.2,-0.02);
\node[below] at (0.2,0) {\small$0.1$};

\draw (0,0.02) -- (0,-0.02);
\node[below] at (0,0) {\small$0$};

\coordinate (A) at (0.2,\Lone);
\coordinate (B) at (\alphanum,\Ltwo);
\coordinate (C) at (\Cx,\Ltwo);
\coordinate (D) at (0,\Ltop);
\coordinate (E) at (-\epsbis,\Ltop);

\fill[mylightgreen] (B) -- (C) -- (E) -- (D) -- cycle;

\draw[ mygreen] (A) -- (B);
\draw[mygreen] (B) -- (C);
\draw[ mygreen] (C) -- (E);
\draw[mygreen] (E) -- (D);

\draw[thick, myred, postaction={
    decorate,
    decoration={
      markings,
      mark=between positions 0.05 and 0.95 step 0.08
      with {
        \draw[myred, thick] (0,-0.04) -- (0,0.04);
      }
    }
  }] (D) -- (A);

\draw[dashed]
  ($(D)!-0.5!(B)$) -- ($(D)!1.75!(B)$)
  node[right] at ($(D)!-.25!(B)$) {$\mathcal{D}$};


\fill[mygreen] (A) circle (0.02);
\fill[mygreen] (B) circle (0.02);
\fill[mygreen] (C) circle (0.02);
\fill[myred] (D) circle (0.02);
\fill[mygreen] (E) circle (0.02);

\end{tikzpicture}
\end{minipage}
\hspace{0.02\textwidth}
\begin{minipage}{0.3\textwidth}
\centering
\begin{tikzpicture}[scale=4,xshift=0.2cm,yshift=0.05cm, x = 1cm, y = 0.9cm]
\def\alphanum{0.4}
\def\Ltop{1.0}
\def\Lone{0.66}
\def\Ltwo{0.33}
\def\eps{0.035}
\def\epsbis{0.0408}
\pgfmathsetmacro{\xnew}{-(\Ltop*\alphanum/3)/(\Lone-\Ltwo)}
\pgfmathsetmacro{\Cx}{\alphanum-\eps}

\draw[->] (-0.15,0.05) -- (-0.15,1.1);
\draw[->] (-0.15,0) -- (.75,0);

\draw (-0.17,\Ltop) -- (-0.13,\Ltop);
\node[left] at (-0.17,\Ltop) {$3$};

\draw (-0.17,\Lone) -- (-0.13,\Lone);
\node[left] at (-0.17,\Lone) {$2$};

\draw (-0.17,\Ltwo) -- (-0.13,\Ltwo);
\node[left] at (-0.17,\Ltwo) {$1$};

\draw (\alphanum,0.02) -- (\alphanum,-0.02);
\node[below] at (\alphanum,0) {\small$0.2$};

\draw (0.2,0.02) -- (0.2,-0.02);
\node[below] at (0.2,0) {\small$0.1$};

\draw (0,0.02) -- (0,-0.02);
\node[below] at (0,0) {\small$0$};

\coordinate (A) at (0.2,\Lone);
\coordinate (B) at (\alphanum,\Ltwo);
\coordinate (C) at (\Cx,\Ltwo);
\coordinate (D) at (0,\Ltop);
\coordinate (E) at (-\epsbis,\Ltop);

\fill[mylightgreen] (B) -- (C) -- (E) -- (D) -- cycle;

\draw[mygreen] (A) -- (B);
\draw[mygreen] (B) -- (C);
\draw[mygreen] (C) -- (E);
\draw[mygreen] (E) -- (D);

\draw[thick, myred, postaction={
    decorate,
    decoration={
      markings,
      mark=between positions 0.05 and 0.95 step 0.08
      with {
        \draw[myred, thick] (0,-0.04) -- (0,0.04);
      }
    }
  }] (D) -- (A);

\draw[dashed]
  ($(D)!-0.5!(B)$) -- ($(D)!1.75!(B)$)
  node[right] at ($(D)!-.25!(B)$) {$\mathcal{D}$};


\fill[mygreen] (A) circle (0.02);
\fill[mygreen] (B) circle (0.02);
\fill[mygreen] (C) circle (0.02);
\fill[myred] (D) circle (0.02);
\fill[mygreen] (E) circle (0.02);

\end{tikzpicture}
\end{minipage}
\hspace{0.02\textwidth}
\begin{minipage}{0.3\textwidth}
\centering
\begin{tikzpicture}[scale=4,xshift=0.1cm,yshift=0.05cm, x = 1cm, y = 0.9cm]

\def\alphanum{0.4}
\def\Ltop{1.0}
\def\Lone{0.66}
\def\Ltwo{0.33}
\def\eps{0.0707}
\pgfmathsetmacro{\xnew}{-(\Ltop*\alphanum/3)/(\Lone-\Ltwo)}
\pgfmathsetmacro{\Cx}{\alphanum-\eps}

\draw[->] (-0.15,0.05) -- (-0.15,1.1);
\draw[->] (-0.15,0) -- (.75,0);

\draw (-0.17,\Ltop) -- (-0.13,\Ltop);
\node[left] at (-0.17,\Ltop) {$3$};

\draw (-0.17,\Lone) -- (-0.13,\Lone);
\node[left] at (-0.17,\Lone) {$2$};

\draw (-0.17,\Ltwo) -- (-0.13,\Ltwo);
\node[left] at (-0.17,\Ltwo) {$1$};

\draw (\alphanum,0.02) -- (\alphanum,-0.02);
\node[below] at (\alphanum,0) {\small$0.2$};

\draw (0.2,0.02) -- (0.2,-0.02);
\node[below] at (0.2,0) {\small$0.1$};

\draw (0,0.02) -- (0,-0.02);
\node[below] at (0,0) {\small$0$};

\coordinate (A) at (0.2,\Lone);
\coordinate (B) at (\alphanum,\Ltwo);
\coordinate (D) at (0,\Ltop);

\draw[thick,mygreen] (A) -- (B);

\draw[dashed]
  ($(D)!-0.5!(B)$) -- ($(D)!1.75!(B)$)
  node[right] at ($(D)!-.25!(B)$) {$\mathcal{D}$};


\fill[mygreen] (A) circle (0.02);
\fill[mygreen] (B) circle (0.02);

\end{tikzpicture}
\end{minipage}
    \caption{\label{fig:counter-example}
    Representations of $\M_Q^{\eta}$ for $\eta = 0.1$ (\emph{left}) and $\eta = 0.05$ (\emph{center})
    and of the limit $\C'_Q$ as $\eta \downarrow 0$ (\emph{right}).
    The interiors are shaded in \textcolor{Green}{green}. The endpoints of the two intervals and the point $(0.1,2)^\top$ generating
    $\M_Q^{\eta}$ according to~\eqref{eq:decomposition} are marked in \textcolor{Green}{green}
    (for four of them, included) or in \textcolor{red}{red} (for the point not included).
    As a consequence, the semi-open red hatched interval is not part of $\M_Q^{\eta}$.
    The dashed line $\mathcal{D}$ is introduced in Step~1 of the proof.}
\end{figure}

    Based on this expression, we then show that $\C_Q'$ is a (small) segment,
    \begin{equation}
    \label{eq:CQprime}
    \C_Q' \eqdef \bigcap_{\eta > 0} \M_Q^{\eta} = \Biggl[ \begin{pmatrix} 0.1 \\ 2 \end{pmatrix},
    \begin{pmatrix} 0.2 \\ 1 \end{pmatrix} \Biggr] \,.
    \end{equation}
    \bigskip

    \emph{Step 1A: Proof of~\eqref{eq:decomposition}.}
    It suffices to show that
    \begin{equation}
    \label{eq:step1proofeqdec}
    \bigl\{ \mfm^\star(\bq) : \bq \in Q^\eta \bigr\}
    = \Biggl[ \begin{pmatrix} -\sqrt{2/3}  \,\eta\\ 3 \end{pmatrix},
    \begin{pmatrix} 0 \\ 3 \end{pmatrix} \Biggr)
    \, \cup \, \Biggl\{ \begin{pmatrix} 0.1 \\ 2 \end{pmatrix} \Biggr\} \, \cup \,
    \Biggl[ \begin{pmatrix} 0.2-\eta /\sqrt{2} \\ 1 \end{pmatrix} , \begin{pmatrix} 0.2 \\ 1 \end{pmatrix} \Biggr] \,,
    \end{equation}
    where $Q^\eta = \bigl\{  \bq \in \Delta(\cB) : d(\bq,Q) \leq \eta \bigr\}$
    denotes the closed $\eta$--neighborhood of $Q$.
    Recall that distances are measured in Euclidean norm.

    We start with proving the reverse inclusion $\supseteq$ in~\eqref{eq:step1proofeqdec}.
        Since the distributions $(0,0.5,0.5)^\top$ and $(0,0.2,0.8)^\top$ belong to $Q$, we have that
    for all $0 < \eta' \leq \eta < 0.2$,
    \begin{align}
        \by_{\eta'} = \begin{pmatrix}
            \sqrt{2/3}  \, \eta' \\
            0.5 -\sqrt{2/3} \,\eta'/2\\
            0.5 -\sqrt{2/3}  \,\eta'/2
        \end{pmatrix} \quad \mbox{and} \quad
        \bz_{\eta'} = \begin{pmatrix}
            0 \\
            0.2 -\eta/\sqrt{2} \\
            0.8 +\eta/\sqrt{2}
        \end{pmatrix}  \quad \mbox{belong to} \quad Q^\eta \,. \label{eq:voisQ1}
    \end{align}
    The first distributions satisfy the conditions ``$q_1 >0$ and $q_1 + q_2 > 0.2$'',
    thus $\mfp^\star(\by_{\eta'}) = \bdelta_2$ and
    \[
    \bigl\{ \mfm^\star(\by_{\eta'}) : 0 < \eta' \leq \eta \bigr\} =
    \Biggl[ \begin{pmatrix} - \sqrt{2/3} \, \eta \\ 3 \end{pmatrix},
    \begin{pmatrix} 0 \\ 3 \end{pmatrix} \Biggr) \subseteq \bigl\{ \mfm^\star(\bq) : \bq \in Q^\eta \bigr\} \,.
    \]
    Similarly, the second distributions satisfy the condition ``$q_1 + q_2 \leq 0.2$'',
    for the indicated $\eta'$ but also for $\eta'=0$,
    thus $\mfp^\star(\bz_{\eta'}) = \bdelta_3$ and
    \[
    \bigl\{ \mfm^\star(\by_{\eta'}) : 0 \leq \eta' \leq \eta \bigr\} =
    \Biggl[ \begin{pmatrix} 0.2-\eta /\sqrt{2} \\ 1 \end{pmatrix} , \begin{pmatrix} 0.2 \\ 1 \end{pmatrix} \Biggr] \subseteq \bigl\{ \mfm^\star(\bq) : \bq \in Q^\eta \bigr\} \,.
    \]
    Finally, the distribution $(0,0.5,0.5)^\top$ belongs to $Q$ thus to $Q^\eta$, so that
    \[
      \Biggl\{ \begin{pmatrix} 0.1 \\ 2 \end{pmatrix} \Biggr\} \subset \bigl\{ \mfm^\star(\bq) : \bq \in Q^\eta \bigr\} \,.
    \]

    For the inclusion $\subseteq$ in~\eqref{eq:step1proofeqdec}, we first use the geometry of $Q$ and of the simplex $\Delta(\cB)$ to show that
    \begin{align}
    \forall \bq \in Q^\eta, \qquad q_1 \leq \sqrt{2/3} \, \eta \qquad \mbox{and} \qquad q_1 + q_2 \geq 0.2 -\eta/\sqrt{2} \,. \label{eq:boundsq1andq1+q2}
    \end{align}
    To prove these bounds, we fix $\bq \in Q^\eta$ and denote by $\bq'$ its projection onto $Q$ in Euclidean norm; in particular, $\Arrowvert\bq - \bq' \Arrowvert = d(\bq,Q)$.

    For the upper bound on $q_1$, we use that
    $q_1 + q_2 + q_3 = 1 = q_2' + q_3'$ (since $q'_1=0$ by definition of $Q$), which yields in turn,
    together with $u^2+v^2 \geq 2uv$,
    \[
        q_1^2 =  (q_2 - q_2')^2 + (q_3 - q_3')^2 + 2(q_2-q_2')(q_3-q_3') \geq 4(q_2-q_2')(q_3-q_3')\,.
    \]
    Therefore, by substituting first the inequality above, and then the equality above,
    \begin{align*}
        \sqrt{\frac{3}{2}} \, q_1 = \sqrt{2q_1^2 - \frac{1}{2}q_1^2} &\leq \sqrt{2q_1^2 - 2(q_2-q_2')(q_3-q_3')} \\
        &= \sqrt{q_1^2 + (q_2 - q_2')^2 + (q_3 - q_3')^2} = \Arrowvert\bq - \bq' \Arrowvert
        = d(\bq,Q) \leq \eta \,.
    \end{align*}

    For the lower bound on $q_1+q_2$, we only need to prove it in the case where $q_1+q_2 \leq 0.2$,
    which we assume next. The bound follows again from the equality
    $q'_3 - q_3 = q_1 + q_2 - q_2'$ and from the fact that $0 \leq q_1 + q_2 \leq 0.2 \leq q'_2$
    (since $\bq' \in Q$ and by definition of $Q$): therefore,
    $0 \leq 0.2-(q_1+q_2) \leq q'_2-(q_1+q_2)$, thus
    \begin{align*}
        \sqrt{2}\bigl(0.2-(q_1+q_2)\bigr) = \sqrt{2\,\bigl(0.2-(q_1+q_2)\bigr)^2}
        \leq \sqrt{2\,\bigl(q'_2 - (q_1+q_2)\bigr)^2} \smash{\underbrace{\leq d(\bq,Q)}_{\mbox{\tiny see below}}}
        \leq \eta \,,
    \end{align*}
    where the inequality marked above comes from
    \[
        2\,\bigl(q_1+(q_2-q'_2)\bigr)^2 = q_1^2 + (q_2 - q'_2)^2 +\underbrace{2q_1(q_2-q'_2)}_{\vphantom{(q_3 - q'_3)^2}\leq 0} +\underbrace{\bigl(q_1+q_2-q'_2\bigr)^2}_{= (q_3 - q'_3)^2} \leq \Arrowvert \bq - \bq' \Arrowvert^2 \,.
    \]

    Finally, we resort to the bounds obtained in~\eqref{eq:boundsq1andq1+q2} together with the definition of $\mfm^{\star}$, and get, for all $\bq \in Q^\eta$,
    \[
    \mfm^\star(\bq) =
    \begin{cases}
    \begin{pmatrix} 0.1 \\ 2 \end{pmatrix} & \mbox{if } q_{1} = 0 \mbox{ and } q_{1} + q_{2}  > 0.2 \,, \smallskip \\
    \begin{pmatrix} -q_1 \\ 3 \end{pmatrix} \in     \Biggl[ \begin{pmatrix} -\sqrt{2/3} \, \eta  \\ 3 \end{pmatrix} , \begin{pmatrix} 0\\ 3 \end{pmatrix} \Biggr) & \mbox{if } q_{1} > 0  \mbox{ and } q_{1} + q_{2}  > 0.2\,, \smallskip \\
    \begin{pmatrix} q_1+q_2 \\ 1 \end{pmatrix} \in     \Biggl[ \begin{pmatrix} 0.2-\eta /\sqrt{2} \\ 1 \end{pmatrix} , \begin{pmatrix} 0.2 \\ 1 \end{pmatrix} \Biggr] \quad & \mbox{if } q_{1} + q_{2}  \leq 0.2 \,,
    \end{cases}
    \]
    which concludes the proof of the inclusion $\subseteq$ in~\eqref{eq:step1proofeqdec}.
    \bigskip

    \emph{Step 1B: Proof of~\eqref{eq:CQprime}.} We consider the line $\cD$ generated by the segment indicated in~\eqref{eq:CQprime} and illustrated by the dashed line in Figure~\ref{fig:counter-example}:
    \[
    \cD = \bigl\{ \bc : \phi(\bc) = 0 \bigr\}\,, \qquad \mbox{where} \qquad \phi : (x,y) \longmapsto y - 3 + 10\,x \,.
    \]
    We first note that the closed-form expression~\eqref{eq:decomposition} of the $\M_Q^{\eta}$
    shows that $\C'_Q \subseteq \cD$.
    We now show that no point of the red hatched semi-open interval of Figure~\ref{fig:counter-example} (which is a subset of $\cD$)
    belongs to any of the sets $\cM_Q^{\eta}$, where $0 < \eta < 0.2$.
    Indeed, as illustrated in Figure~\ref{fig:counter-example}, $\cM_Q^{\eta}$ is included in the left hyperplane
    formed by $\cD$, i.e., for all $\bc \in \cM_Q^{\eta}$, we have $\phi(\bc) \leq 0$.
    Resorting to the Carathéodory's theorem and due to~\eqref{eq:decomposition},
    any point in $\cM_Q^{\eta}$ can be written as a convex combination of at most three points
    in
    \[
    \G_Q^{\eta} = \Biggl[ \begin{pmatrix} -\sqrt{2/3}\,\eta \\ 3 \end{pmatrix},
    \begin{pmatrix} 0 \\ 3 \end{pmatrix} \Biggr)
    \, \cup \, \biggl\{ \begin{pmatrix} 0.1 \\ 2 \end{pmatrix} \biggr\} \, \cup \,
    \Biggl[ \begin{pmatrix} 0.2-\eta/\sqrt{2} \\ 1 \end{pmatrix} , \begin{pmatrix} 0.2 \\ 1 \end{pmatrix} \Biggr]\,.
    \]
    Note that
    \[
    \G_Q^{\eta} \cap \cD = \G_Q^{\eta} \cap \bigl\{ \bc : \phi(\bc) = 0 \bigr\}
    = \biggl\{ \begin{pmatrix} 0.1 \\ 2 \end{pmatrix}, \begin{pmatrix} 0.2 \\ 1 \end{pmatrix} \biggr\}\,,
    \]
    while $\phi(\bc) < 0$ for all other points in $\G_Q^{\eta}$.
    Therefore, by the linearity of $\phi$, only
    points in $\cM_Q^{\eta}$ that can be written as convex combinations of $\G_Q^{\eta} \cap \cD$
    belong to $\M_Q^{\eta} \cap \cD$:
    \[
    \M_Q^{\eta} \cap \cD = \conv\bigl(\G_Q^{\eta} \cap \cD \bigr)
    = \biggl[ \begin{pmatrix} 0.1 \\ 2 \end{pmatrix}, \begin{pmatrix} 0.2 \\ 1 \end{pmatrix} \biggr]\,.
    \]
    Thus, using the inclusion $\C'_Q \subseteq \cD$ stated above,
    \[
    \C'_Q = \C'_Q \cap \cD = \bigcap_{\eta > 0} \bigl( \M_Q^{\eta} \cap \cD \bigr) =
    \biggl[ \begin{pmatrix} 0.1 \\ 2 \end{pmatrix}, \begin{pmatrix} 0.2 \\ 1 \end{pmatrix} \biggr]\,,
    \]
    as announced.
    \bigskip

    \emph{Step 2: Construction of the sequence.} We build a sequence $(\bz_t)_{t \geq 1}$ that satisfies the claim of the counter-example, i.e., such that
    \[
    \lim_{T \to \infty} \frac{1}{T}\sum_{t=1}^{T} d(\bz_t,Q) = 0 \qquad \mbox{but} \qquad \liminf_{T \to \infty} \ d\!\left(\frac{1}{T}\sum_{t=1}^T \mfm^{\star}(\bz_t), \, \C'_Q\right) > 0\,.
    \]
    Following Step~1, we define the sequence in such a way that $\bz_t$ is in a $\eta_t$--neighborhood of $Q$ for $\eta_t = \sqrt{2}/(t+2)^2$, which is smaller than $0.2$ for all $t \geq 1$.
    To do so, we consider a fixed distribution in $Q$, namely, $(0,0.2,0.8)^{\top}$,
    and introduce a vanishing perturbation of order $1/t^2$ to this distribution.
    Formally, we define the sequence $(\bz_t)_{t \geq 1}$ as
    \[
    \bz_t = \bigl(1/(t+2)^2\bigr) \, \bdelta_{1} +  0.2 \, \bdelta_{2} + \bigl(0.8-1/(t+2)^2\bigr)\, \bdelta_{3}  \,.
    \]
    By construction, the sequence $(\bz_t)_{t \geq 1}$ satisfies $d(\bz_t,Q) \leq \sqrt{2}/(t+2)^2$ for all $t \geq 1$, so that
    \[
    \lim_{T \to \infty} \frac{1}{T}\sum_{t=1}^{T} d(\bz_t,Q) = 0 \,.
    \]
    We now use the continuity of the point-to-set distance and show that
    \[
    \lim_{T \to \infty} \ d\!\left(\frac{1}{T}\sum_{t=1}^T \mfm^{\star}(\bz_t), \, \C'_Q\right)
    = \left\Arrowvert \begin{pmatrix}
    0 \\
    3
    \end{pmatrix} - \begin{pmatrix}
    0.1 \\
    2
    \end{pmatrix} \right\Arrowvert = \sqrt{1.01}
    > 0 \,.
    \]
    Indeed, each distribution $\bz_t = (z_{t,1}, z_{t,2}, z_{t,3})^{\top}$ is such that $z_{t,1} >0$ and $z_{t,1} + z_{t,2} > 0.2$; thus, for all $t \geq 1$, the definitions of $\mfp^\star$ and $\mfm^\star$ imply $\mfp^\star(\bz_t) = \bdelta_{2}$ and
    \[
    \mfm^\star(\bz_t) = \begin{pmatrix}
    - 1/(t+2)^2 \\
    3
    \end{pmatrix}\,,
    \qquad \mbox{so that} \qquad
    \lim_{T \to \infty}\frac{1}{T}\sum_{t=1}^T \mfm^{\star}(\bz_t) =
    \begin{pmatrix}
    0 \\
    3
    \end{pmatrix} \,.
    \]
    The average reward converges to the point $(0,3)^\top$, and the closest point in $\C'_Q$
    is $(0.1, 2)^\top$, as follows from the closed-form expression of $\C'_Q$ derived in Step~1
    (see also Figure~\ref{fig:counter-example}). This concludes the proof.
\end{proof}

\begin{remark}
   The sequence $(\bz_t)_{t \geq 1}$ in \citet[Corollary 2]{bernstein2014opportunistic} is generated by a calibrated forecaster, typically involving randomized predictions \citep{oakes1985self,dawid1985comment}.
   In contrast and with no loss of generality, Step~2 of the proof of the counter-example uses a deterministic sequence.
\end{remark}

\bibliography{PS-Blackwell-ACI-bib}
\end{document}